\documentclass[journal]{IEEEtran}
\usepackage{dg_usepackage}
\usepackage{dg_def}
\input{dg_cxref_a}

\def\TV{{\text{TV}}}
\def\tbx{\widetilde{\bx}}

\def\kui{\textcolor{black}}
\def\kuiredd{\textcolor{black}}
\def\kuire{\textcolor{black}}
\def\dg{\textcolor{black}}
\def\black{\textcolor{black}}

\def\dgr{\textcolor{black}}
\def\kuired{\textcolor{black}}

\def\dgrr{\textcolor{black}}

\def\dgrrr{\textcolor{black}}

\ifCLASSINFOpdf
\else
\fi

\begin{document}
\title{MPTV: Matching Pursuit Based Total Variation Minimization for Image Deconvolution}

\author{Dong~Gong, Mingkui~Tan, Qinfeng~Shi, Anton~van den Hengel, and Yanning~Zhang
\thanks{Dong Gong and Yanning Zhang are with the School of Computer Science and Engineering, Northwestern Polytechnical University, Xi'an, China, 710129 (e-mail: edgong01@gmail.com; ynzhang@nwpu.edu.cn). 
}
\thanks{Mingkui Tan is with South China University of Technology, Guangzhou, China, 510006 (e-mail: mingkuitan@scut.edu.cn).}
\thanks{Qinfeng Shi and Anton van den Hengel are with The University of Adelaide, SA 5005, Australia (e-mail: javen.shi@adelaide.edu.au; anton.vandenhengel@adelaide.edu.au).}
\thanks{Corresponding author: Qinfeng Shi. The first two authors contributed equally to this work.}
}

\markboth{IEEE Transactions on Image Processing,~Vol.~xx, No.~xx,~2018}{Shell \MakeLowercase{\textit{et al.}}: Bare Demo of IEEEtran.cls for Journals}

\maketitle
\begin{abstract}
Total variation (TV) regularization has proven effective for a range of computer vision tasks through its preferential weighting of sharp image edges. Existing TV-based methods, however, often suffer from the over-smoothing issue and solution bias caused by the homogeneous penalization. In this paper, we consider addressing these issues by applying inhomogeneous regularization on different image components. We formulate the inhomogeneous TV minimization problem as a convex quadratic constrained linear programming problem. Relying on this new model, we propose a matching pursuit based total variation minimization method (MPTV), specifically for image deconvolution. The proposed MPTV method is essentially a cutting-plane method, which iteratively activates a subset of nonzero image gradients, and then solves a subproblem focusing on those activated gradients only. Compared to existing methods, MPTV is less sensitive to the choice of the trade-off parameter between data fitting and regularization. Moreover, the inhomogeneity of MPTV alleviates the over-smoothing and ringing artifacts, and improves the robustness to errors in blur kernel. Extensive experiments on different tasks demonstrate the superiority of the proposed method over the current state-of-the-art.
\end{abstract}


\IEEEpeerreviewmaketitle
\section{Introduction}

\IEEEPARstart{M}{any} image restoration tasks can be formulated as an inverse problem, which
allows the recovery of the latent image $\bx$ from a measured image $\by$.
The imaging model can be formulated as:
\begin{equation}\label{eq:imaging_model}
  \by = \bA\bx+\bn,
\end{equation}
where $\bx\in \mbR^{n}$ denotes the
latent
image, $\by\in \mbR^{n}$ denotes the
measured
image, $\bA\in \mbR^{n\times n}$ is a linear operator
which models the measurement process, and
$\bn\in \mbR^{n}$ is an additive noise vector.
Directly recovering $\bx$ from model \eqref{eq:imaging_model} is an ill-posed problem, because there are usually too many solutions. A proper prior or regularizer for $\bx$ is important for reducing the ill-posedness of the problem.
Given $\by$ and $\bA$, and assuming $\bn$ is
sampled from
i.i.d. Gaussian white noise, image restoration can be achieved by solving the following for $\bx$:
\begin{equation}\label{eq:problem_model_org}
  \min_{\bx} \frac{1}{2} \|\by-\bA\bx\|_2^2 + \lambda\Omega(\bx),
\end{equation}
where $\Omega(\cdot)$ is the regularizer on $\bx$,
and $\lambda>0$ is {a trade-off parameter}.
Many {kinds of regularizer $\Omega(\cdot)$}, such as the total variation (TV) norm \cite{rudin1992TV}, wavelet frame-based sparse priors \cite{Cai2012IMAGE} and Gaussian mixture models \cite{Zoran2011EPLL}, have been proposed to handle different image restoration tasks, \eg denoising, inpainting and deconvolution.
In this paper, we focus on one particular task, namely, \emph{non-blind image deconvolution}.

\par
For image deconvolution \cite{wang2008new,gong2016active}, the matrix $\bA$ in model \eqref{eq:imaging_model} and \eqref{eq:problem_model_org} represents a convolution matrix, \dgr{with an embedded blur kernel (a.k.a. point spread function, PSF) $\bk$}. The degenerate image $\by$ is usually modeled by $$\by=\bx*\bk+\bn,$$ where $*$ denotes the 2D convolution operator. \emph{Non-blind deconvolution} seeks to recover $\bx$ from the blurred image $\by$ given a known blur kernel $\bk$ (and thus $\bA$).
The problem
is ill-posed and non-trivial to address.
The practical importance of the problem has motivated significant research attention
\cite{richardson1972bayesian,wang2008new,krishnan2009fast,dong2011image}.
In particular, TV regularized  methods have been intensively exploited and have demonstrated great success in image deconvolution \cite{rudin1992TV,goldstein2009split,zuo2011generalized,lou2015weighted}.

\par
The TV-based model was first proposed by Rudin \etal \cite{rudin1992TV} for image denoising \cite{goldstein2009split,wang2008new,yuan2015l0tv}.
It has subsequently been applied to a variety of other tasks including
deconvolution \cite{chan1998total,wang2008new,perrone2014total}, super-resolution \cite{marquina2008imageSR}, and inpainting \cite{shen2002inpating}.
By assuming that the image gradients are sparse, TV-based methods
reflect the tendency of clear, sharp images towards piecewise smoothness.
\dgr{Although $\ell_0$-norm {regularization can be applied to promote sparse solution directly, the resultant optimization problem is difficult to solve} due to its non-convexity. In practice, $\ell_1$-norm or $\ell_{2,1}$-norm based TV regularization is more commonly used \cite{wang2008new,lou2015weighted}.}
In this paper, we consider the $\ell_{2,1}$-norm based isotropic TV norm \cite{rudin1992TV}:
\begin{equation}\label{eq:iso_tv_norm}
  \Omega_{\TV}(\bx)=\|\bD\bx\|_{2,1} = \sum\nolimits_{i=1}^n \|~[(\bD_v \bx)_i, (\bD_h \bx)_i]~\|_{2},
\end{equation}
where $\bD = [\bD_v^{\T}, \bD_h^{\T}]^{\T}$, and $\bD_v\in\!\! \mbR^{n\times n}$ and $\bD_h\in\!\! \mbR^{n\times n}$ denote the first-order difference matrices in vertical and horizontal directions, respectively.  $\|\cdot\|_{2,1}$ denotes the $\ell_{2,1}$-norm. Specifically, $[(\bD_v \bx)_i, (\bD_v \bx)_i]$ denotes a $1$-by-$2$ vector concatenating the $i$-th element of $\bD_v \bx$ and $\bD_h \bx$.

\par
A variety of TV-based methods for reducing the impact of both image noise and blur have been devised.
These methods tend to suffer from a common set of deficiencies, however.
Firstly, the homogeneous penalty may cause over-smoothing of high-frequency image components  (such as edges and corners). Secondly, due to the regularization bias \cite{deledalle2015debiasing,deledalle2017clear,Brinkmann2017bias} rooted in the specific $\ell_1$-norm or $\ell_{2,1}$-norm,
\dgrr{it is difficult to select an appropriate trade-off hyper-parameter $\lambda$ to achieve a solution that has sparse gradients and suffers less bias issue.}
Thirdly, existing TV-based deconvolution methods are usually sensitive to  errors or noise in $\bk$ (\ie $\bA$), {and this} often causes heavy ringing artifacts in the
latent image
$\bx$ \cite{Shan2008High,mosleh2014image}. For existing TV methods, it is non-trivial to achieve a balance between suppressing artifacts and preserving image details.

\par
{In this paper, we propose a matching pursuit algorithm, called MPTV, for solving TV regularization based image deconvolution. The paper extends our conference paper \cite{gong2017MPGL} where we addressed the generalized lasso problem. The new contributions of this paper are summarized as follows. }

\begin{itemize}
\item We {reformulate} the TV minimization problem as a {quadratically constrained linear programming (QCLP)} problem and {then propose a matching pursuit algorithm {called MPTV} to address the resultant problem.} {Instead of focusing on all the image gradients at the beginning,
we iteratively invoke the most beneficial gradient subset, followed by solving a subproblem constrained on selected subsets only.}

\item The proposed MPTV method is able to alleviate the over-smoothing issue while yielding a solution with sparse gradients due to the matching pursuit like optimization strategy. Compared to existing methods, MPTV can help to reduce the solution bias.
\dgrr{Moreover, the proposed optimization scheme also helps to improve the robustness of the method to noise and/or errors in $\bA$. {In particular, for the task of image deconvolution, MPTV helps to suppress the} ringing artifacts and recover the image details.}

\item {Many TV-based methods require an extensive and often imprecise selection of hyper-parameters to balance the sparsity and the fitness to observation. This issue can be significantly alleviated in the proposed MPTV method due to the new formulation and optimization strategy, particularly the early stopping strategy that will be introduced. In other words, MPTV is less sensitive to the selection of the regularization parameter.}

\end{itemize}

\par
The remainder of this paper is organized as follows. In Section \ref{sec:related_work}, we review related work in TV models and image deconvolution.
{In Section \ref{sec:mptv}, we introduce our matching pursuit based TV minimization (MPTV) algorithm and discuss the details, including stopping conditions, parameter setting, and convergence analysis.}
Empirical studies on various datasets are presented in Section \ref{sec:exp}. We conclude in Section \ref{sec:con}.

\section{Notation and Related Work}\label{sec:related_work}
\subsection{Notation}
\label{sec:notation}
Let $\bA=[A_{i,j}]\in \mbR^{m\times n}$ and $\bv=[v_1,...,v_n]^\T\in \mbR^{n}$ denote a matrix and a vector, receptively, where $^\T$ denotes the transpose of a vector/matrix. Let $\0$ and $\1$ be vectors with all zeros and all ones, respectively, and let $\bI$ denote the identity matrix. Let $\bA_i$ or $\bA^i$ be a matrix indexed for some purpose, and let $\bv^i$ or $\bv_i$ be a vector indexed for some purpose.

\par
Given a vector $\bv$, let $\diag(\bv)$ be a diagonal matrix with diagonal elements equal to the vector $\bv$, and $\|\bv\|_p$ be the $\ell_p$-norm. Let $\odot$ denote the element-wise (Hadamard) product, $\otimes$ denote the Kronecker product
and $\supp(\bv)$ denote the support set of $\bv$ .
Given a positive integer $n$, let $[n]=\{1,...,n\}$.
Given any index set $\cT\subseteq [n]$, let $\cT^{c}$ be the complementary set of $\cT$, \ie $\cT^c=[n]\setminus\cT$, and $\card(\cT)$ be the cardinality of $\cT$. For a vector $\bv\in \mbR^n$, let $v_i$ denote the $i$-th element of $\bv$, and $\bv_\cT$ denote the subvector indexed by $\cT$. For a matrix $\bA$, let $\bA_\cT$ denote the columns of $\bA$ with indeces in
$\cT$.

\subsection{Total Variation Models}
Total variation regularization {has been the subject of} much attention over the past two decades, not least due to its effectiveness in image processing \cite{wang2008new,yuan2015l0tv}, compressive sensing \cite{goldstein2009split} and machine learning tasks \cite{wang2014highly,gong2017MPGL}. Considering the piecewise smooth property of images, Rudin \etal \cite{rudin1992TV} proposed \kui{the} TV regularizer for image denoising. Following that, many variants of the original TV regularizer \kui{have been} extensively studied. Beside the isotropic TV model proposed in \cite{rudin1992TV}, anisotropic TV has  also been wildly explored \cite{goldstein2009split,wang2008new}.

To alleviate the optimization difficulties caused by the non-differentiability of {TV model}s, some approximate TV models have been proposed, \eg smooth TV \cite{chan1999smoothTV} and Huber-norm based TV \cite{nikolova2005huberTV}. To preserve the {sparse nature} of the image gradients, the $\ell_0$-norm based TV regularizer was proposed for edge-preserving image editing tasks \cite{xu2011l0smooth} and blur kernel estimation \cite{xu2013unnatural}. Motivated by the $\ell_1-\ell_2$-norm in compressive sensing, Lou \etal \cite{lou2015weighted} proposed a weighted difference of anisotropic and isotropic TV regularizers to alleviate the over-smoothing {incurred by classical TV methods.} For these methods,
achieving acceptable performance is critically dependent on
{careful turning of hyper-parameters.} In fact, due to the bias nature in regularization, it is non-trivial to reduce bias issue
while pursing sparsity.
{To handle different types of noises, different data fitting functions are also studied}. Classical TV models \cite{rudin1992TV,wang2008new} use $\ell_2$ loss to fit the Gaussian noise in observation. To improve {the robustness of the model to outliers}, non-smooth loss functions, including $\ell_1$ loss \cite{yang2009l1loss,xu2010two}, $\ell_\infty$ loss \cite{clason2012linftyloss} and $\ell_0$ loss \cite{yuan2015l0tv}, {have been used} to handle Laplacian noise, uniform noise and impulse noise, respectively.

\par
{Various optimization strategies}  for TV regularized problems have also been extensively investigated.
Rudin \etal \cite{rudin1992TV} minimize the TV model using a gradient projection method, which, converges slowly due to the non-smoothness of the problem.
\dgrr{Osher \etal \cite{osher2005iterative} proposed an iterative regularization method based on Bregman distance minimization for solving the TV based image restoration problem. Burger \cite{burger2016bregman} discussed a series of Bregman distance based approaches for inverse problems and pointed out that the Bregman iteration scheme is beneficial for alleviating the bias caused by the TV regularizer.}
In the past two decades, many other  techniques {have been} proposed to accelerate convergence, including the primal-dual interior point method \cite{chan1999nonlinear}, the splitting Bregman method \cite{goldstein2009split}, the half-quadratic formulation based method \cite{wang2008new},  and alternating direction multiplier methods \cite{zuo2011generalized}.
\dgrr{Specifically, to handle the spatially inhomogeneous structural and textural information in the image, some approaches proposed to use spatially varying constraints with the TV regularization \cite{gilboa2006variational,bredies2013spatially}. Gilboa \etal \cite{gilboa2006variational} proposed a variational framework for image denoising, in which spatially varying constraints are used for different image local areas.
In \cite{bredies2013spatially}, a spatially dependent parameter selection scheme is proposed to achieve spatially adaptive TV regularization for image restoration.}

\begin{figure*}[!t]
\subfigure[Blurred signal $\by$, blur kernel $\bk$ and the original sharp signal $\bx$]{
\centering
\includegraphics[trim=0 3 0 5, clip, width=0.32\linewidth]{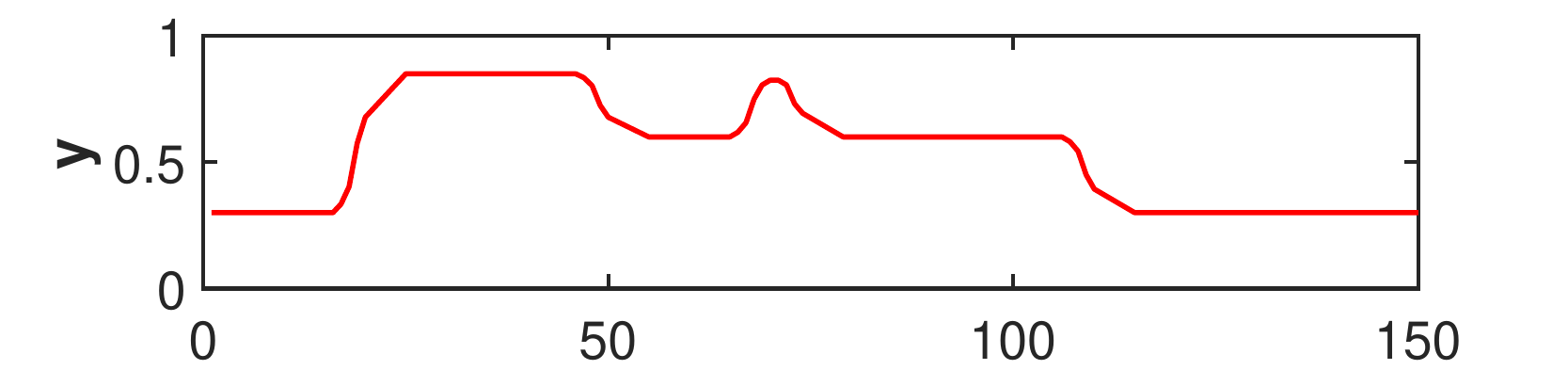}
\centering
\includegraphics[trim=0 3 0 5, clip, width=0.32\linewidth]{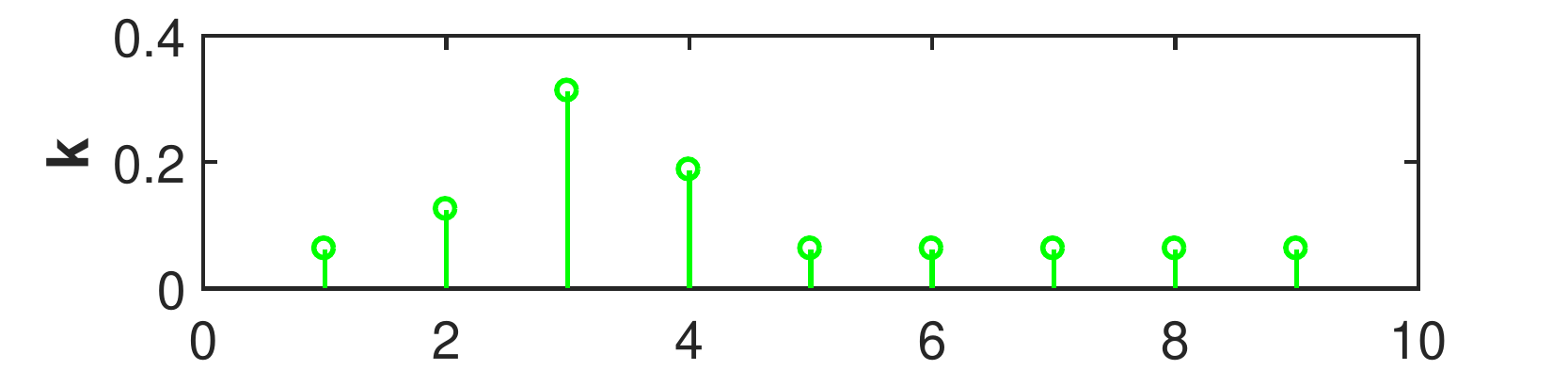}
\centering
\includegraphics[trim=0 3 0 5, clip, width=0.32\linewidth]{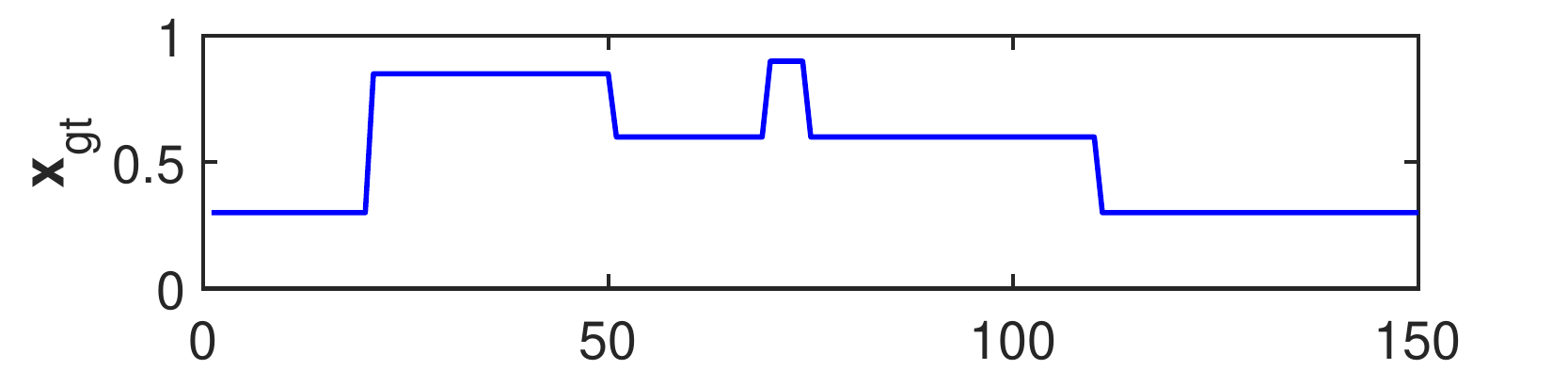}
}
\hfill
\vspace{-0.05cm}
\subfigure[Recovered signal $\bx$]{
\centering
\includegraphics[trim=0 3 0 5, clip, width=0.32\linewidth]{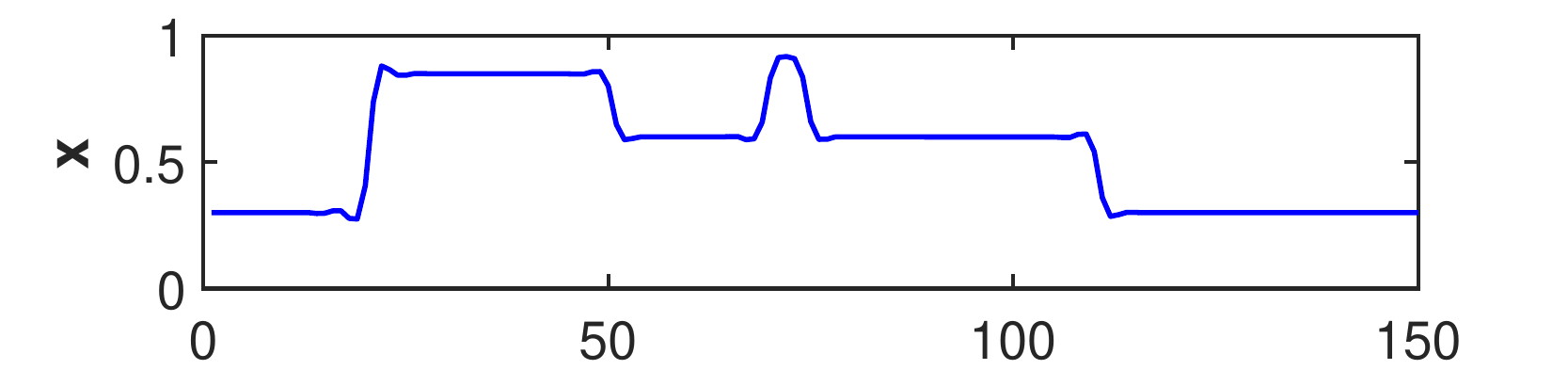}
\centering
\includegraphics[trim=0 3 0 5, clip, width=0.32\linewidth]{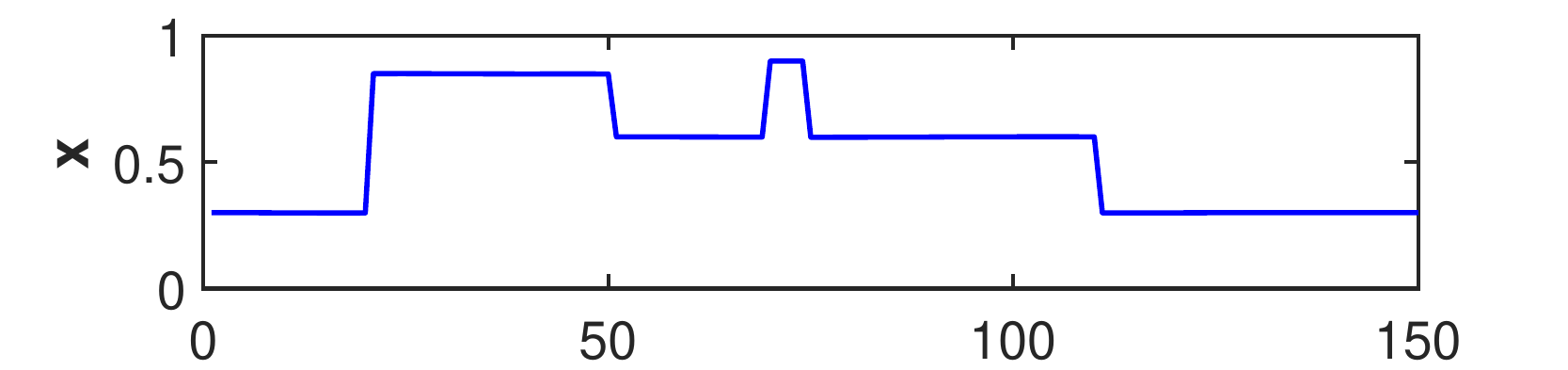}
\centering
\includegraphics[trim=0 3 0 5, clip, width=0.32\linewidth]{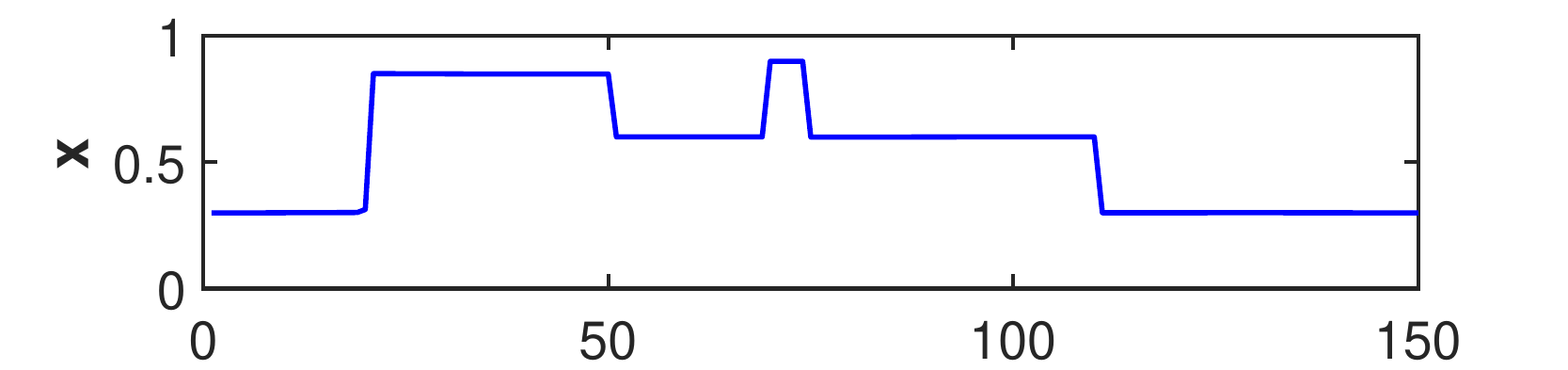}
}
\hfill
\vspace{-0.1cm}
\subfigure[Gradients $\bD\bx$ of recovered signal]{
\centering
\includegraphics[trim=0 2 0 5, clip, width=0.32\linewidth]{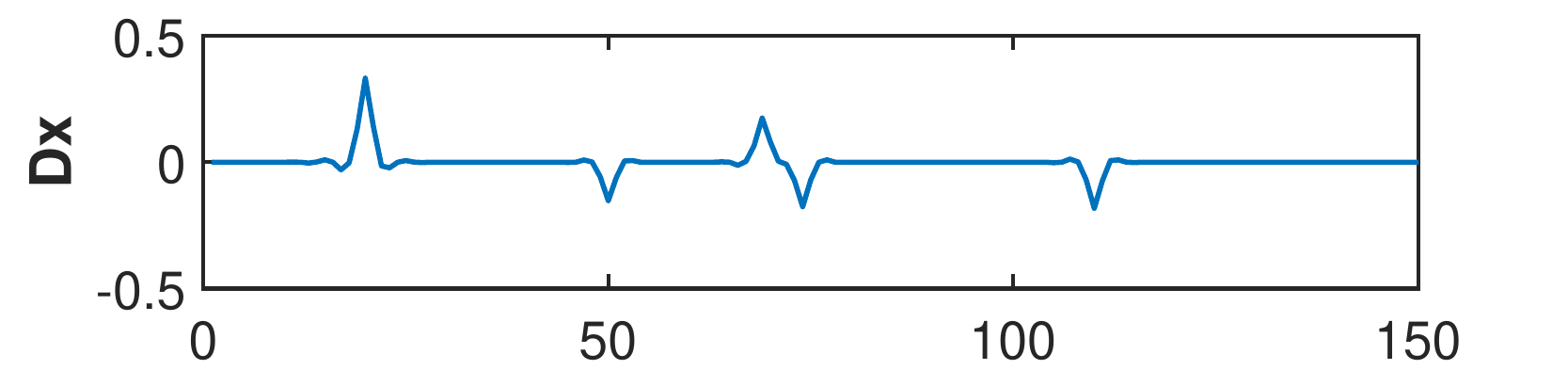}
\centering
\includegraphics[trim=0 2 0 5, clip, width=0.32\linewidth]{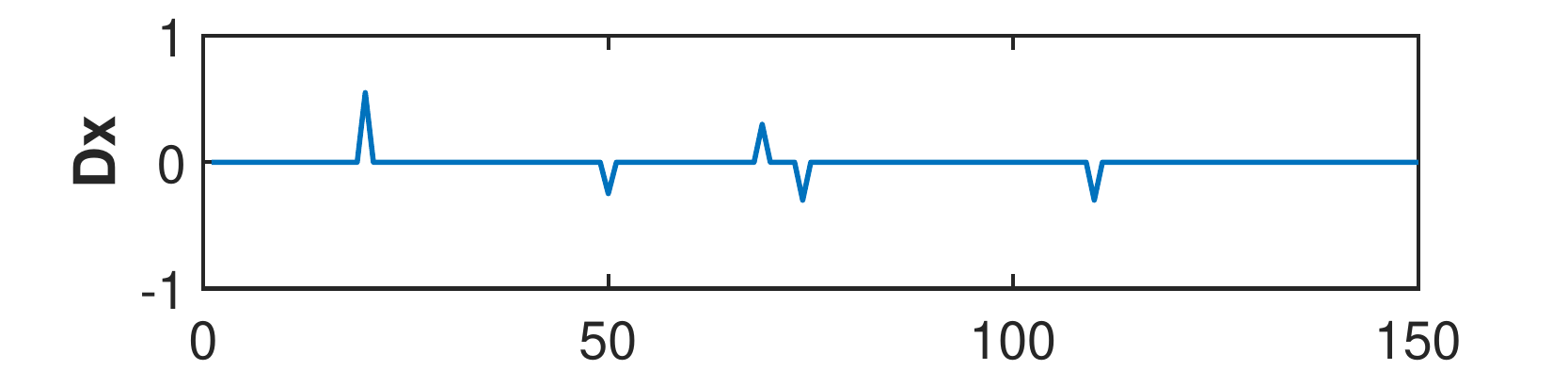}
\centering
\includegraphics[trim=0 2 0 5, clip, width=0.32\linewidth]{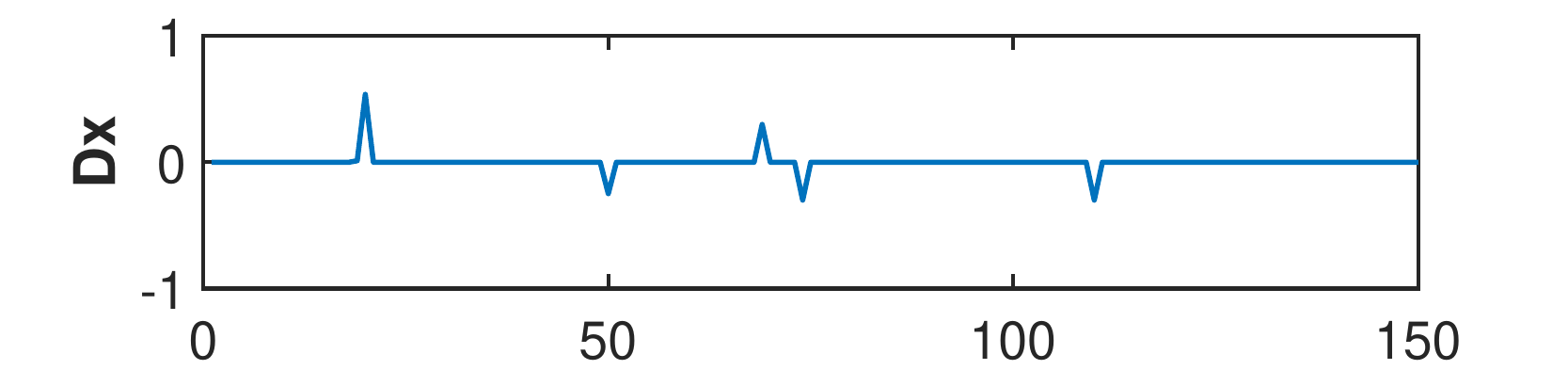}
}
\vspace{-0.2cm}
\caption{Deconvolution results on a 1D signal example.  (a) The blurred signal $\by$ is generated by blurring the sharp signal $\bx_{\text{gt}}$ with the blur kernel $\bk$. In (b) and (c), the results from left to right are generated by solving the TV regularized model using ADMM \cite{boyd2011ADMM} (TV-ADMM), solving problem \eqref{eq:mptv2} using the proposed MPTV method \dgr{($\kappa=1$)}, and solving problem \eqref{eq:mptv2} using ADMM given the support of the ground truth signal $\bx_{\text{gt}}$, respectively. (c) shows the gradients of the recovered signal in (b). In each case, $\lambda=0.01$.
}
\vspace{-0.3cm}
\label{fig:1d_example}
\end{figure*}

\subsection{Image Deconvolution}
\label{sec:related_work_deconv}
{Abundant studies have been conducted} on (non-blind) image deconvolution, {starting from} the classical ones including Wiener filter and Richadson-Lucy (RL) algorithms \cite{richardson1972bayesian}. Based on the piecewise smooth property of natural images and statistical studies of image gradients, many methods perform image deconvolution by preserving significant edges or sparse gradients. To this end, a series of TV based methods have been extensively exploited \cite{wang2008new,lou2015weighted,yuan2015l0tv}. Additionally,
in \cite{levin2007coded_irls}, an iterative reweighed least square (IRLS) has been used to achieve excellent results by encouraging the piecewise smooth property in image deconvolution. Krishnan and Fergus \cite{krishnan2009fast} proposed to perform image deconvolution using a hyper-Laplacian prior.
Moreover, instead of preserving sparse gradients, many other methods also perform well by relying on patch similarity \cite{dabov2008bm3ddeb}, kernel similarity \cite{kheradmand2014Kernel}, progressive multi-scale deconvolution \cite{yuan2008progressive}, group based sparse representation \cite{zhang2014group}, or stochastic optimization \cite{xiao2015stochastic}.

\par
\dgr{Besides the empirically designed models above, some methods have been proposed to learn models  for image deconvolution \cite{schuler2013machine,Zoran2011EPLL,xu2014deep,schmidt2014csf,gong2018learning}. The generative prior learning methods \cite{Zoran2011EPLL,schmidt2010generative} focus on learning image priors/regularizers (\eg a Gaussian Mixture Model based prior in \cite{Zoran2011EPLL}) from a set of clear and sharp images, which usually leads to non-convex problems and time-consuming optimization. For the sake of computational efficiency, some discriminative learning based methods \cite{schuler2013machine,xu2014deep,schmidt2014csf} learn the deconvolution models (\eg a convolutional neural network based model \cite{xu2014deep}) relying on pairs of sharp and blurred images. The learning based deconvolution methods are effective to utilize more information from training examples than the empirically designed approaches.
However, suffering from limited flexibility on the model design, existing learning based methods \cite{xu2014deep,schuler2013machine} often require customized training for specific blur kernels or noise, which limits their practicability.}

\par
In practice, image deconvolution results often suffer from wave-like  ringing artifacts, especially \kui{in regions} near strong edges. This \kui{issue} may be caused by the unavoidable error in blur \kui{kernel estimation} \cite{Shan2008High}, the Gibbs phenomenon \cite{yuan2007image}, the zero values in the frequency spectrum of the blur kernel \cite{mosleh2014image} and/or the mismatch between the data and the model (\eg non-conforming noise, saturation) \cite{cho2011outlier,Whyte2014Deblurring}. To reduce these artifacts in deconvolution, Yuan \etal \cite{yuan2008progressive} performed the Richardson-Lucy algorithm in a residual inter-and-intra-scale scheme with \kui{edge-preserving} bilateral filters. Shan \etal \cite{Shan2008High} proposed to \kui{combin} a smoothness constraint \kui{with} the $\ell_1$-norm regularizer.  Mosleh \etal \cite{mosleh2014image} proposed a post processing method to detect and suppress ringing based on frequency analysis. Some methods tried to avoid artifacts using more accurate imaging models and data fitting terms \cite{cho2011outlier,xu2010two,Whyte2014Deblurring}.

\section{Matching Pursuit Total Variation \\for Image Deconvolution}
\label{sec:mptv}

In this section, we present a matching pursuit based TV minimization method for non-blind image deconvolution.

\subsection{\kuired{A QCLP Reformulation of TV Model}}
Image deconvolution is rendered more challenging by the ill-posed nature of the problem and the noise in both the blurred image and the estimated blur kernel.
\dgr{
With the TV regularizer introduced in \eqref{eq:iso_tv_norm}, the clear and sharp image can be recovered by solving the following classic TV-norm regularized optimization problem:
\begin{equation}\label{eq:tv}
   \min_{\bx}  ~\frac{1}{2}\|\by-\bA\bx\|_2^2 + \lambda\Omega_{\TV}(\bx),
\end{equation}
where  $\Omega_{\TV}(\bx)=\|\bD\bx\|_{2,1} = \sum\nolimits_{i=1}^n \|~[(\bD_v\bx)_i, (\bD_h\bx)_i]~\|_{2}$. In problem \eqref{eq:tv}, the TV regularizer treats all elements in $\bx$ homogeneously. Although the $\ell_{2,1}$-norm is used for inducing sparsity on image gradients, it tends to shrink the large elements in image gradients towards zero \cite{Brinkmann2017bias}. As a result, while the TV regularizer in \eqref{eq:tv} favors a piecewise smooth solution, it may incur over-smoothing \dgr{deconvolution results} \cite{lou2015weighted} due to the bias issue \cite{deledalle2017clear,Brinkmann2017bias}.
We thus consider minimizing the total variation value by explicitly detecting and only preserving the sparse nonzero gradients, while suppressing the insignificant components. In this way, we can maintain the significant high-frequency elements in the image and suppress the noise. To illustrate the above issue, we show an example for 1D signal deconvolution in Fig. \ref{fig:1d_example}. This experiment shows that minimizing TV with active gradient detection can reduce the regularization bias significantly, and thus alleviates the over-smoothness and increases estimation accuracy accordingly.}

\par
\dgr{To explicitly detect the sparse nonzero image gradients in image deconvolution, we will formulate the TV model with a new binary vector to indicate the nonzero gradients, which will be estimated with the image simultaneously. Firstly,
we introduce $\bz=[\bz_v^\T, \bz_h^\T]^\T\in\mbR^{2n}$ to denote the concatenation of the gradients of two directions, namely $\bD_v\bx$ and $\bD_h\bx$. To find the most significant image gradients which contribute the most to the quality of image deconvolution, we introduce two binary vectors $\btau_v \in\{0,1\}^n$ and $\btau_h \in\{0,1\}^n$ to indicate the nonzero ones in $\bD_v\bx$ and $\bD_h\bx$, respectively.
We further define $\btau_v = \btau_h \triangleq \btau$, since each term in the $\ell_{2,1}$-norm based TV regularizer consists of two elements over vertical and horizontal directions, \eg $(\bD_v\bx)_i$ and $(\bD_h\bx)_i$ shown in \eqref{eq:iso_tv_norm}. We then let $\tbtau=[\btau^\T, \btau^\T]^\T$ indicate the nonzero components in $\bD\bx$ via $(\bz\odot\tbtau)$. To induce the sparsity, we impose an $\ell_0$-norm constraint $\|\btau\|_0\leq \kappa$. For simplicity, let $\Lambda=\{\btau|\btau\in\{0,1\}^{n},\|\btau\|_0\leq \kappa\}$ be the feasible domain of $\btau$.}

With the introduction of $\btau$ and the auxiliary variable $\bz$, we propose to address the following  TV-norm regularized problem:
\begin{equation}
\begin{split} \label{eq:mptv2_nonsplit}
  \min_{\btau\in\Lambda} \min_{\bx,\bxi,\bz}  &~~\frac{1}{2}\|\bxi\|_2^2 + \lambda\sum\nolimits_{i=1}^n \|~[(\bz_v)_i, (\bz_h)_i]~\|_2\\
  \st &~~\bxi=\by-\bA\bx,~\bD\bx=\bz\odot\tbtau, \\
\end{split}
\end{equation}
where $(\bz_v)_i$ and $(\bz_h)_i$ denote the $i$-th elements of $\bz_v$ and $\bz_h$, respectively.
\kui{The constraint $\|\btau\|_0\leq \kappa$ explicitly \kuired{constrains} the sparsity of gradients via $\bD\bx=\bz\odot\tbtau$, leading to an image $\bx$ with sparse $\bD\bx$ and a small total variation.
}

\par
To simplify the computation of $\sum\nolimits_{i=1}^n \|~[(\bz_v)_i, (\bz_h)_i]~\|_2$,
\dgr{we define binary matrices $\bC_i\in \{0,1\}^{2\times 2n}, \forall i\in[n]$ to select both $(\bz_v)_i$ and $(\bz_h)_i$, $\forall i\in[n]$ from $\bz$ via multiplication between $\bC_i$ and $\bz$, \ie $\bC_i\bz$. As an example, there is $\bC_i\bz=[(\bz_v)_i, (\bz_h)_i]^\T$.
In each $\bC_i$, the $i$-th and $(i+n)$-th elements on the two rows, respectively, are 1, and the rest are 0. }
Then, the $\ell_{2,1}$-norm in problem \eqref{eq:mptv2_nonsplit} can be rewritten as $\sum_{i=1}^n \|\bC_i \bz\|_2$.
\dgr{Moreover, we introduce a new vector $\bd$ to represent $\bz$, and let its subvectors $\bd_i=\bC_i\bz, \forall i\in[n]$.}
Then, we can obtain an equivalent formulation of problem \eqref{eq:mptv2_nonsplit}:
\begin{equation}
\begin{split} \label{eq:mptv2}
  \min_{\btau\in\Lambda} \min_{\bx,\bxi,\bz,\bd}  &~~\frac{1}{2}\|\bxi\|_2^2 + \lambda\sum\nolimits_{i=1}^n \|\bd_i\|_2\\
  \st &~~\bxi=\by-\bA\bx,~\bD\bx=\bz\odot\tbtau, \\
  &~~\bd_i=\bC_i\bz, \forall i\in [n].\\
\end{split}
\end{equation}

\par
In problem \eqref{eq:mptv2}, the integer $\kappa$ reflects \kui{our} rough knowledge of the sparsity of $\bD\bx$. Note that there are $|\Lambda|=\sum_{i=0}^\kappa\binom{n}{i}$ feasible $\btau$'s in $\Lambda$. Problem \eqref{eq:mptv2} tends to find the optimal $\btau$ from $\Lambda$ to minimize the objective in (\ref{eq:mptv2}).
Even though the sparsity constraint $\|\btau\|_0\leq \kappa$ explicitly induces sparsity in image gradients, the regularizer $\lambda\sum\nolimits_{i=1}^n \|\bd_i\|_2$ is still necessary for the robust recovery of $\bz$ due to possible noise in $\by$. Meanwhile, benefiting from the sparsity constraint and the optimization scheme that will be introduced, a small $\lambda$ can be used to reduce the bias \kui{induced by the regularization} \cite{tan2013mpl,nowak2007gradient}, \kui{which is beneficial} in alleviating over-smoothness.

\par
\kuired{Unfortunately,} problem \eqref{eq:mptv2} is a mixed integer programming problem, and thus hard to address. We consider a \emph{convex relaxation} to this problem. \kui{To achieve this, we first introduce the following proposition.}
\begin{proposition} \label{prop:mptv_dual}
By \kui{deriving the dual form of the inner minimization problem in (\ref{eq:mptv2})} w.r.t. $\bx$, $\bxi$, $\bz$ and $\bd$ \kui{given fixed $\btau\in\Lambda$}, problem \eqref{eq:mptv2} can be \kui{transformed} into
\begin{equation} \label{eq:mptv_dual}
\begin{split}
  \min_{\btau\in \Lambda}\max_{\balpha}& ~-\frac{1}{2}\|\balpha\|_2^2 + \balpha^\T\by\\
  \st &~~ \bA^\T\balpha =\bD^\T\bbeta, \\
  &~~ \tau_i \|\bbeta_i\|_2 \leq \lambda, ~\bbeta_i=\bC_i\bbeta, ~\forall i\in [n].\\
\end{split}
\end{equation}
where $\balpha\in \mbR^n$, $\bbeta_v\in \mbR^{n}$, $\bbeta_h\in \mbR^{n}$, $\bbeta=[\bbeta_v^\T, \bbeta_h^\T]^\T$ are the Lagrangian dual variables.
For any $\btau\in \Lambda$,  $\balpha^*=\bxi^*$ at the optimum of the inner problem.
\end{proposition}
The proof can be found in Appendix.

\par
\kuired{From Proposition \ref{prop:mptv_dual},  $\balpha^*$ equals to the fitting error $\bxi^*$.
Without loss of generality, we assume $\balpha \in \cA=[-l, l]^n$, where $l$ is a sufficiently large positive number and  $\cA$ is a compact domain. Based on \eqref{eq:mptv_dual}, the feasible domain of $\balpha$ w.r.t. each $\btau$ is $\cA_{\btau}=\{\balpha| \bA^\T\balpha =\bD^\T\bbeta, \tau_i \|\bbeta_i\|_2 \leq \lambda, \forall i\in [n], \balpha\in [-l,l]^n\}$. }
\kuired{Note that, according to the constraints in \eqref{eq:mptv_dual}, all $\cA_{\btau}$'s share the same $\bbeta$.}
We further define
\begin{equation}
  \phi(\balpha, \btau) = \frac{1}{2}\|\balpha\|_2^2 - \balpha^\T\by, ~\balpha \in \cA_{\btau}.
\end{equation}
By applying the minimax inequality in \cite{boyd2004convex}, we have
\begin{equation}\label{eq:minimax_inequality}
  \min_{\btau\in\Lambda}~\max_{\balpha\in\cA}-\phi(\balpha,\btau)
  \geq\max_{\balpha\in\cA}~\min_{\btau\in\Lambda}-\phi(\balpha,\btau).
\end{equation}
According to \eqref{eq:minimax_inequality}, $\max_{\balpha\in\cA}\min_{\btau\in\Lambda}-\phi(\balpha,\btau)$ is a lower bound of problem \eqref{eq:mptv_dual} and it is a convex problem.
By introducing a new variable $\theta\in\mbR$, $\max_{\balpha\in\cA}\min_{\btau\in\Lambda}-\phi(\balpha,\btau)$ can be equivalently written as a \emph{quadratically constrained linear programming} (QCLP) problem \cite{pee2011solving}:
\begin{equation}\label{eq:mptv2_qclp}
  \min_{\balpha\in\cA,\theta}~\theta,~~ \st~\phi(\balpha,\btau)\leq\theta, \forall \btau\in\Lambda,
\end{equation}
which is a convex relaxation of problem (\ref{eq:mptv_dual}).
\kuired{Note that each feasible $\btau$ corresponds to a constraint.
Note that there are $T=\sum_{i=0}^\kappa \binom{n}{i}$ elements in $\Lambda$. In other words,  problem \eqref{eq:mptv2_qclp} has exponentially many constraints, which makes it difficult to address directly.}

\subsection{\kui{Optimization of the QCLP TV Model}}
\label{sec:qclp_opt}
Though there are exponentially many constraints in \eqref{eq:mptv2_qclp}, most of them are inactive at the optimum, since only a subset of components are relevant in fitting the observation. Accordingly, we seek to address problem \eqref{eq:mptv2_qclp} using a \emph{cutting-plane method} \cite{kortanek1993ccp,tan2013mpl} as shown in Algorithm \ref{algo:ccp_dual}.
\dgr{Instead of handling all constraints at the same time, Algorithm \ref{algo:ccp_dual} iteratively finds the active constraints and then solves a subproblem \eqref{eq:ccp_master_prob} with the selected active constraints only.}

\begin{algorithm}[htp]\label{algo:ccp_dual}
\caption{Cutting-plane for the QCLP Problem \eqref{eq:mptv2_qclp}}
\KwIn{Observation $\by$, $\bA$, parameter $\lambda$.}
Initializing $\balpha^0=\by-(\1\otimes(\sum_{i=1}^n y_i/n))$, and $\btau_0=\0$\;
Set $\Lambda_0=\emptyset$ and $t=1$\;
\While{Stopping conditions are not achieved}
{
  Find the $\btau_t$ corresponding to \textbf{the most violated constraint} based on $\balpha^{t-1}$\;
  Set $\Lambda_{t}=\Lambda_{t-1}\cup\{\btau_t\}$\;
  \dgr{\textbf{Solve the subproblem} corresponding to $\Lambda_t$:}
\begin{equation}\label{eq:ccp_master_prob}
    \min_{\balpha\in \cA,\theta\in\mbR} ~\theta, ~~\st~ \phi(\balpha,\btau)-\theta\leq 0, \forall \btau \in \Lambda_t,
  \end{equation}
  obtaining the solution $\balpha^t$. Let $t=t+1$\;
}
\end{algorithm}

\par
\dgr{Algorithm \ref{algo:ccp_dual} involves two main steps: finding the most violated constraints and the subproblem optimization step. Since each $\btau\in\Lambda$ corresponds to a constraint, in the $t$-th iteration, we find the most active $\btau_t$ based on $\balpha^{t-1}$ and add it into the active constraint set $\Lambda_t$, which is initialized as an empty set $\emptyset$.
Then we update $\balpha^t$ by solving the subproblem \eqref{eq:ccp_master_prob} with the constraints defined in $\Lambda_t$. The algorithm terminates when the stopping conditions are achieved. }

\par
\dgr{As will be shown later, in each iteration $t$, the activated $\btau_t$ indicates at most $\kappa$ nonzero elements in $\bD_v\bx$ and $\bD_h\bx$. Activating the most active $\btau_t$ is equivalent to finding $\kappa$ indices of the corresponding nonzero elements, which can be recorded in an index set $\cC_t\subseteq [n]$, \ie $\cC_t=\mathrm{supp}(\btau_t)$. For convenience, we define an index set $\cS_t=\cup_{i=1}^t\cC_i$ to record the indices indicated by all $\btau$'s in $\Lambda_t$.
As shown in Section \ref{sec:sub_opt}, due to the difficulty of directly solving the subproblem \eqref{eq:ccp_master_prob} w.r.t. the dual variable $\balpha$ \cite{Tan2014Towards,tan2013mpl}, we will investigate an efficient way to solve the subproblem {w.r.t. to the primal variable} $\bx$ instead.}

\subsection{\kuired{Solving QCLP Subproblem in Primal Form}}
\label{sec:sub_opt}
\dgr{In each iteration of Algorithm \ref{algo:ccp_dual}, after updating the active constraint set, we address subproblem \eqref{eq:ccp_master_prob} with a subset of activated $\btau$'s in $\Lambda_t$. Although the number of constraints in \eqref{eq:ccp_master_prob} is significantly smaller than the original problem \eqref{eq:mptv2_qclp}, directly optimizing it w.r.t. the dual variable $\balpha$ is not easy \cite{Tan2014Towards,tan2013mpl}. However, considering that the activated $\btau$'s in $\Lambda_t$ indicate only a subset of nonzero elements in $\bz$ and $\bD\bx$, we will transform problem \eqref{eq:ccp_master_prob} to an equivalent problem (w.r.t. the primal variables $\bx$) that can be solved much faster \cite{Tan2014Towards,gong2017MPGL}.}

\par
\dgr{Recall that $\cS_t\subseteq [n]$ is defined to record the indices of the nonzero elements indicated by $\btau$'s in $\Lambda_t$, and $\tbtau=[\btau^\T, \btau^\T]^\T \in \{0,1\}^{2n}$ in \eqref{eq:mptv2_nonsplit} is defined to indicate nonzero elements in $\bD\bx\in\mbR^{2n}$.
For convenience, we extend the definition of $\cS_t$ for $\btau$ to a notation $\tcS_t$ for $\tbtau$, \ie, for an $\cS_t$, there is $\tcS_t = \cS_t\cup (\cS_t+n)$ and $(\cS_t+n)=\{i|i=j+n, j\in \cS_t\}$.
Armed with the definition of $\tcS_t$, we transform problem \eqref{eq:ccp_master_prob} into problem \eqref{eq:subproblem_primal} w.r.t. the primal variables $\bx$ and $\bz_{\tcS_t}$ as shown in  Proposition \ref{prop:subproblem}.}
\begin{proposition}
\label{prop:subproblem}
  \dgr{Let $\cS=\cup_{i=1}^t \cC_i$. Assume there is no overlapping element among $\cC_i$'s, problem \eqref{eq:ccp_master_prob} can be addressed by solving
  \begin{equation}\label{eq:subproblem_primal}
  \begin{split}
    \min_{\bx,\bz_{\tcS}}  &~\frac{1}{2}\|\by-\bA\bx\|_2^2 + \lambda\sum\nolimits_{i\in \cS} \|\bC_{i\tcS}\bz_{\tcS}\|_2\\
    \st &~(\bD\bx)_{\tcS}=\bz_{\tcS},(\bD\bx)_{\tcS^c}=\0.
  \end{split}
\end{equation}}
Additionally, the optimal value of $\balpha^*$ under problem \eqref{eq:ccp_master_prob} can be recovered by $\balpha^*=\bxi^*$ where $\bxi^*=\by-\bA\bx^*$.
\end{proposition}
The proof can be found in Appendix. \dgr{In \eqref{eq:subproblem_primal}, the subscript $t$ of $\tcS_t$ is omitted for simplifying the representation. }

\par
\dgr{Proposition \ref{prop:subproblem} implies that activating $\btau$'s in the QCLP problem in \eqref{eq:mptv2_qclp} corresponds to activating a subset of nonzero image gradients in the TV minimization problem in primal, which are indexed by $\cS_t$.
The nonzero gradients activated in each iteration of Algorithm \ref{algo:ccp_dual} are recorded in $\cC_t$.
}

\par
\dgr{In problem \eqref{eq:subproblem_primal}, only a subset of image gradients (\ie $\bD\bx_{\tcS}$ and $\bz_{\tcS}$) indicated by $\tcS$ can be nonzero, the constraint $(\bD\bx)_{\tcS^c}=\0$ thus reduces the uncertainty of the problem, making the optimization easy.
We will introduce an efficient alternating direction method of multipliers (ADMM) based algorithm to handle subproblem \eqref{eq:subproblem_primal} in Section \ref{sec:sub_opt_admm}.}

\subsection{Finding the Most Violated Constraint}
\label{sec:finding_cons}
\dgr{In each iteration of Algorithm \ref{algo:ccp_dual}, we need to find the most active $\btau$ within a large number of elements in $\Lambda$ based on the updated $\balpha^{t-1}$ and the corresponding $\bbeta$.
}

\par
At the optimum of problem \eqref{eq:mptv2_qclp}, the following condition should hold for all $\btau$'s:
\begin{equation}\label{eq:opt_cond_worst}
  \bA^\T\balpha =\bD^\T\bbeta,~ \tau_i \|\bbeta_i\|_2 \leq \lambda,~ \forall i\in [n].
\end{equation}
\dgr{Given an $\balpha$ and $\bbeta$, a subvector $\bbeta_i$ with the larger $\|\bbeta_i\|_2$ (and $\tau_i=1$) violates the optimality condition \eqref{eq:opt_cond_worst} the more.
In this sense, the most active $\btau$ indicates \emph{the largest number} of $\|\bbeta_i\|_2$ with \emph{the largest values}.
Due to the constraint $\|\btau\|_0\leq \kappa$, with a given $\bbeta$, we construct the most active $\btau$ by finding the $\kappa$ largest $\|\bbeta_i\|_2$, and then setting the corresponding $\tau_i$ to 1 and the rest to 0.}

\par
\dgr{However, we cannot directly obtain the solution of the dual variables $\balpha$ and $\bbeta$, since the optimization of subproblem \eqref{eq:ccp_master_prob} is achieved by solving the problem in \eqref{eq:subproblem_primal} w.r.t. the primal variable $\bx$, as introduced in Section \ref{sec:sub_opt}.
Thus, in each iteration, after solving the subproblem in primal (\ie problem \eqref{eq:subproblem_primal}), we firstly recover $\balpha$ from the solution $\bx$ via $\balpha=\bxi=\by-\bA\bx$, and then reconstruct $\bbeta$ based on the equality constraint $\bA^\T\balpha =\bD^\T\bbeta$ in \eqref{eq:ls_worst_case_ana}. }
Recall that $\balpha\in \mbR^{n}$ and $\bbeta = [\bbeta_v^\T, \bbeta_h^\T]^\T\in \mbR^{2n}$. It is thus highly ill-posed to recover $\bbeta$ from $\balpha$ by solving the linear system $\bA^\T\balpha =\bD^\T\bbeta$.
\dgr{We use the standard $\ell_2$-norm based regularizer $\|\bbeta\|_2^2$ to alleviate the ill-posed nature. Relying on the $\ell_2$-norm based regularizer, we can obtain a closed-form solution of $\bbeta$ efficiently, which simplifies the optimization.
We thus try to obtain $\bbeta$ approximately by solving:}
\begin{equation}\label{eq:ls_worst_case_ana}
  \min_{\bbeta} \frac{1}{2} \|\bD^\T\bbeta-\bA^\T\balpha\|_2^2 + \frac{r}{2}\|\bbeta\|_2^2,
\end{equation}
where $r>0$ is a penalty parameter. Even though $\bD$ is a concatenation of two parts, \ie $\bD=[\bD_v^\T, \bD_h^\T]^\T$, we can still efficiently solve \eqref{eq:ls_worst_case_ana} using Fast Fourier Transforms (FFTs) (see details in Appendix). Then we can easily obtain $\bbeta_i,\forall i\in [n]$ based on $\bbeta_i=\bC_i \bbeta$. For convenience, we define a vector $\bg$, where $g_i=\|\bbeta_i\|_2$.

\par
\dgr{In the $t$-th iteration of Algorithm \ref{algo:ccp_dual}, with the recovered $\bbeta$, we can find the most active $\btau_t$.
As introduced in Section \ref{sec:qclp_opt}, in practice, we recode the $\kappa$ indices indicated by $\btau_t$ into a set $\cC_t$, \ie $\cC_t=\mathrm{supp}(\btau_t)$ and let $\cS_t=\cup_{i=1}^t\cC_i$ record the indices of the all activated $\btau$'s.
\dgrr{In general, once an element has been activated and added into $\cS_{t-1}$, it is unlikely to be activated again and selected in $\cC_t$ in the subsequent iterations. However, if the subproblem is not solved accurately, some of the activated elements may have large values of $g_i$ (\ie $\|\bbeta_i\|_2$) and be activated again.}
To avoid overlapping components among $\cC_t$ in practice, we form $\cC_t$ from $[n]\setminus\cS_{t-1}$.
The algorithm for finding the most violated constraint is summarized in Algorithm \ref{algo:worst-case_ana}.}

\begin{algorithm}[htp]\label{algo:worst-case_ana}
\caption{Finding the Most Violated Constraint}
\KwIn{$\balpha$, $\kappa$, $\bA$ and regularizer parameter $r$.}\dgr{Recover $\bbeta$ from $\balpha$ by solving problem \eqref{eq:ls_worst_case_ana}}\;
Let $\bbeta_i=\bC_i\bbeta$ and calculate $g_i=\|\bbeta_i\|_2,\forall i\in[n]$\;
Initialize $\btau=\0$, and find the $\kappa$ largest $g_i$'s\;
Set $\tau_i$ corresponding to the $\kappa$ largest $g_i$'s to 1\;
Form $\cC$, and return $\btau$ and $\cC$.
\end{algorithm}

\begin{figure*}[!t]
\centering
\subfigure[Input $\by$ and $\bk$]{
\begin{minipage}[b]{.16\textwidth}
\centerline{
\begin{overpic}[width=1\textwidth]
{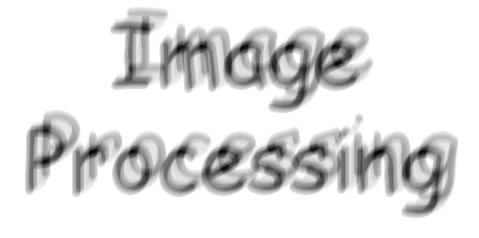}
\put(68,45){\includegraphics[width=0.24\linewidth]{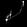}}
\put(69,47){\color{white}{\bf $\bk$}}
\end{overpic}}
\end{minipage}
}
\hspace{-0.4cm}
\subfigure[Ground truth $\bx^*$]{
\begin{minipage}[b]{.16\textwidth}
\centerline{
\begin{overpic}[width=1\textwidth]
{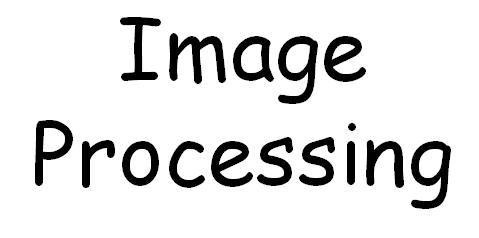}
\end{overpic}}
\end{minipage}
}
\hspace{-0.4cm}
\subfigure[Results of TV-ADMM]{
\begin{minipage}[b]{.16\textwidth}
\centerline{
\begin{overpic}[width=1\textwidth]
{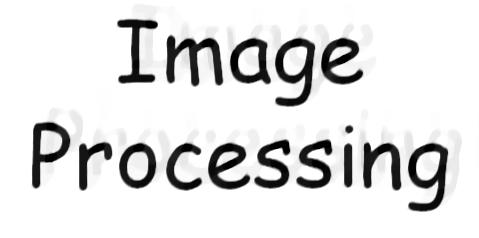}
\end{overpic}}
\end{minipage}
}
\hspace{-0.4cm}
\subfigure[Results of MPTV]{
\begin{minipage}[b]{.16\textwidth}
\centerline{
\begin{overpic}[width=1\textwidth]
{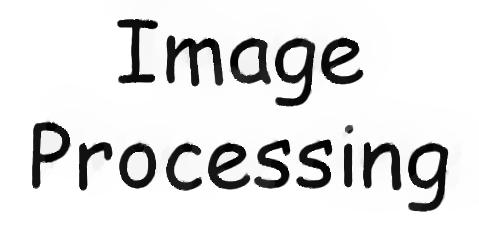}
\end{overpic}}
\end{minipage}
}
\hspace{-0.4cm}
\subfigure[PSNR]{
\begin{minipage}[b]{.16\textwidth}
\centerline{
\begin{overpic}[width=1\textwidth]
{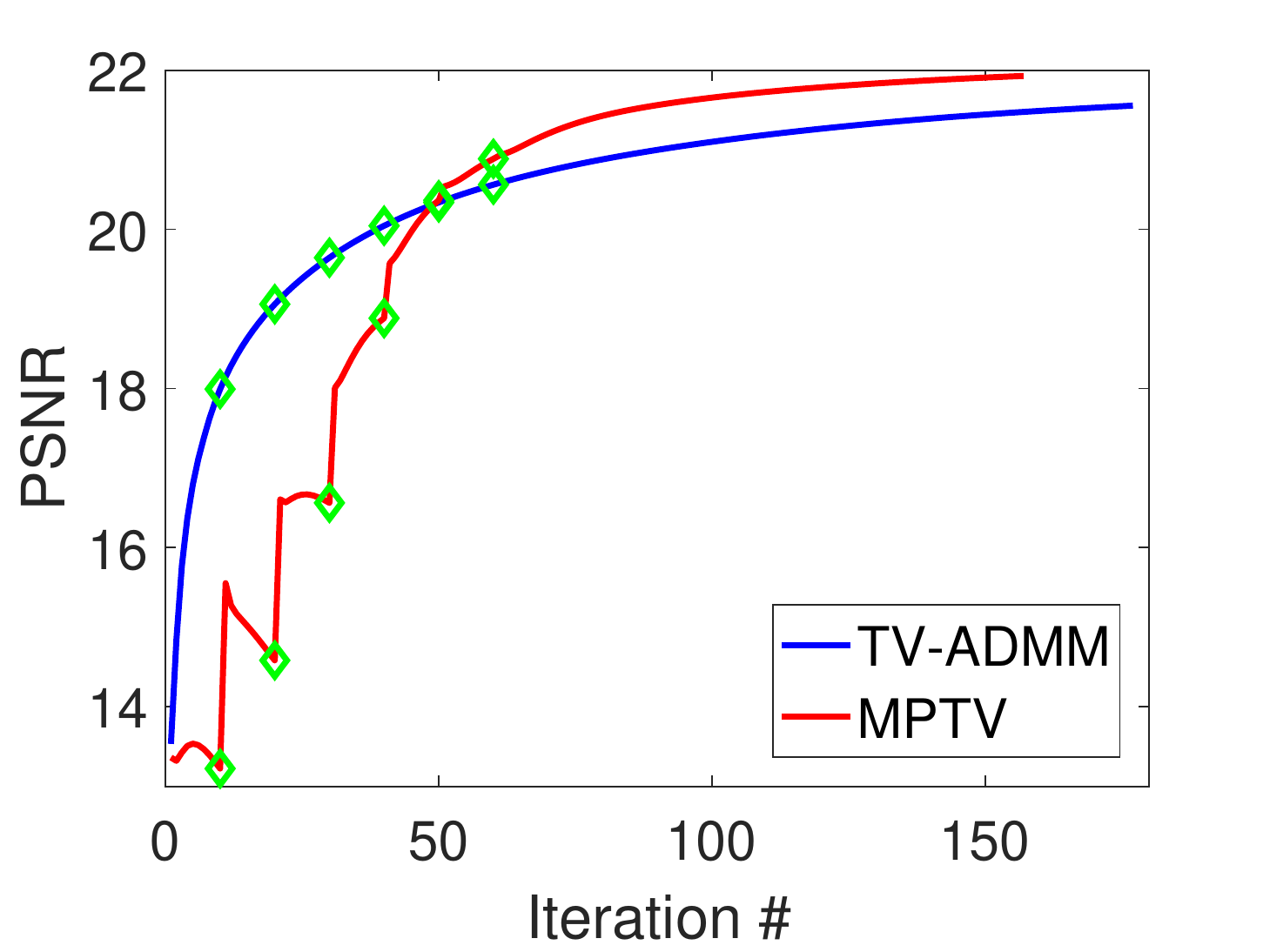}
\end{overpic}}
\end{minipage}
}
\hspace{-0.4cm}
\subfigure[SSIM]{
\begin{minipage}[b]{.16\textwidth}
\centerline{
\begin{overpic}[width=1\textwidth]
{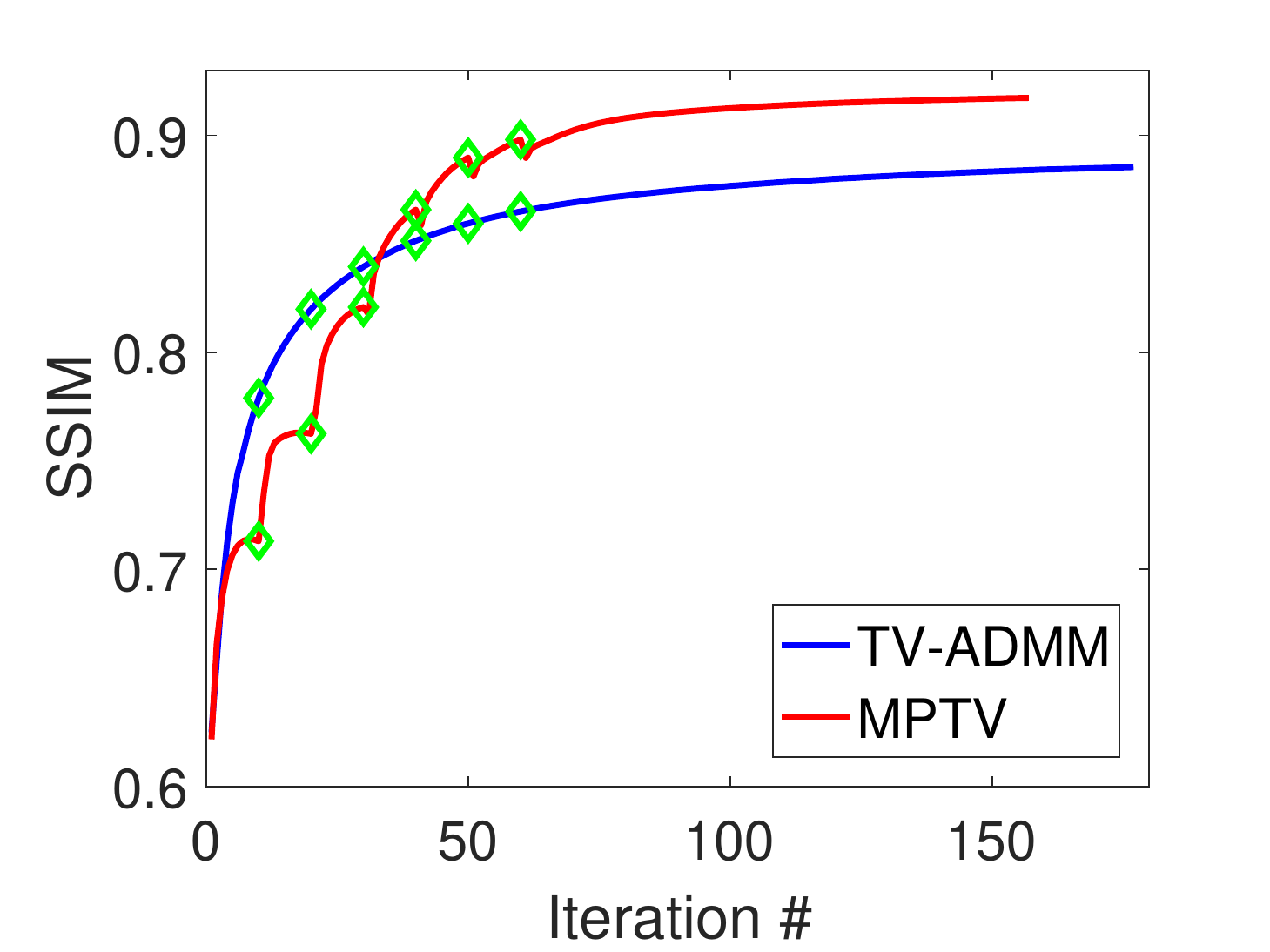}
\end{overpic}}
\end{minipage}
}
\hfill
\vspace{-0.05cm}
\subfigure[Intermediate results $\bx^t$ of MPTV (from $\bx^1$ to $\bx^6$)]{
\begin{tabular}[]{c}
\begin{minipage}[b]{.16\textwidth}
\centerline{
\begin{overpic}[width=1\textwidth]
{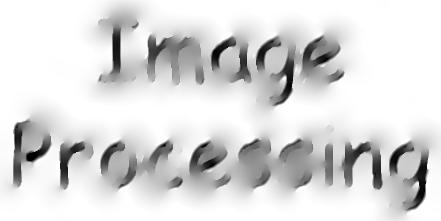}
\end{overpic}}
\end{minipage}
\hspace{-0.2cm}
\begin{minipage}[b]{.16\textwidth}
\centerline{
\begin{overpic}[width=1\textwidth]
{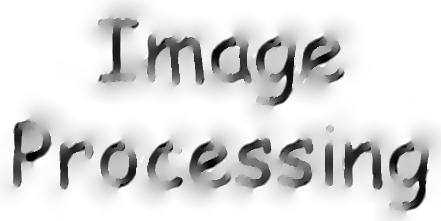}
\end{overpic}}
\end{minipage}
\hspace{-0.2cm}
\begin{minipage}[b]{.16\textwidth}
\centerline{
\begin{overpic}[width=1\textwidth]
{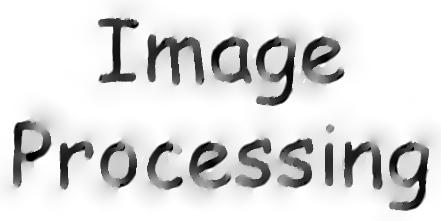}
\end{overpic}}
\end{minipage}
\hspace{-0.2cm}
\begin{minipage}[b]{.16\textwidth}
\centerline{
\begin{overpic}[width=1\textwidth]
{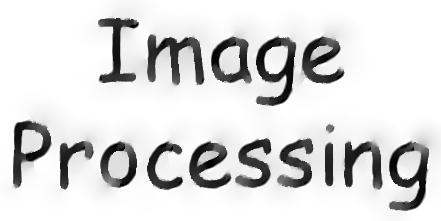}
\end{overpic}}
\end{minipage}
\hspace{-0.2cm}
\begin{minipage}[b]{.16\textwidth}
\centerline{
\begin{overpic}[width=1\textwidth]
{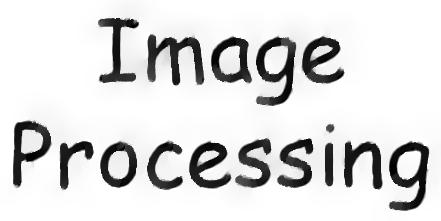}
\end{overpic}}
\end{minipage}
\hspace{-0.2cm}
\begin{minipage}[b]{.16\textwidth}
\centerline{
\begin{overpic}[width=1\textwidth]
{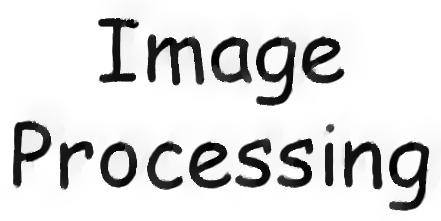}
\end{overpic}}
\end{minipage}
\hspace{-0.2cm}
\end{tabular}
}
\hfill
\vspace{-0.05cm}
\subfigure[Intermediate $\bg^t$ in MPTV for gradient activation (from $\bg^0$ to $\bg^5$, corresponding to the $\bx^1$ to $\bx^6$ in (g))]{
\begin{tabular}[]{c}
\begin{minipage}[b]{.16\textwidth}
\centerline{
\begin{overpic}[width=1\textwidth]
{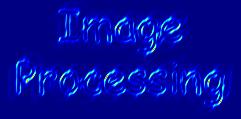}
\end{overpic}}
\end{minipage}
\hspace{-0.2cm}
\begin{minipage}[b]{.16\textwidth}
\centerline{
\begin{overpic}[width=1\textwidth]
{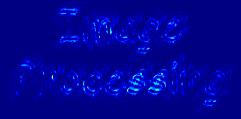}
\end{overpic}}
\end{minipage}
\hspace{-0.2cm}
\begin{minipage}[b]{.16\textwidth}
\centerline{
\begin{overpic}[width=1\textwidth]
{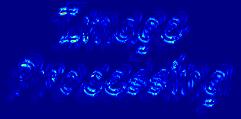}
\end{overpic}}
\end{minipage}
\hspace{-0.2cm}
\begin{minipage}[b]{.16\textwidth}
\centerline{
\begin{overpic}[width=1\textwidth]
{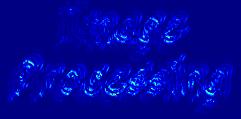}
\end{overpic}}
\end{minipage}
\hspace{-0.2cm}
\begin{minipage}[b]{.16\textwidth}
\centerline{
\begin{overpic}[width=1\textwidth]
{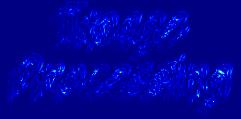}
\end{overpic}}
\end{minipage}
\hspace{-0.2cm}
\begin{minipage}[b]{.16\textwidth}
\centerline{
\begin{overpic}[width=1\textwidth]
{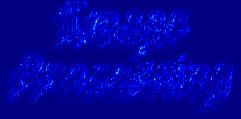}
\end{overpic}}
\end{minipage}
\hspace{-0.2cm}
\end{tabular}
}
\caption{An example for the MPTV method that is presented in Algorithm \ref{algo:mptv_primal}. (a) Input blurred image and blur kernel containing estimation error. (b) Ground truth image. (c) and (d) Results of TV-ADMM and the proposed MPTV, respectively. (e) and (f) show the convergence of TV-ADMM and MPTV. MPTV converges to a result with higher PSNR and SSIM values. (g) Intermediate results of MPTV, $\bx^t$ in Algorithm \ref{algo:mptv_primal}, corresponding to the green dots in (e) and (f). (h) shows intermediate $\bg^t$ in MPTV, corresponding to the $\bx^t$ in
(g).
Note that MPTV \kuire{is
  terminated} after 6 iterations. }
\label{fig:results_ip}
\end{figure*}

\subsection{Matching Pursuit for Minimizing TV Model}
\dgr{Based on Proposition \ref{prop:subproblem} and Algorithm \ref{algo:worst-case_ana}, we can implement our algorithm in primal form as summarized in Algorithm \ref{algo:mptv_primal}, which is referred as the \emph{matching pursuit total variation (MPTV)} algorithm.}
\dgr{In each iteration of Algorithm \ref{algo:mptv_primal}, we recover the dual variables $\balpha$ and $\bbeta$ as well as find the most active $\btau_t$ (relying on Algorithm \ref{algo:worst-case_ana}), and then record all indices indicated by the activated $\btau$'s into $\cS_t$. As shown in Fig. \ref{fig:results_ip} (h), the nonzero gradients are activated according to the $\bg^t$ (\ie $\|\bbeta_i\|_2, \forall i\in[n]$), in which the elements with the largest values indicate the most significant gradients in the image.
During the iterations, the image $\bx$ is recovered based on the gradually activated gradients (indicated by $\cS_t$).
Since no $\btau$ is involved at the initial stage, \eg $\cS_0=\emptyset$, we initialize $\bx$ via solving \eqref{eq:subproblem_primal} given $\cS=\emptyset$, \ie letting $\bx^0=\1\otimes(\sum_{i=1}^n y_i/n)$.}

\par  \textbf{Refinement for $\cS$}.
\dgr{
To boost the robustness of the MPTV algorithm for real-world images in practice,
we can take several additional simple steps in Algorithm \ref{algo:mptv_primal} to refine the $\cS$, which ensures we can make a conservative estimate of the image pixels.}
To conduct refinement in 2D image coordinates, we define a 2D map $\bM$ with $\bM(i,j)=1$ for $(i,j)\in\cS$ and 0 for others, where
by a slight abuse of notation, we let
$\cS$ denote the locations of activated components in 2D coordinate system.
Firstly, we perform a binary morphology dilation after an erosion operation, such that $\bM=(\bM\ominus \bR_e)\oplus\bR_d$, where $\ominus$ and $\oplus$ denote the binary erosion and dilation, respectively. The structuring elements for both $\bR_e$ and $\bR_d$ are disks with radius 3. This step helps to remove isolated active components in $\cS$. Fewer artifacts thus arise from fragmented active components. \dgr{Since the ideally sharp edges rarely appear in real-world images, this step does not remove the useful and significant structures while removing the isolated components. }
Secondly, similar to \cite{Whyte2014Deblurring}, to improve robustness for natural images and alleviate the visible boundaries between the active and inactive components,
we blur the mask $\bM$ slightly using a Gaussian filter with standard deviation 3 pixels. After refining the $\bM$, we update $\cS$ by recording the indices of the nonzero elements of $\bM$ in $\cS$. \dgr{Note that the above operations are not necessary for images with ideally sparse gradients.}

\begin{algorithm}[htp]\label{algo:mptv_primal}
\caption{MPTV for Image Deconvolution}
\KwIn{Observation $\by$, parameter $\lambda$ and $\rho$.}Initialize $\bx^0=\1\otimes(\sum_{i=1}^n y_i/n)$, $\balpha^0=\by-\bA\bx^0$, $\cS_0=\emptyset$\;
Set iteration index $t=1$\;
\While{Stopping conditions are not achieved}
{
  \dgr{\textbf{Finding the most violated constraint:} Find the most active $\btau_t$ based on $\balpha^{t-1}$ via Algorithm \ref{algo:worst-case_ana}, and record the corresponding indices into $\cC_{t}$\;
  Let $\cS_{t} = \cS_{t-1}\cup \cC_{t}$\;
  \textbf{Subproblem optimization:} Solve subproblem \eqref{eq:subproblem_primal} via the ADMM in Algorithm \ref{algo:subprob_admm}, obtaining $\bx^{t}$. Let $\balpha^{t}=\by-\bA\bx^{t}$, and let $t=t+1$\;}
}
\end{algorithm}

\subsection{Setting of Parameter $\kappa$}
\label{sec:kappa_setting}
The integer $\kappa$ parameter controls the number of activated gradients in each iteration. In general, a large $\kappa$ can decrease the iteration number of MPTV algorithm, and a very small $\kappa$ helps to prevent activating too many gradient components. \kuire{Thus, the value of $\kappa$} may affect the quality of the recovered image. \dgr{A sensitive study of $\kappa$ is conducted in Section \ref{sec:para_sens_exp}.} {To balance between} the efficiency and the performance, we here provide a strategy for setting $\kappa$.

Recall that $\kappa$ reflects a rough knowledge of the support of $\bD_v\bx$ and $\bD_h\bx$ together, that is, $K=\card(\supp(\bD_v\bx) \cup \supp(\bD_h\bx))$. As $K$ is unknown, we first reconstruct $\bbeta^0$ from
the initialization $\balpha^0$, then obtain $\bg^0$ by letting $g^0_i=\|\bbeta^0_i\|_2$ and normalize $\bg^0$ as $\bg^0/\|\bg^0\|_2$. Following the thresholding strategy in previous work \cite{gong2017MPGL}, we set $\kappa$ as the number of of elements in $\bg^0$ larger than $\zeta\|\bg^0\|_\infty$. In practice, $\zeta\geq 0.6$ works well for image deconvolution.

\subsection{\kui{Solution Bias and Early Stopping}}
\label{sec:early_stopping}
In general
TV based image deconvolution, many existing methods are sensitive to the choice of the $\lambda$ and face a dilemma -- a too big a $\lambda$
produces
a solution with sparse gradients and helps to resist noise, but underfits the data and causes over-smoothness in the recovered image, while a small $\lambda$ fits the observation well but reduces the regularization strength and may
result in an estimate which suffers from
ringing artifacts and/or noise.
As a result it is often necessary to
intensively search for an appropriate $\lambda$ for a specific observation by considering the trade-off between the regularization strength and the solution bias. This is often impractical and inconvenient for users without specific knowledge.
Due to the nature of our optimization scheme, we bypass this dilemma to a large degree, by proposing an early stopping condition below.

\par
\kuired{Recall that the MPTV algorithm incrementally activates the nonzero elements in $\bD\bx$, the value of $\Omega_{\text{TV}}(\bx)$ increases from 0 iteratively. By defining $\psi(\bx) = \|\by-\bA\bx\|_2^2+\lambda\Omega_{\text{TV}}(\bx)$, the algorithm will be stopped when $$\kuire{|\psi(\bx^{t-1})-\psi(\bx^t)|}/\psi(\bx^0)\leq \epsilon,$$ where $\epsilon$ is a tolerance variable.} 

\par
\dgr{A too small $\epsilon$ may lead to too many iterations and activate too many $\btau$'s due to the noise in $\by$, resulting in a non-sparse solution.}
\dgr{We thus use an early stopping condition with a relatively large $\epsilon=1\times 10^{-3}$ (and maximum iteration number as 7) to prevent MPTV from too many iterations, which helps to obtain a solution $\bx$ with small total variation, \ie small $\Omega_{\TV}(\bx)$, and sparse gradients.
Since the early stopping strategy helps to ensure the sparsity of image gradients, we thus can use a relatively small $\lambda$ to reduce the regularization bias in solution.}
As a result, \kuired{MPTV performs well for a wide range of $\lambda$}, which saves a lot of computational cost for hyper-parameter tuning (see Fig. \ref{fig:para_lambda_sens}). \kuired{The sensitive study of $\epsilon$ can be found in Section \ref{sec:para_sens_exp}.}

\subsection{\dgr{Discussions on Ringing Suppression by MPTV}}
As is outlined above, \dg{the deconvolution results often suffer from} ringing artifacts in flat areas and near strong edges (Fig. \ref{fig:results_ip} (c)), \dgrr{which may be caused by errors in the blur kernel estimate \cite{Shan2008High}, the Gibbs phenomenon \cite{yuan2007image}, etc, as discussed in Section \ref{sec:related_work_deconv}. }
\dg{The classical TV based model, unfortunately, is inadequate for avoid the undesired ringing artifacts.}

\par
For the deconvolution process given a fixed blur kernel $\bk$ and corresponding convolution matrix $\bA$, an image $\tbx$ containing ringing artifacts often lies in a domain $\widetilde{\mbX}=\{\bx~|~ \bx=\widehat{\bx}+\br, \|\widehat{\bx}-\bx^*\|_2\approx 0, \|\bA\br\|_2 \approx 0 \}$, where $\bx^*$ denotes the ground truth image, and $\br$ is an error term for artifacts.
An artifact term $\br$ corresponding to medium frequency ripples \cite{Whyte2014Deblurring} lies close to $\bA$'s nullspace.
\dgr{Since for any $\tbx\in\widetilde{\mbX}$, there is $\|\by-\bA\tbx\|_2\approx 0$, a solution only minimizing the $\ell_2$-norm based data fidelity term is very likely in $\widetilde{\mbX}$.}
{In the conventional model \eqref{eq:mptv2_nonsplit}, the strength of the TV regularizer is only controlled by the parameter $\lambda$.
Although a large weight of the TV regularizer can suppress the potential ringing artifacts, it also over smooths the significant edges and textures in the recovered image \cite{wang2008new,mosleh2014image}.
When $\lambda$ is not large enough, minimizing the TV based objective cannot produce images with sparse gradients and leads to results lying in $\widetilde{\mbX}$ with visible ringing.}
\dgrr{It is hard to find an accurate solution with less ringing artifacts only by tuning the parameter $\lambda$.}

\dgr{Instead of directly optimizing the objective in \eqref{eq:problem_model_org}, MPTV gradually activates the most significant gradients according to the fitting errors as shown in Fig. \ref{fig:results_ip}.} \dgrr{As shown in \cite{yuan2007image}, the ringing artifacts are less significant in the fitting error, due to the smaller magnitudes. The possible artifacts can be naturally suppressed by the gradient activation step and the early stopping condition. Thus, benefiting from the optimization scheme, the proposed MPTV can suppress the ringing artifacts without the need of large $\lambda$. Recall that MPTV can also produce the desired sparse gradients with a small $\lambda$, which is consistent with that for ringing suppression.}
The example in Fig. \ref{fig:results_ip} and the experimental results in Section \ref{sec:exp} demonstrate that by selectively, and inhomogeneously, activating the appropriate gradients the intended effect of the regularizer can be maintained without causing the unwanted ringing.

\subsection{Convergence Analysis}
\label{sec:analysis}
Similar to the analysis in \cite{Tan2014Towards} and \cite{gong2017MPGL}, it can be proved that the $\theta$'s generated through progressive iterations of Algorithm \ref{algo:ccp_dual} increase monotonically. As a result, the following \dgrrr{proposition} shows that MPTV converges to a global solution of \eqref{eq:mptv2_qclp}.
\dgrrr{\begin{proposition}
\black{Let $\{(\balpha_t, \theta_t)\}$ be the sequence generated by Algorithm \ref{algo:ccp_dual}. If the most violated constraint finding problem and subproblem \eqref{eq:ccp_master_prob} can be solved, the sequence $\{(\balpha_t, \theta_t)\}$ will converge to an optimal solution of problem \eqref{eq:mptv2_qclp}.}
\label{thm:convergence}
\end{proposition}}

The proof can be found in Appendix.

\par
\kuiredd{In practice, we may use an early stopping to prevent MPTV from activating too many constraints, which will help to obtain images with sparse gradients.}

\section{\dgr{Efficient Optimization of Subproblem \eqref{eq:subproblem_primal}}}
\label{sec:sub_opt_admm}
\dgr{
In MPTV algorithm (\ie Algorithm \ref{algo:mptv_primal}), after updating $\cS_t$, we need to solve the subproblem \eqref{eq:subproblem_primal}.
To handle the equality constraints, we apply the alternating direction method of multipliers (ADMM) \cite{boyd2004convex} which can be up to a few orders of magnitude faster than solving \eqref{eq:ccp_master_prob} directly \cite{Tan2014Towards}.}
Specifically, we iteratively update the primal and dual variables of the augmented Lagrangian function \cite{boyd2011ADMM} of \eqref{eq:subproblem_primal}.
Let subvectors $\bgamma_{\tcS}$ and $\bgamma_{\tcS^c}$ be the dual variables w.r.t. the two constraints, respectively. By introducing a dual variable $\bgamma\in \mbR^{2n}$ and a positive penalty parameter $\rho>0$, we obtain the augmented Lagrangian for \eqref{eq:subproblem_primal}:
\dgr{
\begin{equation}\label{eq:subprob_admm_augLag}
\begin{split}
  \cL&(\bx, \bz_{\tcS}, \bgamma) = \frac{1}{2}\|\by-\bA\bx\|_2^2 + \lambda\sum\nolimits_{i\in \cS} \|\bC_{i\tcS}\bz_{\tcS}\|_2 \\
  &+ \bgamma_{\tcS}^\T((\bD\bx)_{\tcS}-\bz_{\tcS}) + \bgamma_{\tcS^c}^\T(\bD\bx)_{\tcS^c} \\
  &+ \frac{\rho}{2} \|(\bD\bx)_{\tcS}-\bz_{\tcS}\|_2^2 + \frac{\rho}{2} \|(\bD\bx)_{\tcS^c}\|_2^2.
\end{split}
\end{equation}}
Then, we can solve for $\bx$
by iteratively minimizing $\cL(\cdot)$ w.r.t. $\bx$ and $\bz_{\tcS}$, and updating $\bgamma$. Let $k$ denote the iteration number. Applying ADMM, we carry out the following steps at each iteration:

\par \noindent \textbf{Step 1} Compute $\bz_{\tcS}^{k+1}$ with fixed $\bx^{k}$ and $\bgamma^{k}$ by solving: \begin{equation}\label{eq:subprob_admm_z}
  \min_{\bz_{\tcS}} \lambda\!\sum_{i\in \cS} \|\bC_{i\tcS}\bz_{\tcS}\|_2 +\frac{\rho}{2} \|(\bD\bx^k)_{\tcS}-\bz_{\tcS} + \frac{1}{\rho}\bgamma_{\tcS}^k\|_2^2,
  \end{equation}
which has an unique closed-form solution.
For $\forall i\in \cS$, let $\bmu_i = \bC_{i\tcS}(\bD\bx^k+\frac{1}{\rho}\bgamma^k)$. We can obtain the solution of \eqref{eq:subprob_admm_z} by a two-dimensional shrinkage \cite{wang2008new} for $\forall i\in\cS$:
\begin{equation}\label{eq:admm_l2tv_shrinkage}
  \bC_{i\tcS}\bz_{\tcS}^{k+1} = \max\left\{ \|\bmu_i\|_2 -\frac{\lambda}{\rho}, 0 \right\} \frac{\bmu_i}{\|\bmu_i\|_2}, \forall i\in \cS,
\end{equation}
where a convention $0\cdot(0/0)=0$ is followed. Following this, for convenience, we define a temporary variable $\bz'=[{\bz'}_v^{\T}, {\bz'}_h^{\T}]^\T$, and let $\bz'_{\tcS}=\bz_{\tcS}^{k+1}$ and $\bz'_{\tcS^c}=\0$.

\par \noindent \textbf{Step 2} Update $\bx^{k+1}$ by solving $\min_{\bx} \cL(\bx, \bz_{\tcS}^{k+1}, \bgamma^{k})$ w.r.t. $\bx$:
\begin{equation}\label{eq:subprob_admm_x}
  \min_{\bx} \frac{1}{2}\|\by-\bA\bx\|_2^2 + \frac{\rho}{2}\|\bD\bx-\bz'+\frac{1}{\rho}\bgamma^{k}\|_2^2,
\end{equation}
which is quadratic in $\bx$. Let $\bnu=[\bnu_v^\T, \bnu_h^\T]^\T$ where $\bnu_v = \bz'_v-\frac{1}{\rho}\bgamma_v^{k}$ and $\bnu_h = \bz'_h-\frac{1}{\rho}\bgamma_h^{k}$. The minimizer can be recovered from the normal equation:
\begin{equation}
  \left[ \bA^\T\bA + \rho(\bD^\T\bD) \right] \bx = \bA^\T\by+\rho\bD^\T\bnu.
\end{equation}
Under the periodic boundary condition, we can obtain the solution via  FFTs \cite{wang2008new}:
\begin{equation}\label{eq:subprob_admm_x_fft}
  \bx^{k+1}\!\! =\!\! \cF^{-1}\!\!\left(\!\!\frac{\overline{\cF(\bA)}\cF(\by)\!\! +\!\!\rho(\overline{\cF(\bD_v)}\cF(\bnu_v)\!\! +\! \overline{\cF(\bD_h)}\cF(\bnu_h))} {\overline{\cF(\bA)}\cF(\bA)\!\!+\!\! \rho(\overline{\cF(\bD_v)}\cF(\bD_v)\!\!+\!\!\overline{\cF(\bD_h)}\cF(\bD_h))}\!\!\right),
  \nonumber
\end{equation}
where $\cF(\cdot)$ and $\cF^{-1}(\cdot)$ denote the Fourier transform and the inverse transform, respectively, $\overline{\cF(\cdot)}$ denotes the complex conjugate of $\cF(\cdot)$, and the multiplication and division are all component-wise operators.

\par \noindent \textbf{Step 3} Update the dual variable $\bgamma$:
\begin{equation}\label{eq:subprob_admm_beta}
  \bgamma^{k+1} = \bgamma^k + \rho (\bD\bx^{k+1}-\bz').
\end{equation}

\par
The ADMM method for solving problem \eqref{eq:subproblem_primal} is summarized in Algorithm \ref{algo:subprob_admm}.
As shown in Proposition \ref{prop:subproblem}, the dual variable can be
recovered
by $\balpha^*=\bxi^*=\by-\bA\bx^*$.

\begin{algorithm}[htp]\label{algo:subprob_admm}
\caption{ADMM for Solving Subproblem \eqref{eq:subproblem_primal}}
\KwIn{Observation $\by$, parameter $\lambda$ and $\rho$, initialization of image $\bx^0$, index set $\cS_t$.}Initialize $\bbeta^0=\0$. Set iteration number as $k=0$\;
\While{Stopping conditions are not achieved}
{
  Compute $\bz_{\tcS}^{k+1}$ according to \eqref{eq:admm_l2tv_shrinkage}\;
  Generate $\bz'$ by letting $\bz'_{\tcS}=\bz_{\tcS}^{k+1}$ and $\bz'_{\tcS^c}=\0$\;
  Compute $\bx^{k+1}$ by solving problem \eqref{eq:subprob_admm_x}\;
  Update $\bgamma^{k+1}$ according to \eqref{eq:subprob_admm_beta}\;
  If the stopping condition is achieved, stop\;
  Let $k=k+1$\;
}
\end{algorithm}
\kui{\textbf{Stopping conditions of Algorithm \ref{algo:subprob_admm}}.
\dgr{For Algorithm \ref{algo:subprob_admm} applied in solving subproblem \eqref{eq:subproblem_primal}, by defining $\varphi(\bx)=\|\by-\bA\bx\|_2$, \kuired{the algorithm is stopped when
$$|\varphi(\bx^{t-1})-\varphi(\bx^t)|/\varphi(\bx^0)\leq \epsilon_{\text{in}},$$}where $\epsilon_{\text{in}}$ is a tolerance value.}}
This stopping condition is adapted from \cite{goldstein2009split}.
In practice, a sufficiently small $\epsilon_{\text{in}}$ is enough to find the most-violated constraint. We suggest setting $\epsilon_{\text{in}}=0.001$ and the maximum iteration to $100$.
\dgrrr{To avoid the possible early stopping issue in ADMM at the very beginning iterations, we can let the algorithm be stopped only after a minimal number of iterations. }

\section{Experiments}
\label{sec:exp}

In this section, we evaluate the performance of our method against the state-of-the-art image deconvolution methods on both synthetic data and real-world images.
To \kuire{make a comprehensive} study of the proposed method,  \kuire{we consider different types of blur kernels (PSFs) and images.}
We will also study the sensitivity of the MPTV algorithm to \kuire{noises}, blur level and different parameter settings, and \kuire{demonstrate}
its ability to suppress
ringing artifacts.
All experiments are conducted using MATLAB on a desktop computer with an Intel Core i5 CPU with 8GB of RAM.
We \dgr{use brute-force search for}
the parameters to obtain the best performance for every method.

\subsection{Experimental Settings}
\subsubsection{Synthetic dataset generation}
\kuire{\emph{Three synthetic datasets} are generated based} on sharp images $\bx^*$ belonging to three  different types -- images with sparse gradients, text images with near sparse gradients and natural images. We \kuire{use a set of different blur kernels to construct different $\bA$'s}. \kuire{For each $\bA$} we generate the blurry testing images $\by$'s \kuire{by the imaging model \eqref{eq:imaging_model}}  $\by=\bA\bx^*+\bn$, where $\bn$ denotes the additive Gaussian white noise with noise level\footnote{An image $\by$ with $(100\times \sigma) \%$ Gaussian noise is generated by adding noise with zero mean and standard derivation $\sigma$ for image $\bA\bx^*$ with $[0,1]$ intensity range. } $0.3\%$ unless stated otherwise. \kuire{All synthetic blurry images in these three datasets are generated using a set of \emph{different blur kernels (PSFs)}, including the
Gaussian blur kernel, disk blur kernel, linear motion blur kernel and motion blur kernels with different sizes from Levin \etal's dataset \cite{levin2009understanding}, as shown in Fig. \ref{fig:syn_k_imgsp} (a).} A nonperiodic boundary condition \cite{sorel2012boundary} is used for imitating the real blurred image. We will discuss more details of the datasets in the following.

\begin{figure}[!t]
\centering
\hspace{-0.4cm}
\subfigure[Blur kernels for synthesizing data]{
\begin{tabular}[]{c}
\begin{minipage}[b]{.05\textwidth}
\begin{overpic}[width=1\textwidth]
{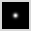}
\put(66,12){\scriptsize \color{white}{\bf\#1}}
\end{overpic}
\end{minipage}
\begin{minipage}[b]{.05\textwidth}
\begin{overpic}[width=1\textwidth]
{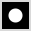}
\put(66,12){\scriptsize \color{white}{\bf\#2}}
\end{overpic}
\end{minipage}
\begin{minipage}[b]{.05\textwidth}
\begin{overpic}[width=1\textwidth]
{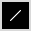}
\put(66,12){\scriptsize \color{white}{\bf\#3}}
\end{overpic}
\end{minipage}
\begin{minipage}[b]{.05\textwidth}
\begin{overpic}[width=1\textwidth]
{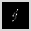}
\put(66,12){\scriptsize \color{white}{\bf\#4}}
\end{overpic}
\end{minipage}
\\
\hfill
\begin{minipage}[b]{.05\textwidth}
\begin{overpic}[width=1\textwidth]
{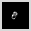}
\put(66,12){\scriptsize \color{white}{\bf\#5}}
\end{overpic}
\end{minipage}
\begin{minipage}[b]{.05\textwidth}
\begin{overpic}[width=1\textwidth]
{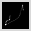}
\put(66,12){\scriptsize \color{white}{\bf\#6}}
\end{overpic}
\end{minipage}
\begin{minipage}[b]{.05\textwidth}
\begin{overpic}[width=1\textwidth]
{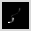}
\put(66,12){\scriptsize \color{white}{\bf\#7}}
\end{overpic}
\end{minipage}
\begin{minipage}[b]{.05\textwidth}
\begin{overpic}[width=1\textwidth]
{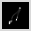}
\put(66,12){\scriptsize \color{white}{\bf\#8}}
\end{overpic}
\end{minipage}
\end{tabular}
}
\hspace{-0.7cm}
\subfigure[Synthetic image]{
\begin{tabular}[]{c}
\begin{minipage}[b]{.105\textwidth}
\begin{overpic}[width=1\textwidth]
{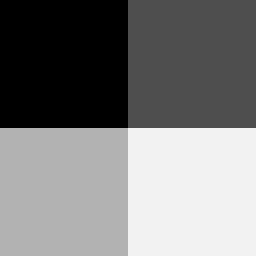}
\put(80,7){\color{black}{\bf$\bx^*$}}
\end{overpic}
\end{minipage}
\begin{minipage}[b]{.105\textwidth}
\begin{overpic}[width=1\textwidth]
{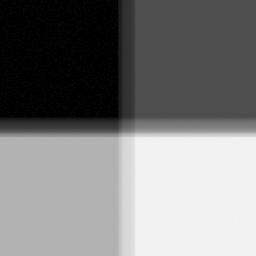}
\put(80,10){\color{black}{\bf\by}}
\end{overpic}
\end{minipage}
\end{tabular}
}
\caption{(a) Blur kernels (PSFs) for generating synthetic data. \#1: $25\times 25$ Gaussian blur kernel with standard derivation 1.6; \#2: $15\times 15$ disk blur kernel; \#3: $11\times 11$ linear motion blur with length 15 and angle $45^\circ$; \#4-\#8: Motion blur kernel from Levin \etal's dataset \cite{levin2009understanding} (squared kernels with side length 19, 15, 27, 21 and 23). \dg{All blur kernels are padded as same size for better illustration.}
(b) An example of the \kuire{synthetic image with sparse gradients, where kernel \#6 is used.}}
\label{fig:syn_k_imgsp}
\end{figure}

\subsubsection{Evaluation metrics}
Since we have the ground truth for synthetic datasets, we use the peak-signal-noise-ratio (PSNR) and structural similarity index (SSIM) \cite{wang2004ssim} as the  metric.

\subsubsection{Comparison with other algorithms}
We compare the proposed MPTV method with several state-of-the-art competitors for image deconvolution, including the methods by promoting sparsity on image gradients (such as \eg ~FTVd \cite{wang2008new}, L0-Abs \cite{portilla2009l0abs}, WDTV \cite{lou2015weighted}, IRLS \cite{levin2007coded_irls}, hyper-Laplacian prior based method \cite{krishnan2009fast} and L0TV \cite{yuan2015l0tv}), and other methods (\eg BM3D \cite{dabov2008bm3ddeb} and a kernel similarity based method \cite{kheradmand2014Kernel}). \kuire{Moreover, we also compare with a TV-ADMM method by minimizing the isotropic TV model using ADMM method, which is a direct baseline of the proposed method.}

\subsection{Experiments on Synthetic Image with Sparse Gradients}
\kuire{Recall that the proposed MPTV is based on the identification of non-zero subsets of the sparse gradient.} We thus first evaluate the performance of the proposed method on an image with sparse gradients.

Fig. \ref{fig:syn_k_imgsp} (b) shows a $256\times 256$ sharp image with sparse gradients for synthesizing blurred testing images. We perform image deconvolution using different methods and \kuire{record the averaged values of PSNR and SSIM with 8 different kernels in Table \ref{tab:syn_sparse}.} From Table \ref{tab:syn_sparse}, the proposed MPTV performs the best in terms of PSNR and SSIM. \kuire{Here, the execution time of MPTV for handling a $256\times 256$ image is within 1 second.}

\begin{table*}[!t]
\begin{center}
\footnotesize
\caption{Comparison on the Synthetic Image with Sparse Gradients (PSNR/SSIM).}
\vspace{-0.35cm}
\centering
\begin{tabular}{p{1.4cm}|p{1.4cm}<{\centering}|p{1.4cm}<{\centering}<{\centering}<{\centering}|p{1.4cm}<{\centering}<{\centering}|p{1.4cm}<{\centering}|p{1.4cm}<{\centering}|p{1.4cm}<{\centering}|p{1.4cm}<{\centering}|p{1.4cm}<{\centering}|p{1.4cm}<{\centering}}
\hline
 {PSF index} & {1} & {2} & {3} & {4} & {5} & {6} & {7} & {8} & {Avg.} \\ \cline{1-10}
 {Input}                            & 31.37/0.9622 & 27.05/0.9178 & 27.41/0.9236 & 28.22/0.9324 & 29.03/0.9466 & 22.92/0.8610 & 22.69/0.8911 & 23.81/0.8881 & 26.56/0.9154    \\ \hline
 {FTVd \cite{wang2008new}}          & 47.36/0.9820 & 27.81/0.4727 & 27.54/0.4416 & 34.34/0.7328 & 32.99/0.6801 & 27.68/0.4767 & 34.32/0.7375 & 31.07/0.5823 & 32.89/0.6382     \\ \hline
{L0-Abs \cite{portilla2009l0abs}}   & 37.69/0.9670 & 33.83/0.9221 & 35.99/0.8572 & 39.22/0.8641 & 39.37/0.8673 & 33.52/0.7786 & 38.32/0.8595 & 37.76/0.8522 & 36.96/0.8710     \\ \hline
{BM3D \cite{dabov2008bm3ddeb}}      & 37.00/0.9355 & 25.68/0.4194 & 24.61/0.3763 & 40.62/0.9190 & 46.88/0.9814 & 23.14/0.3071 & 40.06/0.9435 & 35.57/0.8428 & 34.20/0.7156     \\ \hline
{WDTV \cite{lou2015weighted}}       & 47.44/0.9876 & 35.17/0.8668 & 34.42/0.7964 & 39.54/0.8953 & 40.45/0.9159 & 32.13/0.7047 & 40.29/0.9141 & 36.42/0.8256 & 38.23/0.8633     \\ \hline
{IRLS \cite{levin2007coded_irls}}   & 35.97/0.9826 & 33.96/0.9717 & 42.48/0.9831 & 49.24/0.9905 & 48.52/0.9929 & 48.82/0.9885 & 50.68/0.9940 & 45.66/0.9813 & 44.42/0.9856     \\ \hline
{HL \cite{krishnan2009fast}}        & 36.76/0.9874 & 33.72/0.9572 & 39.90/0.9764 & 49.23/0.9921 & 48.97/0.9943 & 40.54/0.9214 & 51.27/0.9935 & 44.88/0.9805 & 43.16/0.9754     \\ \hline
{KS \cite{kheradmand2014Kernel}}    & 37.21/0.9455 & 25.91/0.3960 & 25.06/0.3593 & 39.98/0.9229 & 46.80/0.9832 & 23.60/0.3020 & 40.14/0.9487 & 35.87/0.8555 & 34.32/0.7141     \\ \hline
{L0TV \cite{yuan2015l0tv}}          & 46.69/0.9821 & 41.09/0.9890 & 41.76/0.9568 & 30.95/0.6408 & 30.84/0.7854 & 31.88/0.8409 & 37.43/0.8709 & 21.87/0.4777 & 35.31/0.8180     \\ \hline
{TV-ADMM}                           & 44.14/0.9703 & 37.24/0.9693 & 40.78/0.9726 & 51.48/0.9790 & 52.93/0.9793 & 44.05/0.9567 & 53.59/0.9795 & 43.60/0.9707 & 45.98/0.9722     \\ \hline
{MPTV}                             & \bf 50.41/0.9997 & \bf 41.09/0.9973 & \bf 47.56/0.9979 & \bf 56.98/0.9996 & \bf 57.13/0.9997 & \bf 51.80/0.9955 & \bf 59.85/0.9998 & \bf  50.51/0.9965 & \bf 51.92/0.9983   \\ \hline
\end{tabular}
\label{tab:syn_sparse}
\end{center}
\end{table*}

\subsubsection{\kuire{Sensitivity study on noise}}
Based on the synthetic data above, we conduct a noise level sensitivity study for the proposed MPTV. We first
generate
two blurred images, one with the Gaussian blur kernel (\#1 in Fig. \ref{fig:syn_k_imgsp} (a)) and the other with a motion blur kernel (\#6 in Fig. \ref{fig:syn_k_imgsp} (a)), and then add Gaussian noise with noise with levels varying  from $1\%$ to $10\%$ with interval $1\%$
to the two blurred images. Fig.~\ref{fig:noise_sens} records the PSNR and SSIM values of different methods, which show that the performance of many previous methods decreases as noise level increases. The performance of MPTV is high and stable since the cutting-plane optimization method restricts the elements of $\bD\bx$ corresponding to the inactive $\btau$ to be zero.

\begin{figure*}[!t]
\centering
\subfigure[]{
\label{fig:noise_sens_a}
\centering
\includegraphics[trim =1mm 0mm 13mm 0mm, clip, width=0.23\linewidth]{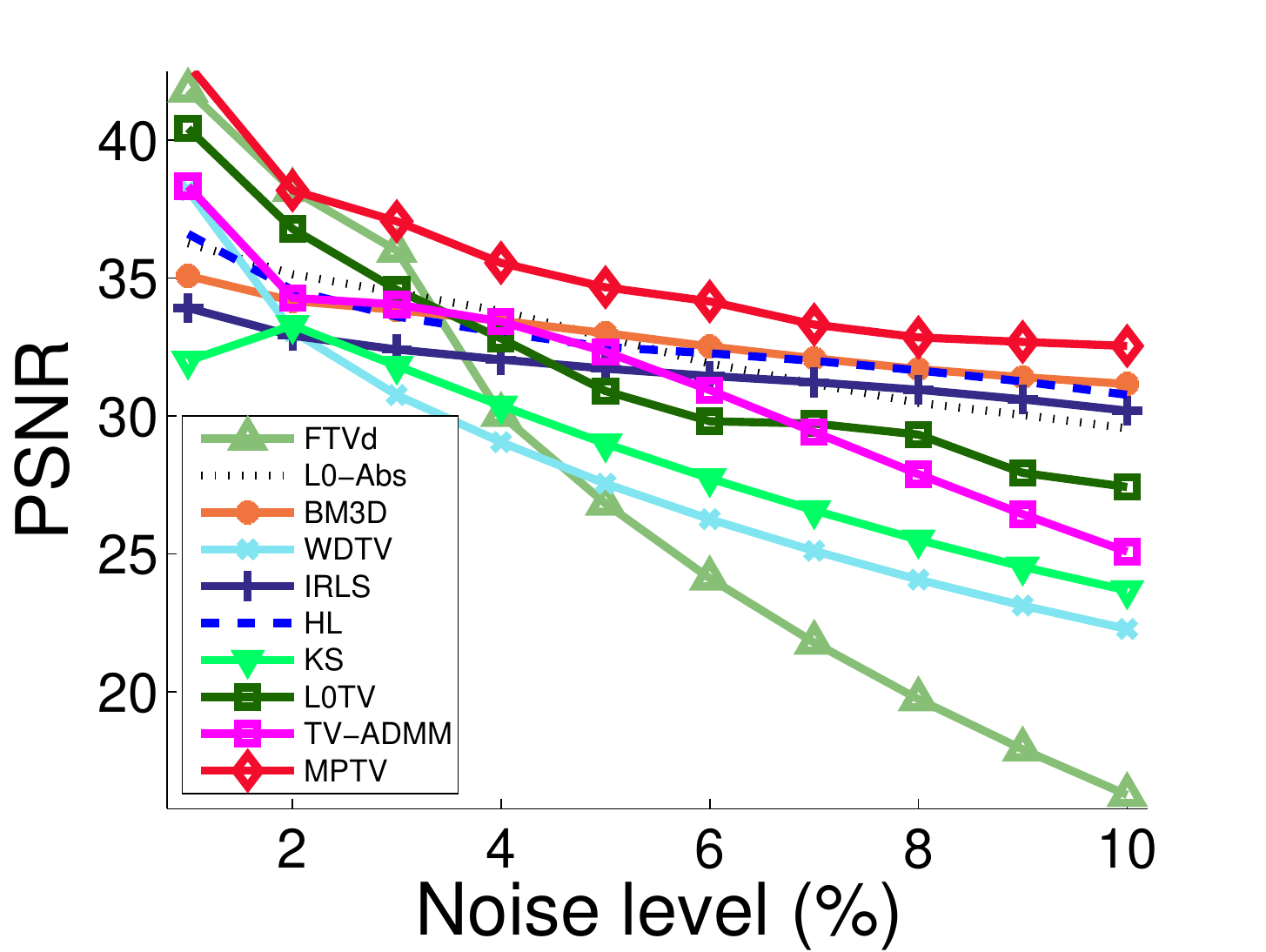}
}
\hfill
\subfigure[]{
\label{fig:noise_sens_b}
\centering
\includegraphics[trim =1mm 0mm 13mm 0mm, clip, width=0.23\linewidth]{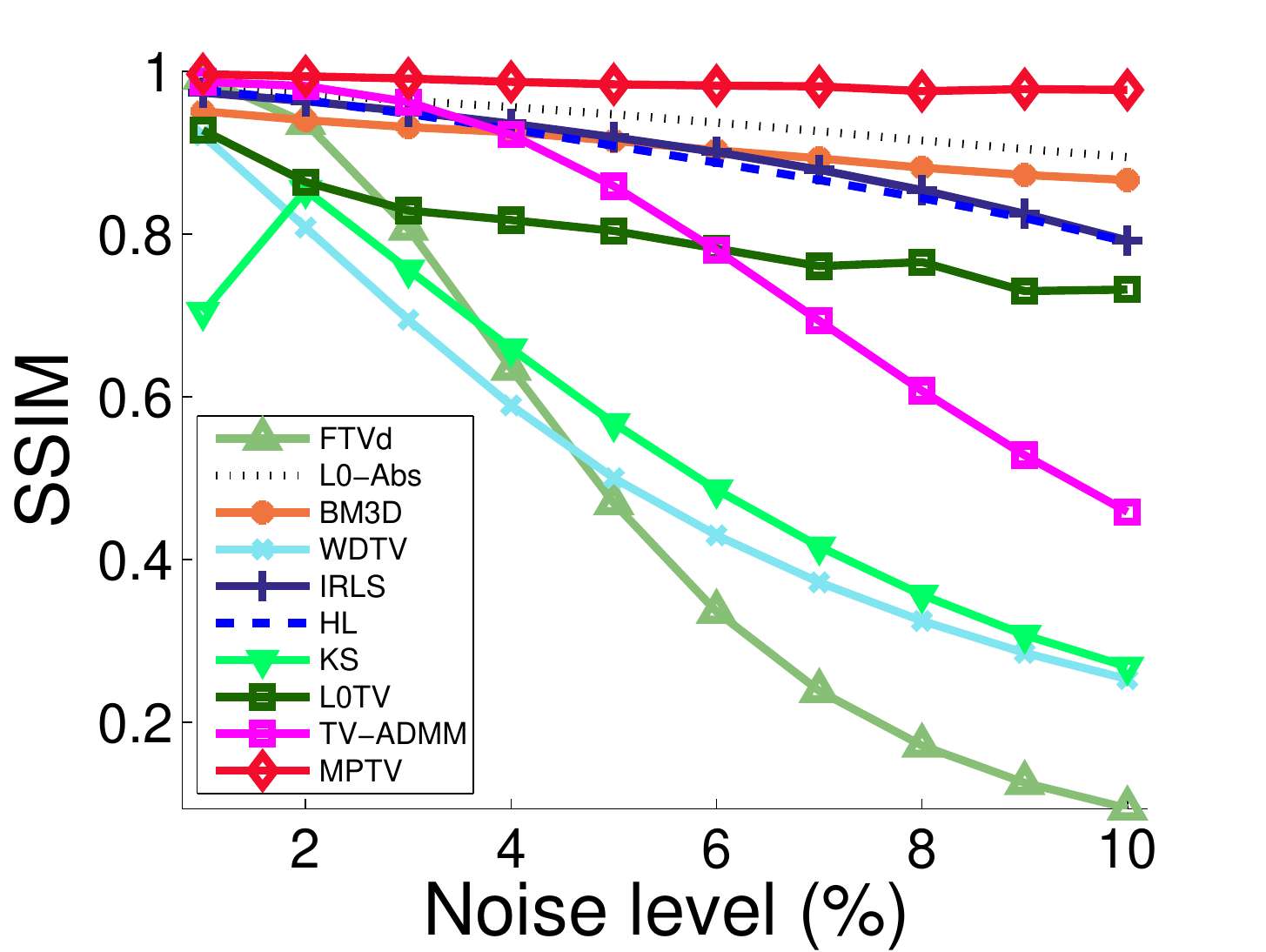}
}
\hfill
\subfigure[]{
\label{fig:noise_sens_c}
\centering
\includegraphics[trim =1mm 0mm 13mm 0mm, clip, width=0.23\linewidth]{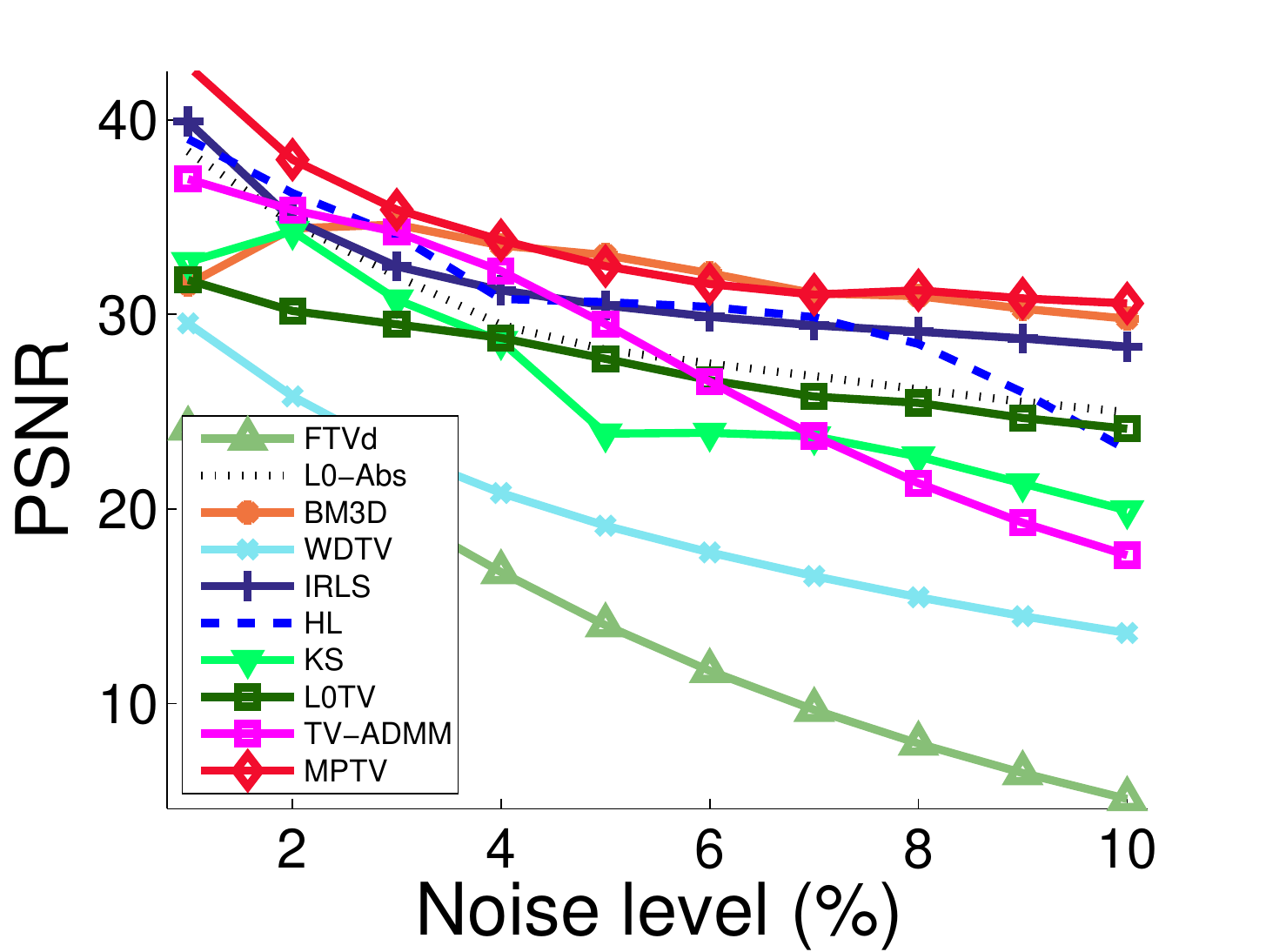}
}
\hfill
\subfigure[]{
\label{fig:noise_sens_d}
\centering
\includegraphics[trim =1mm 0mm 13mm 0mm, clip, width=0.23\linewidth]{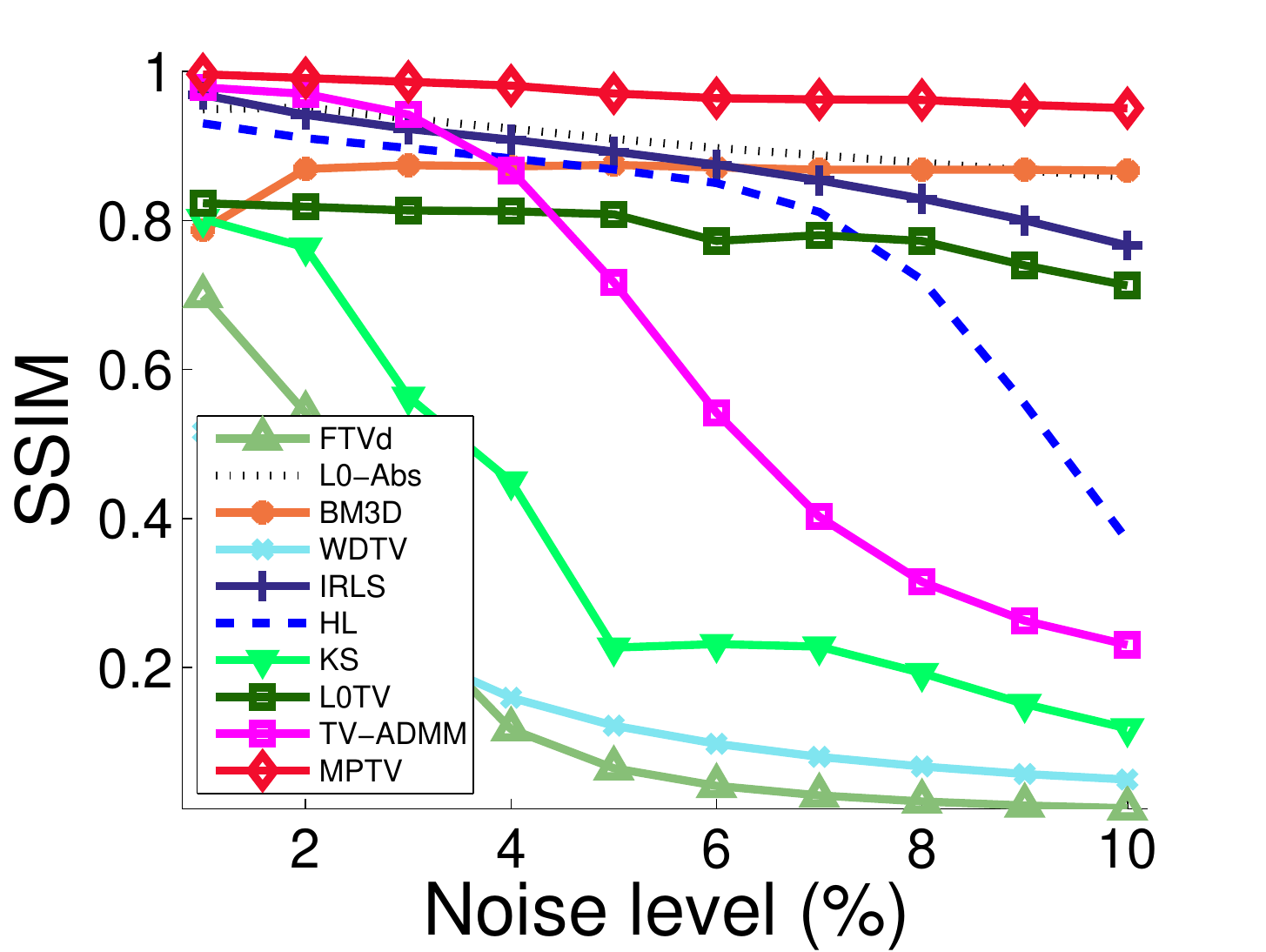}
}
\vspace{-0.2cm}
\caption{Study of noise sensitivity. PSNR and SSIM of the deconvolution results are evaluated for the blurry images contaminated by Gaussian noise with increasing noise level. (a) and (b) Results on the blurry images with Gaussian blur. (c) and (d) Results on the blurry images with motion blur.}
\label{fig:noise_sens}
\end{figure*}

\subsubsection{\kuire{Sensitivity study on blur level}}
The deconvolution task becomes more difficult as the level of blur increases.
In order to test the robustness of the various methods to increasing blur level, we generate synthetic images using both the Gaussian and linear motion blur kernels, as they exhibit a natural parametrization for this purpose.  We thus
generate 8 squared Gaussian blur kernels where the side lengths range from $15$ to $85$
pixels
with an interval of $10$, and standard derivations $1.6, 3, 5$ and $7$,
and 8 linear motion blur kernels with $45^\circ$ angle and lengths from $7$ to $63$ with an interval of $8$. Fig.~\ref{fig:blur_sens} shows that the performance of all methods decreases with increasing blur level for both Gaussian and motion blur. Nevertheless, the proposed MPTV performs the best for all blur levels.

\begin{figure*}[htp]
\centering
\subfigure[]{
\label{fig:blur_sens_a}
\centering
\includegraphics[trim =0mm 0mm 13mm 0mm, clip, width=0.23\linewidth]{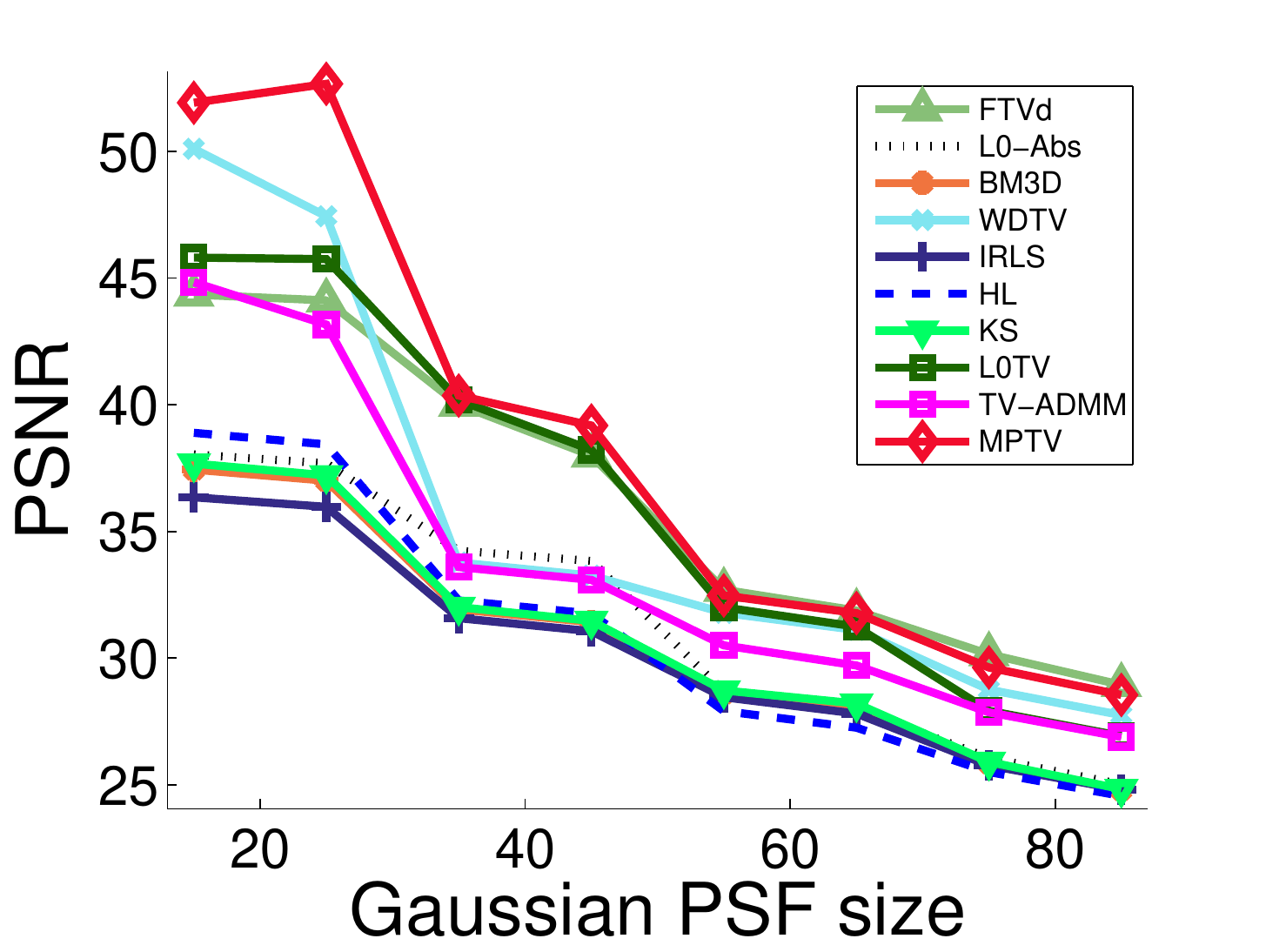}
}
\hfill
\subfigure[]{
\label{fig:blur_sens_b}
\centering
\includegraphics[trim =0mm 0mm 13mm 0mm, clip, width=0.23\linewidth]{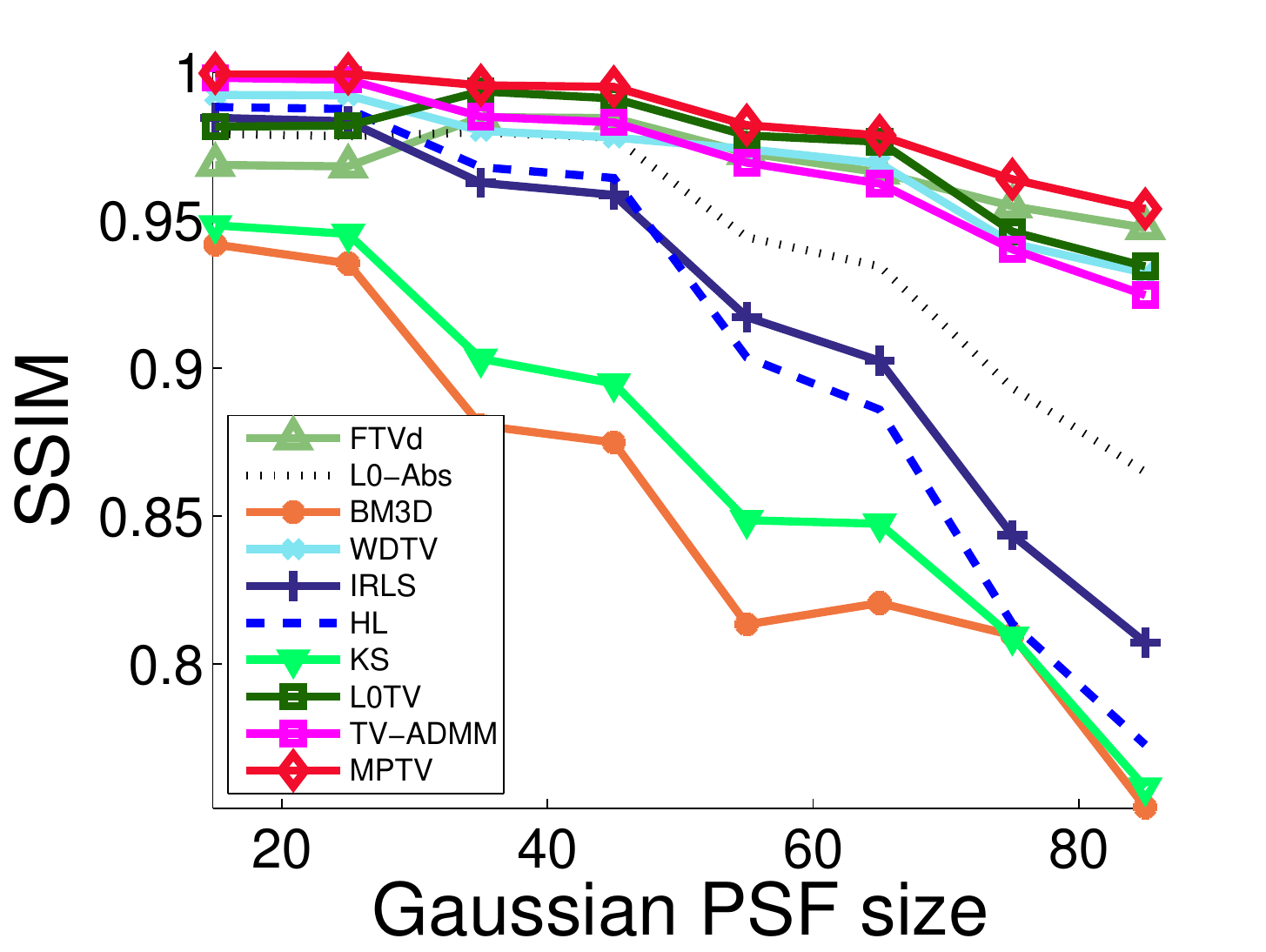}
}
\hfill
\subfigure[]{
\label{fig:blur_sens_c}
\centering
\includegraphics[trim =0mm 0mm 13mm 0mm, clip, width=0.23\linewidth]{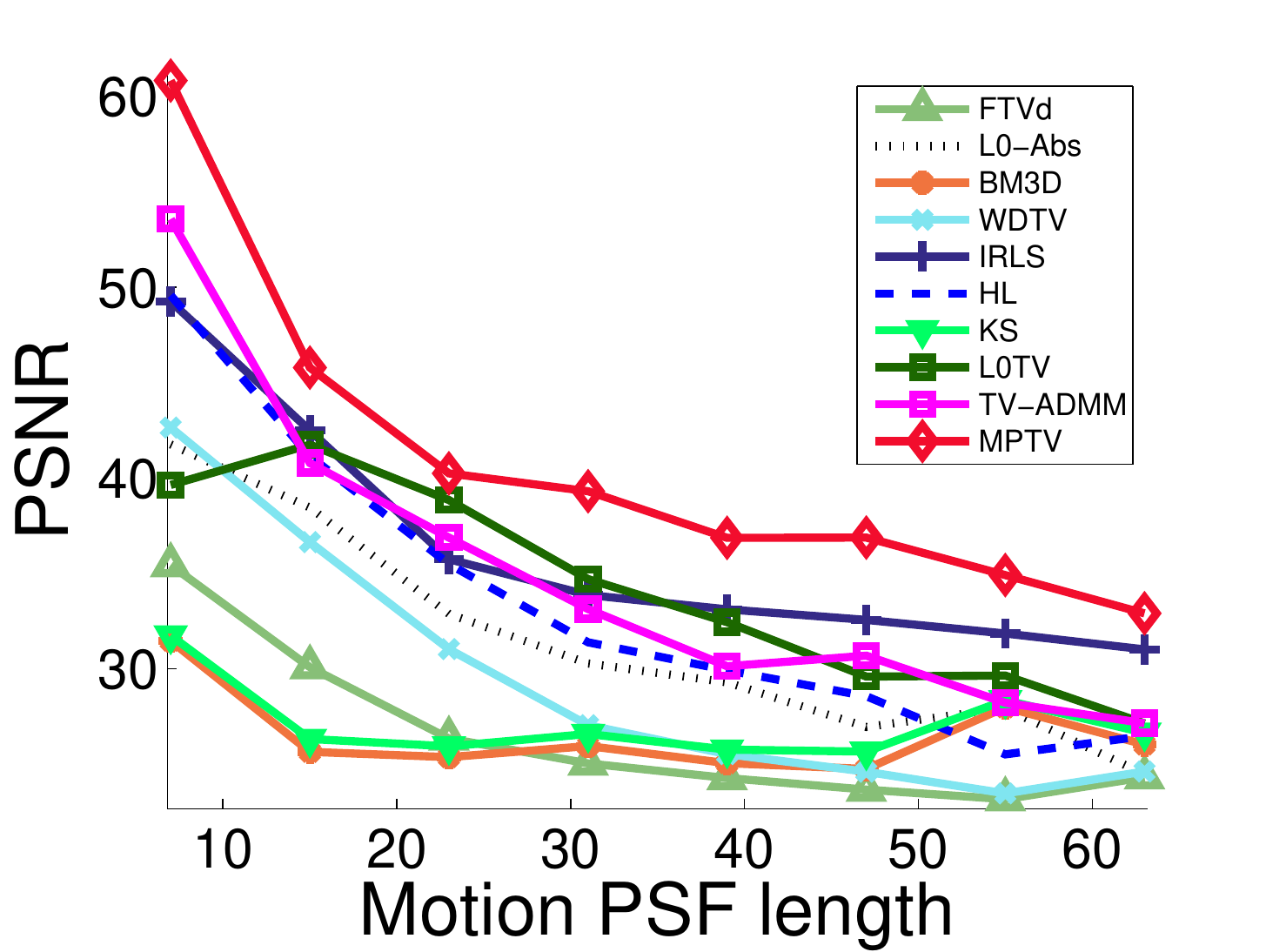}
}
\hfill
\subfigure[]{
\label{fig:blur_sens_d}
\centering
\includegraphics[trim =0mm 0mm 13mm 0mm, clip, width=0.23\linewidth]{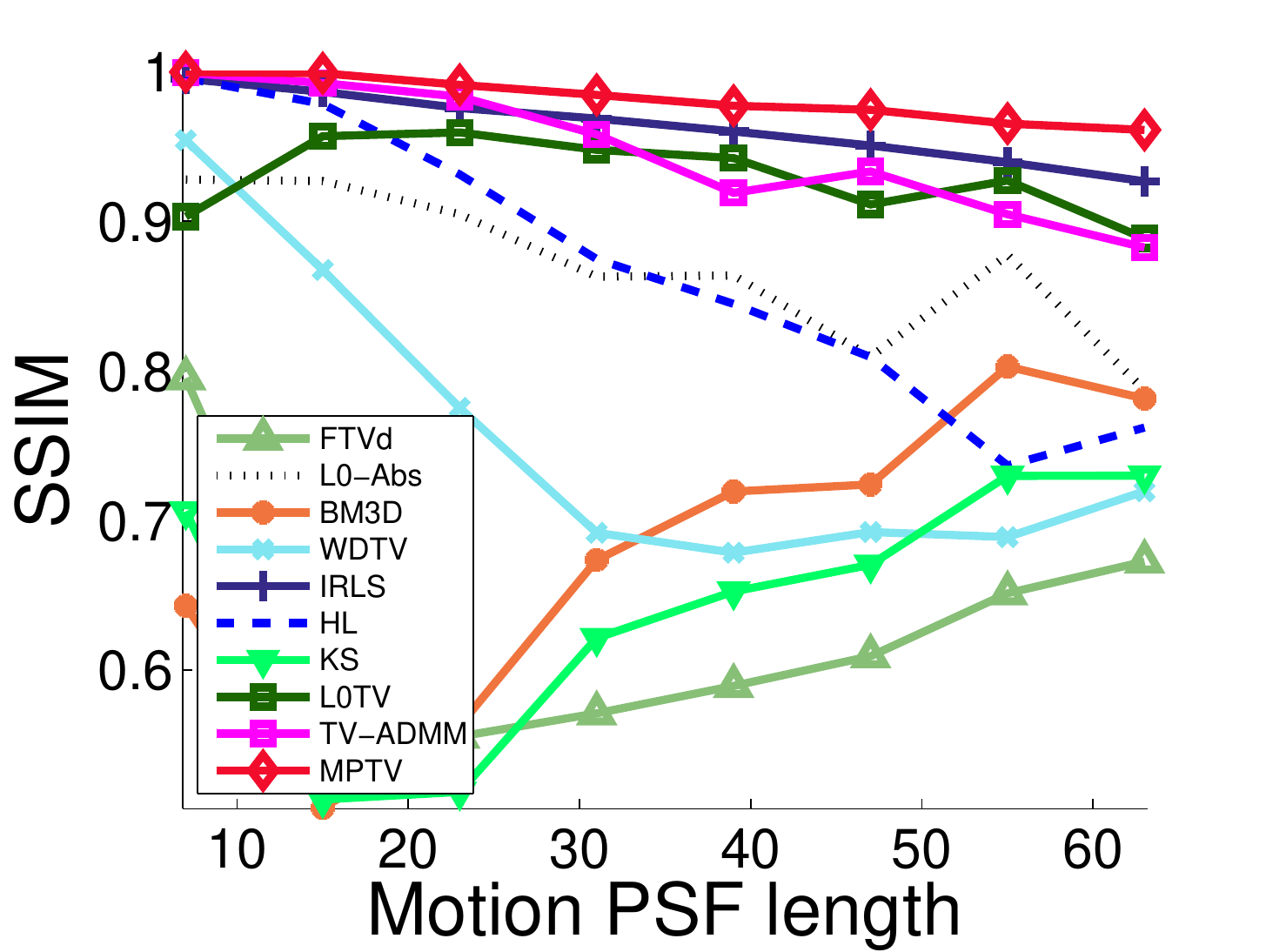}
}
\vspace{-0.2cm}
\caption{\kui{Blur level sensitivity study of various methods}. PSNR and SSIM of the deconvolution results for 8 blurry images contaminated by Gaussian blur or linear motion blur with increasing blur level. (a) and (b) Results on the images blurred by the Gaussian PSFs. (c) and (d) Results on the images blurred by the linear motion PSFs.}
\label{fig:blur_sens}
\end{figure*}

\subsubsection{Sensitivity study on parameters}
\label{sec:para_sens_exp}
\dgr{Firstly,
we conduct a sensitivity study for the parameter $\lambda$.
MPTV with a proper $\lambda$ achieves solutions with high accuracy and less regularizer bias. Due to the cutting-plane method and the early stopping strategy, MPTV is not sensitive to the value of $\lambda$.}
In this experiment, we primarily compare MPTV with the TV-ADMM method that shares a similar problem formulation and the same implementation details apart from the binary indicator $\btau$ and the cutting-plane based optimization scheme. We perform deconvolution using the two methods given 20 $\lambda$'s with values from $1\times 10^{-5}$ to $0.96\times 10^{-3}$ with an interval of $5\times 10^{-5}$, and record the average PSNR and SSIM values of all testing images in Fig. \ref{fig:para_lambda_sens}. According to evaluation based on both PSNR and SSIM, the proposed MPTV method is more robust to the setting of $\lambda$ and outperforms the TV-ADMM for all $\lambda$'s, which proves the effectiveness of the proposed MPTV formulation and the cutting-plane based optimization method.

\begin{figure}[htp]
\centering
\subfigure[Evaluation based on PSNR]{
\label{fig:noise_sens_a}
\centering
\includegraphics[width=0.45\linewidth]{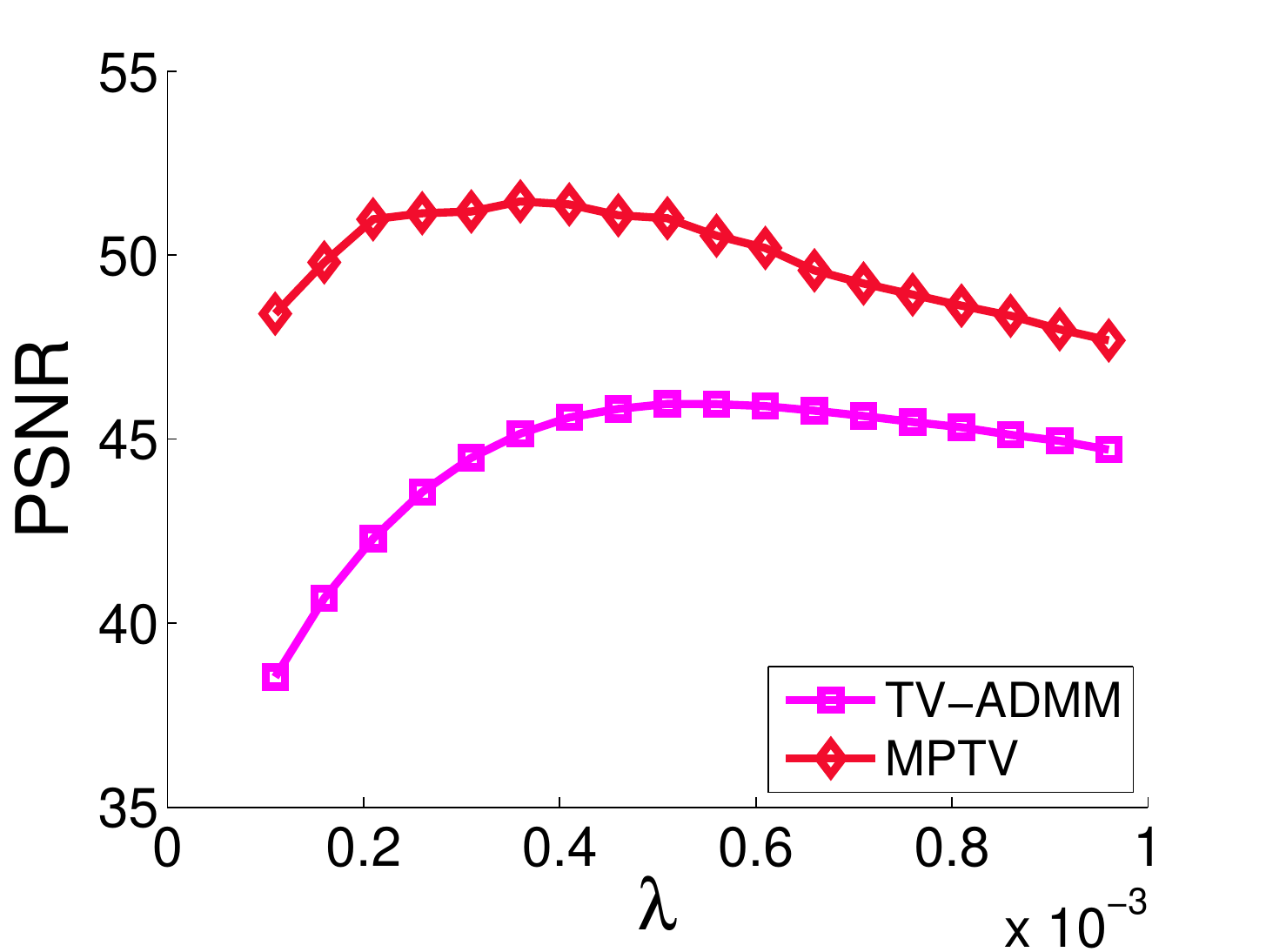}
}
\subfigure[Evaluation based on SSIM]{
\label{fig:noise_sens_b}
\centering
\includegraphics[width=0.45\linewidth]{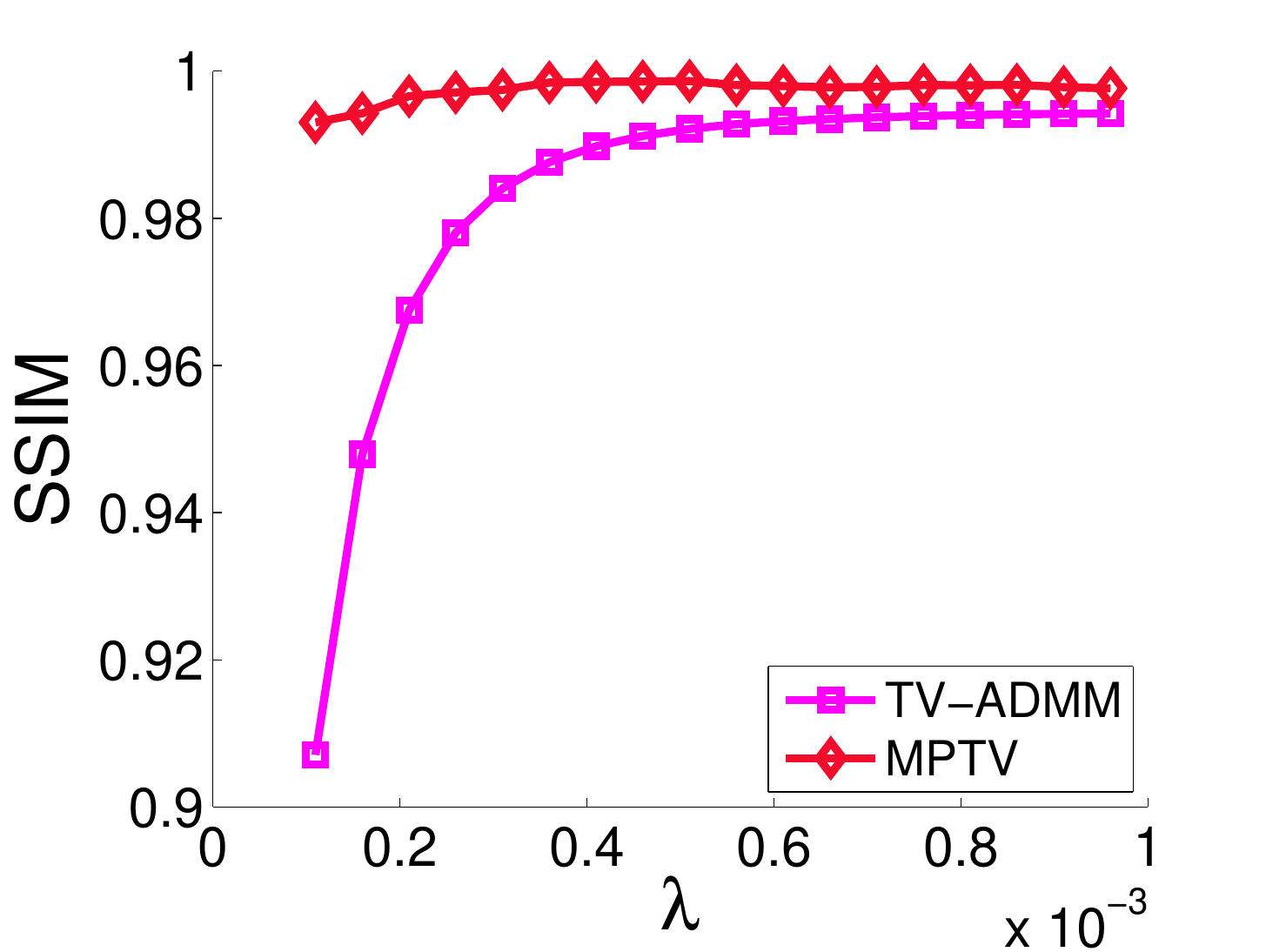}
}
\vspace{-0.15cm}
\caption{Sensitivity study of parameter $\lambda$. The deconvolution results of TV-ADMM and MPTV with varying $\lambda$ are evaluated and compared.}
\label{fig:para_lambda_sens}
\end{figure}

\par
\dgr{In the second experiment, we study the sensitivity of the parameter $\kappa$ for MPTV. The PSNR values of the image results and the corresponding $\kappa$ are reported in Fig. \ref{fig:kappa_epsilon_sens} (a). The same stopping condition described in Section \ref{sec:early_stopping} is used for different $\kappa$'s (from 64 to 512).
From Fig. \ref{fig:kappa_epsilon_sens} (a), a proper $\kappa$ helps to prevent MPTV from activating too many or too less nonzero gradients, resulting in high-quality results.
If $\kappa$ is too small, MPTV may not activate enough nonzero elements in a limited number of iterations. When $\kappa$ increases, the performance slightly degrades.
}

\par
\dgr{In the third experiment, we study the influence of the parameter $\epsilon$ on the performance of MPTV. Different $\epsilon$ values result in different iteration numbers.
Fig. \ref{fig:kappa_epsilon_sens} (b) records the PSNR values with 8 different $\epsilon$'s from $1\times 10^{-5}$ to $5\times 10^{-2}$. As shown in Fig. \ref{fig:kappa_epsilon_sens} (b), MPTV works robust to a wide range of $\epsilon$. When $\epsilon$ is large enough, \eg $\epsilon=1\times 10^{-3}$ in Fig \ref{fig:kappa_epsilon_sens} (b), the early stopping helps MPTV to produce images with sparse gradients and satisfactory quality.
A too small $\epsilon$ leads to too many iterations and thus activates too many nonzero gradients, which slightly degrades the performance.}

\begin{figure}[htp]
\centering
\subfigure[PSNR vs. $\kappa$]{
\centering
\includegraphics[trim =1mm 0mm 13mm 0mm, clip, width=0.45\linewidth]{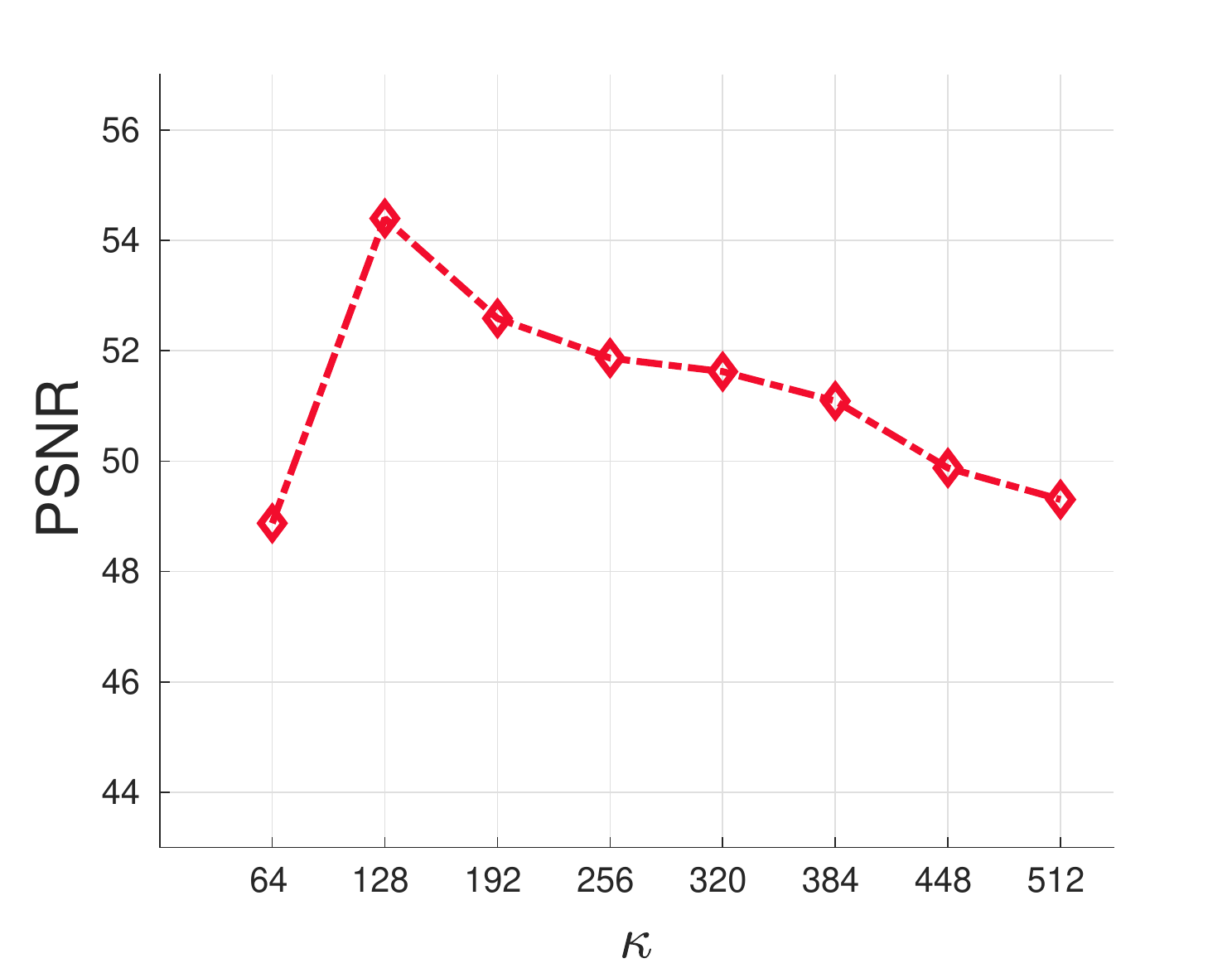}
}
\subfigure[PSNR vs. $\epsilon$]{
\centering
\includegraphics[trim =1mm 0mm 10mm 0mm, clip, width=0.45\linewidth]{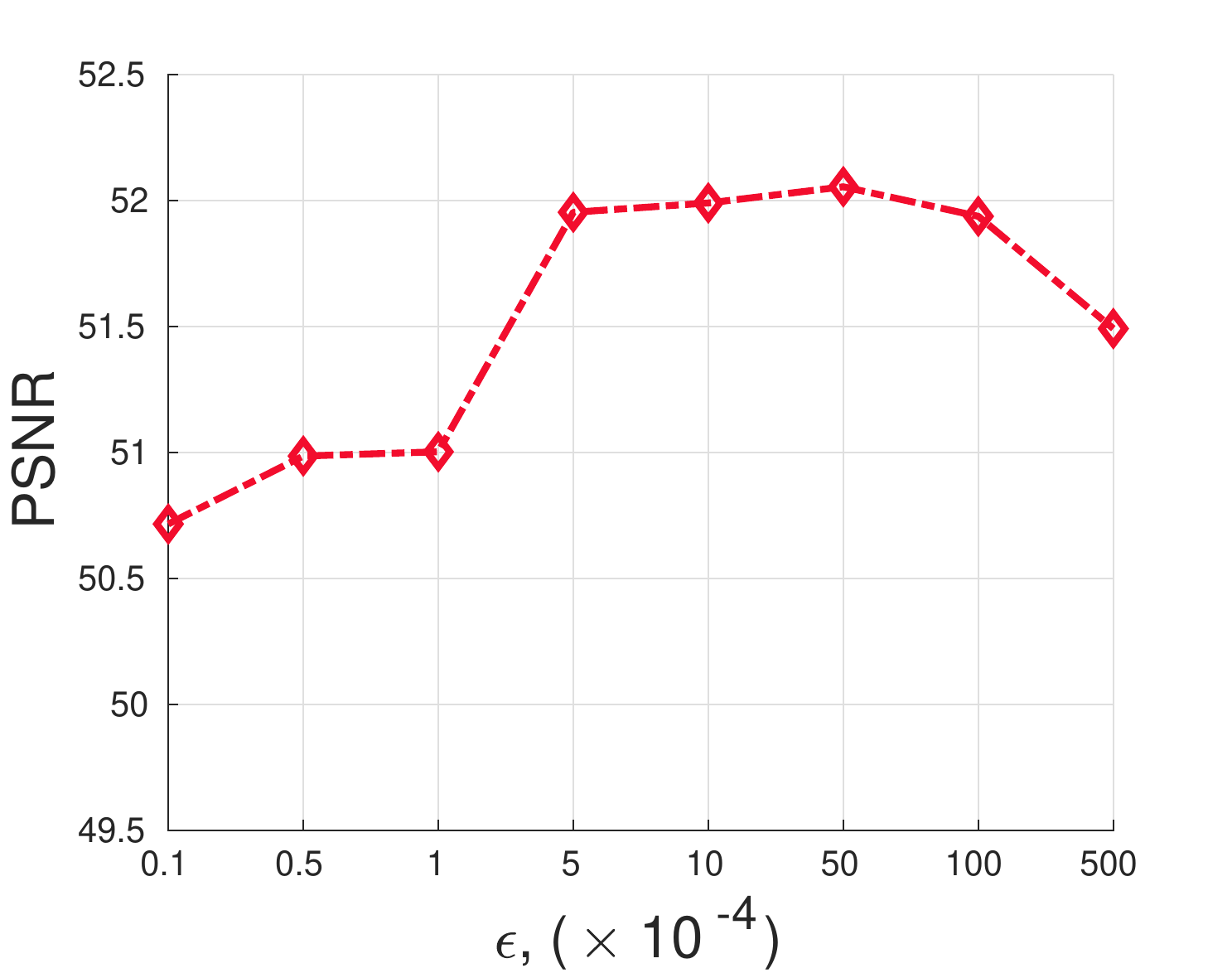}
}
\vspace{-0.15cm}
\caption{Sensitivity studies of $\kappa$ and $\epsilon$ for MPTV. PSNR of the deconvolution results with different values for the parameters are shown. (a) Sensitivity of $\kappa$. (b) Sensitivity of the stopping condition parameter $\epsilon$. \emph{In (b), the x axis is scaled for visualization}.}
\label{fig:kappa_epsilon_sens}
\end{figure}

\subsection{Quantitative Studies on Text Images}\kuire{A text image often contains near sparse gradients and
represents an important class of real images.}  \kuire{We thus study the performance of MPTV on text images using a dataset containing 14 ground truth text images and 8 kernels (in Fig. \ref{fig:syn_k_imgsp} (a)).} Here, we perform deconvolution with ground-truth blur kernels. For each blur kernel, we show the averaged PSNR and SSIM of both input blurred images and the deblurring results in Table \ref{tab:syn_text}\footnote{The comparison with L0-ABS \cite{portilla2009l0abs} is absent since the corresponding code only works on square images.}. \kuire{From the table, MPTV consistently performs better than state-of-the-art methods.} \kuire{In particular,} the results of MPTV appear sharper than others.

\begin{table*}[!t]
\begin{center}
\footnotesize
\caption{Qualitative Comparison on Text Images with Near Sparse Gradients (PSNR/SSIM).}
\vspace{-0.35cm}
\centering
\begin{tabular}{p{1.4cm}|p{1.4cm}<{\centering}|p{1.4cm}<{\centering}<{\centering}<{\centering}|p{1.4cm}<{\centering}<{\centering}|p{1.4cm}<{\centering}|p{1.4cm}<{\centering}|p{1.4cm}<{\centering}|p{1.4cm}<{\centering}|p{1.4cm}<{\centering}|p{1.4cm}<{\centering}}
\hline
 {PSF index} & {1} & {2} & {3} & {4} & {5} & {6} & {7} & {8} & {Avg.} \\ \cline{1-10}
 {Input}                            & 20.04/0.8009 & 15.97/0.5684 & 16.88/0.6312 & 17.51/0.6632 & 17.54/0.6945 & 13.49/0.4328 & 13.86/0.4882 & 14.63/0.4938 & 16.24/0.5966     \\ \hline
{FTVd \cite{wang2008new}} & 28.91/0.9573 & 27.03/0.9347 & 32.58/0.9398 & 37.00/0.9423 & 35.81/0.9490 & 31.88/0.8716 & 36.95/0.9492 & 34.18/0.9337 & 33.04/0.9347     \\ \hline
{BM3D \cite{dabov2008bm3ddeb}} & 27.14/0.9517 & 25.52/0.8942 & 28.76/0.8926 & 37.93/0.9866 & 37.01/\textbf{0.9891} & 28.86/0.8562 & 38.49/0.9891 & 33.18/0.9499 & 32.11/0.9387     \\ \hline
{WDTV \cite{lou2015weighted}} & 27.38/0.9459 & 25.40/0.9103 & 32.06/0.9302 & 37.25/0.9427 & 35.95/0.9501 & 31.96/0.8720 & 37.25/0.9514 & 34.23/0.9339 & 32.68/0.9296     \\ \hline
{IRLS \cite{levin2007coded_irls}} & 25.78/0.9384 & 22.96/0.8808 & 30.16/0.9569 & 38.49/0.9820 & 35.63/0.9833 & 37.22/0.9790 & 38.11/0.9859 & 35.00/0.9802 & 32.92/0.9608     \\ \hline
{HL \cite{krishnan2009fast}} & 25.60/0.9342 & 21.03/0.8210 & 27.63/0.9418 & 35.06/0.9825 & 33.65/0.9818 & 32.60/0.9667 & 36.26/0.9871 & 32.72/0.9751 & 30.57/0.9488     \\ \hline
{KS \cite{kheradmand2014Kernel}} & 27.52/0.9422 & 25.70/0.8948 & 29.28/0.8912 & 37.95/0.9738 & 37.12/0.9804 & 29.27/0.8535 & 38.67/0.9879 & 33.55/0.9502 & 32.38/0.9343     \\ \hline
{L0TV \cite{yuan2015l0tv}} & 29.24/0.9502 & 25.62/0.9018 & 31.99/0.9152 & 13.91/0.5324 & 14.38/0.5769 & 14.94/0.5111 & 14.09/0.5202 & 11.62/0.4153 & 19.47/0.6654     \\ \hline
{TV-ADMM} & 29.98/0.9673 & 26.91/0.9396 & 33.64/0.9635 & 38.98/0.9738 & 37.46/0.9760 & 37.10/0.9642 & 39.13/0.9775 & 36.68/0.9718 & 34.99/0.9667     \\ \hline
{MPTV} & \bf 30.15/0.9681 & \bf 27.82/0.9434 & \bf 34.27/0.9719 & \bf 39.94/0.9873 & \textbf{38.04}/0.9830 & \bf 38.05/0.9815 & \bf 39.96/0.9893 & \bf 37.30/0.9823 & \bf 35.69/0.9759     \\ \hline
\end{tabular}
\label{tab:syn_text}
\end{center}
\end{table*}

\subsection{Experiments on Natural Images}
In this section, we evaluate the proposed MPTV algorithm and other methods using natural images. We construct a dataset including 192 blurred natural images from 24 sharp images and the 8 kernels shown in Fig. \ref{fig:syn_k_imgsp} (a).
Since the gradients of natural images are much denser, MPTV is required to activate many gradients to fit the observation, leading to lower superiority over the comparator methods, as shown in Table \ref{tab:syn_natural}. Nevertheless, the proposed MPTV performs better than others in terms of PSNR and SSIM values.
Fig. \ref{fig:syn_natural_cman} provides a visual comparison which demonstrates that the latent images estimated by MPTV exhibit
sharper edges and less ringing artifacts than those of its comparators. We also note that MPTV failed to model some of the subtle textures in the background due to the binary indicator for the activated gradients.

\begin{figure*}[htp]
\centering
\subfigure[Ground truth $\bx^*$]{
\centering
\includegraphics[width=0.16\linewidth]{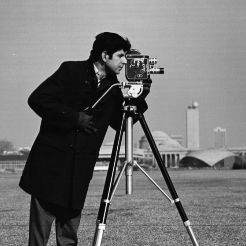}
}
\hspace{-0.4cm}
\subfigure[Blurred image $\by$]{
\centering
\includegraphics[width=0.16\linewidth]{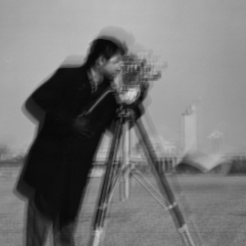}
}
\hspace{-0.4cm}
\subfigure[FTVd \cite{wang2008new}]{
\centering
\includegraphics[width=0.16\linewidth]{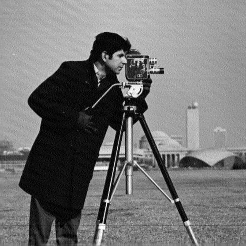}
}
\hspace{-0.4cm}
\subfigure[L0-Abs \cite{portilla2009l0abs}]{
\centering
\includegraphics[width=0.16\linewidth]{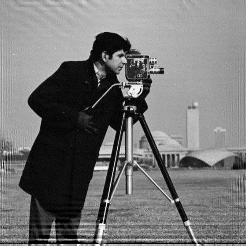}
}
\hspace{-0.4cm}
\subfigure[BM3D \cite{dabov2008bm3ddeb}]{
\centering
\includegraphics[width=0.16\linewidth]{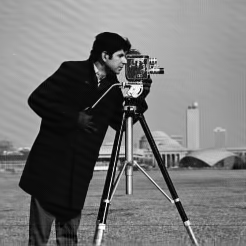}
}
\hspace{-0.4cm}
\subfigure[WTVD \cite{lou2015weighted}]{
\centering
\includegraphics[width=0.16\linewidth]{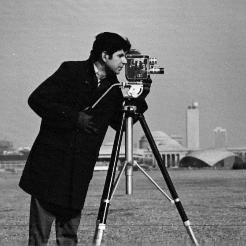}
}
\hspace{-0.4cm}
\subfigure[IRLS \cite{levin2007coded_irls}]{
\centering
\includegraphics[width=0.16\linewidth]{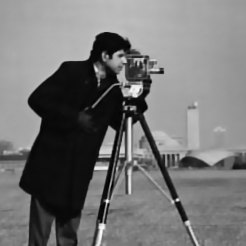}
}
\hspace{-0.4cm}
\subfigure[HL \cite{krishnan2009fast}]{
\centering
\includegraphics[width=0.16\linewidth]{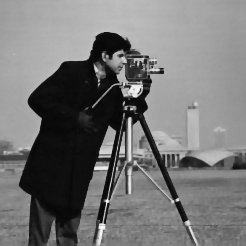}
}
\hspace{-0.4cm}
\subfigure[KS \cite{kheradmand2014Kernel}]{
\centering
\includegraphics[width=0.16\linewidth]{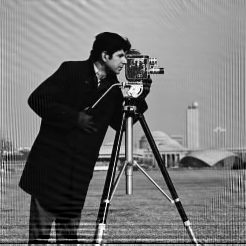}
}
\hspace{-0.4cm}
\subfigure[L0TV \cite{yuan2015l0tv}]{
\centering
\includegraphics[width=0.16\linewidth]{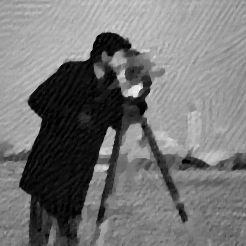}
}
\hspace{-0.4cm}
\subfigure[TVADMM]{
\centering
\includegraphics[width=0.16\linewidth]{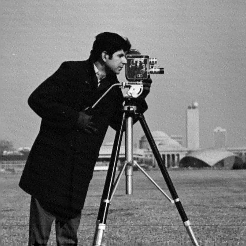}
}
\hspace{-0.4cm}
\subfigure[MPTV]{
\centering
\includegraphics[width=0.16\linewidth]{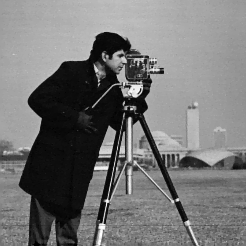}
}
\vspace{-0.2cm}
\caption{An example of the results on synthetically blurred natural images. The blurred image $\by$ is generated using the sharp image \texttt{cameraman} and the \#6 blur kernel in Fig. \ref{fig:syn_k_imgsp} (a). (a) Ground truth sharp image $\bx^*$. (b) Input blurred image $\by$. (c)-(l) Deconvolution results of different methods in comparison.}
\label{fig:syn_natural_cman}
\end{figure*}

\subsubsection{\kuire{Sensitivity study on blur kernel errors}}
In practice, the blur kernel for image deconvolution is usually an estimate containing errors, which lead to ringing artifacts \cite{Shan2008High} and incur degraded deconvolution performance.
We thus here describe an experiment to study the sensitivity of different methods
to errors in the kernel estimate.
Seven $256\times 256$ sharp natural images (among the 24 sharp images above) and two blur kernels (Gaussian blur kernel \#1 and motion blur kernel \#6) are used to generate the test data.
The ground truth known blur kernel is applied to each image, and each kernel then has noise added to form an erroneous blur kernel estimate which is passed to the various deblurring methods. The added noise is sampled from Gaussian distribution with a noise level increasing from
$0.2\%$ to $0.5\%$ (with an interval of $0.05\%$).

As shown in Fig. \ref{fig:syn_kerror_sens}, the proposed method is more robust than other algorithms since the cutting-plane of the proposed MPTV estimates the sharp image by gradually activating the significant gradients which helps to suppress the ringing artifacts.

\begin{table*}[!t]
\begin{center}
\footnotesize
\caption{Qualitative Comparison on Natural Images (PSNR/SSIM).}
\vspace{-0.35cm}
\centering
\begin{tabular}{p{1.4cm}|p{1.4cm}<{\centering}|p{1.4cm}<{\centering}<{\centering}<{\centering}|p{1.4cm}<{\centering}<{\centering}|p{1.4cm}<{\centering}|p{1.4cm}<{\centering}|p{1.4cm}<{\centering}|p{1.4cm}<{\centering}|p{1.4cm}<{\centering}|p{1.4cm}<{\centering}}
\hline
 {PSF index} & {1} & {2} & {3} & {4} & {5} & {6} & {7} & {8} & {Avg.} \\ \cline{1-10}
 {Input} & 25.75/0.7437 & 21.53/0.4991 & 22.42/0.5822 & 23.18/0.6131 & 23.22/0.6203 & 18.66/0.3782 & 19.17/0.4138 & 20.05/0.4371 & 21.75/0.5359     \\ \hline
 {FTVd \cite{wang2008new}} & 29.74/0.8459 & 23.93/0.5673 & 23.71/0.5770 & 31.07/0.8197 & 29.67/0.7663 & 25.00/0.6734 & 31.53/0.8321 & 27.25/0.7236 & 27.74/0.7257     \\ \hline
{L0-Abs \cite{portilla2009l0abs}} & 30.95/\textbf{0.8911} & 19.36/0.5190 & 19.69/0.5368 & 26.52/0.7709 & 29.53/0.8588 & 22.58/0.6414 & 31.74/0.8913 & 22.69/0.6637 & 25.38/0.7216     \\ \hline
{BM3D \cite{dabov2008bm3ddeb}} & 30.76/0.8799 & 24.79/0.6477 & 24.44/0.6576 & 33.07/0.8916 & 34.50/0.9281 & 21.77/0.5983 & 35.33/0.9392 & 29.55/0.8434 & 29.28/0.7982     \\ \hline
{WDTV \cite{lou2015weighted}} & 30.60/0.8818 & 27.62/0.7784 & 29.37/0.8323 & 34.86/0.9228 & 34.57/0.9216 & 29.87/0.8496 & 35.68/0.9359 & 30.90/0.8742 & 31.69/0.8746     \\ \hline
{IRLS \cite{levin2007coded_irls}} & 29.02/0.8423 & 25.27/0.6933 & 27.96/0.8019 & 30.35/0.8609 & 30.25/0.8582 & 29.42/0.8394 & 31.97/0.8938 & 30.09/0.8576 & 29.29/0.8309     \\ \hline
{HL \cite{krishnan2009fast}} & 29.59/0.8540 & 25.49/0.6980 & 28.56/0.8164 & 31.89/0.8889 & 31.85/0.8836 & 30.10/0.8614 & 33.59/0.9161 & 31.04/0.8773 & 30.26/0.8495     \\ \hline
{KS \cite{kheradmand2014Kernel}} & 30.48/0.8780 & 11.67/0.1524 & 11.85/0.1807 & 19.79/0.4958 & 22.92/0.7444 & 15.53/0.3038 & 26.98/0.8139 & 15.09/0.3183 & 19.29/0.4859     \\ \hline
{L0TV \cite{yuan2015l0tv}} & 29.28/0.8615 & 26.17/0.7345 & 28.37/0.8189 & 21.52/0.4082 & 21.97/0.4506 & 22.73/0.5399 & 23.59/0.5455 & 18.80/0.2801 & 24.05/0.5799     \\ \hline
{TV-ADMM} & 30.93/0.8881 & 27.42/0.7753 & 28.25/0.8119 & 35.06/0.9263 & 34.67/0.9230 & 31.70/0.8869 & 35.74/0.9363 & 30.76/0.8744 & 31.82/0.8778     \\ \hline
{MPTV} & \textbf{30.96}/0.8869 & \bf 27.89/0.7918 & \bf 30.08/0.8627 & \bf 35.53/0.9394 & \bf 35.02/0.9330 & \bf 32.85/0.9188 & \bf 36.32/0.9489 & \bf 32.67/0.9161 & \bf 32.67/0.8997     \\ \hline
\end{tabular}
\label{tab:syn_natural}
\end{center}
\end{table*}

\begin{figure*}[htp]
\centering
\subfigure[]{
\label{fig:blur_sens_a}
\centering
\includegraphics[trim =0mm 0mm 11mm 0mm, clip, width=0.23\linewidth]{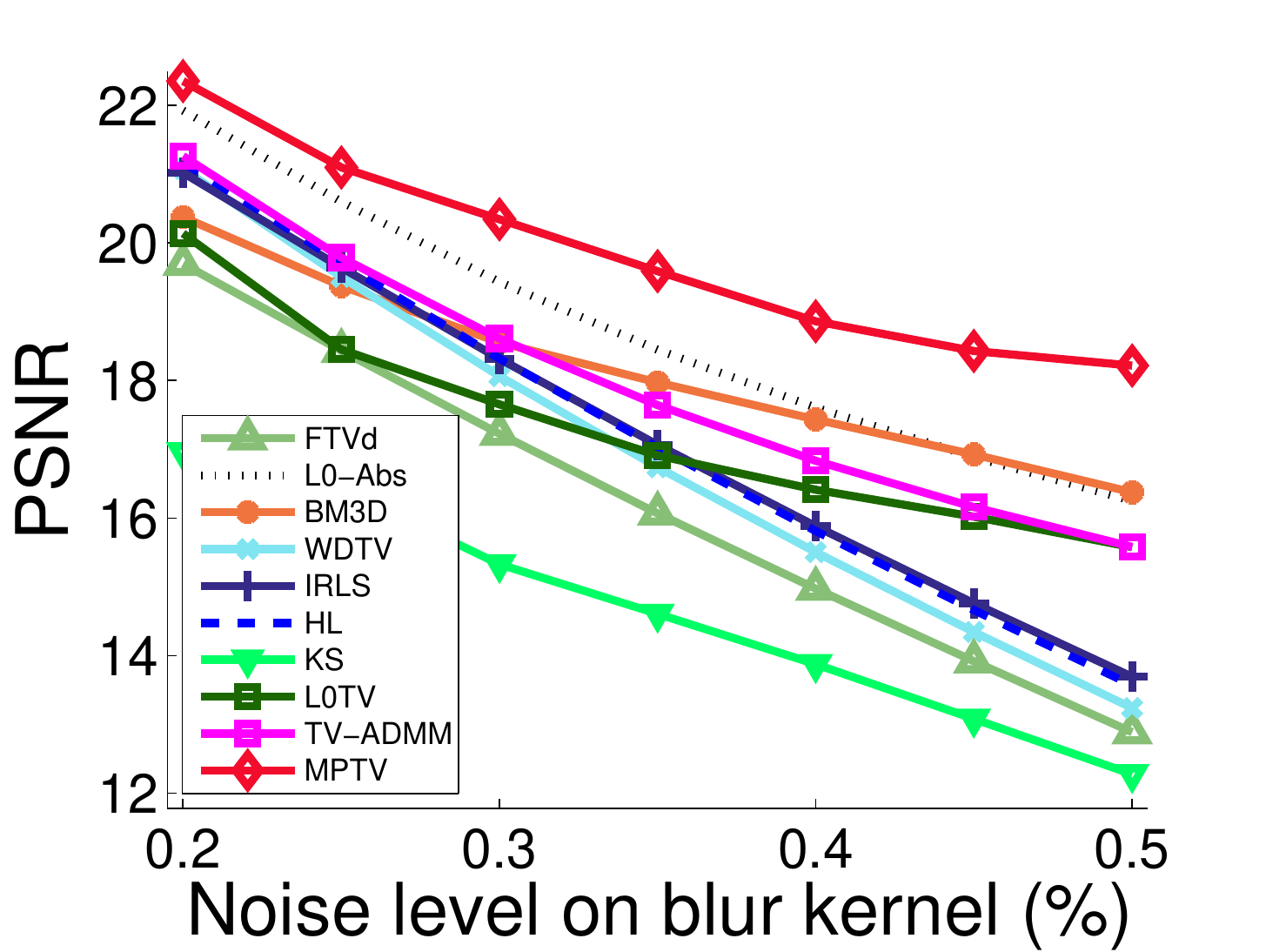}
}
\hfill
\subfigure[]{
\label{fig:blur_sens_b}
\centering
\includegraphics[trim =1.5mm 0mm 11mm 0mm, clip, width=0.23\linewidth]{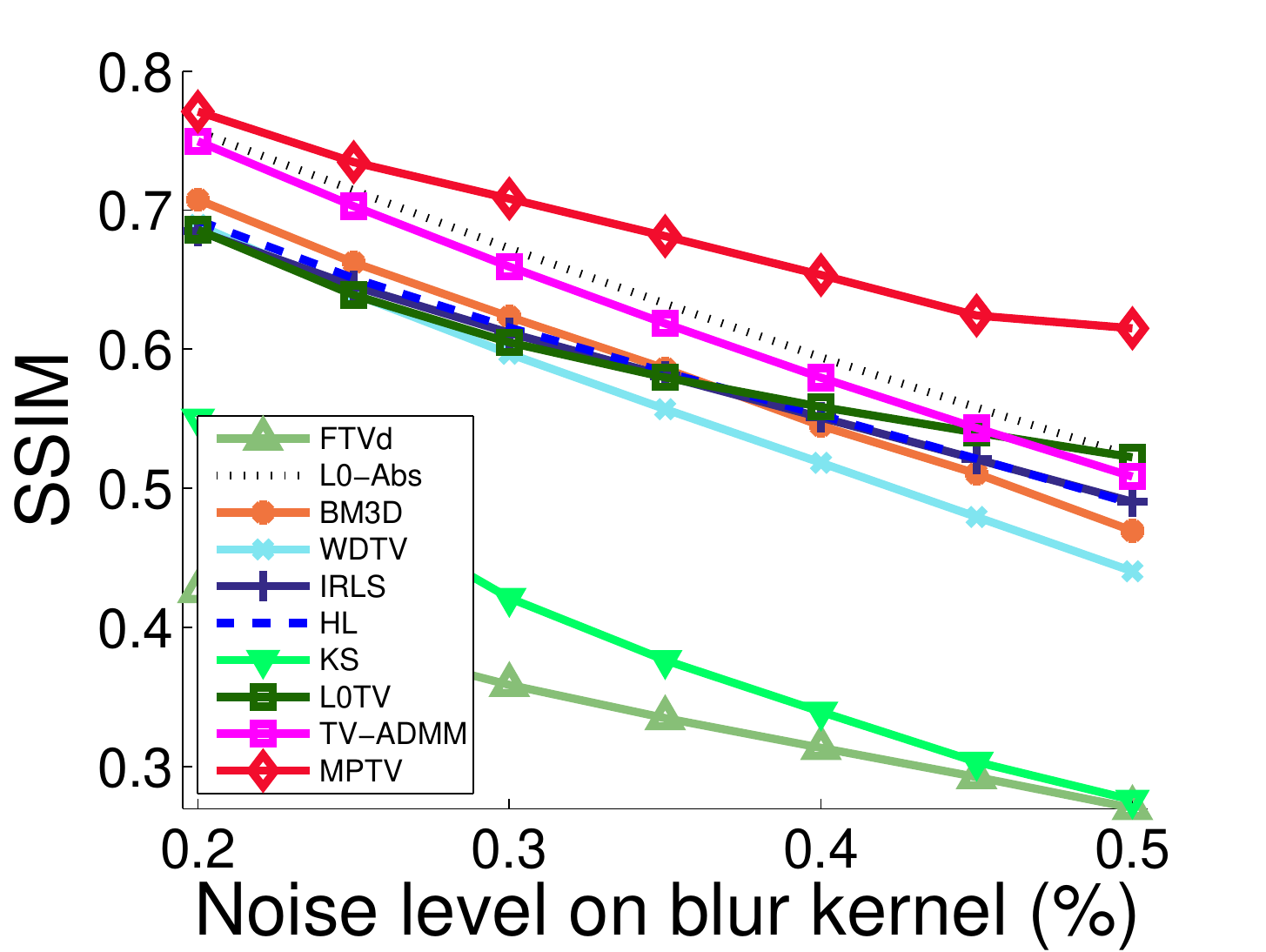}
}
\hfill
\subfigure[]{
\label{fig:blur_sens_c}
\centering
\includegraphics[trim =0mm 0mm 11mm 0mm, clip, width=0.23\linewidth]{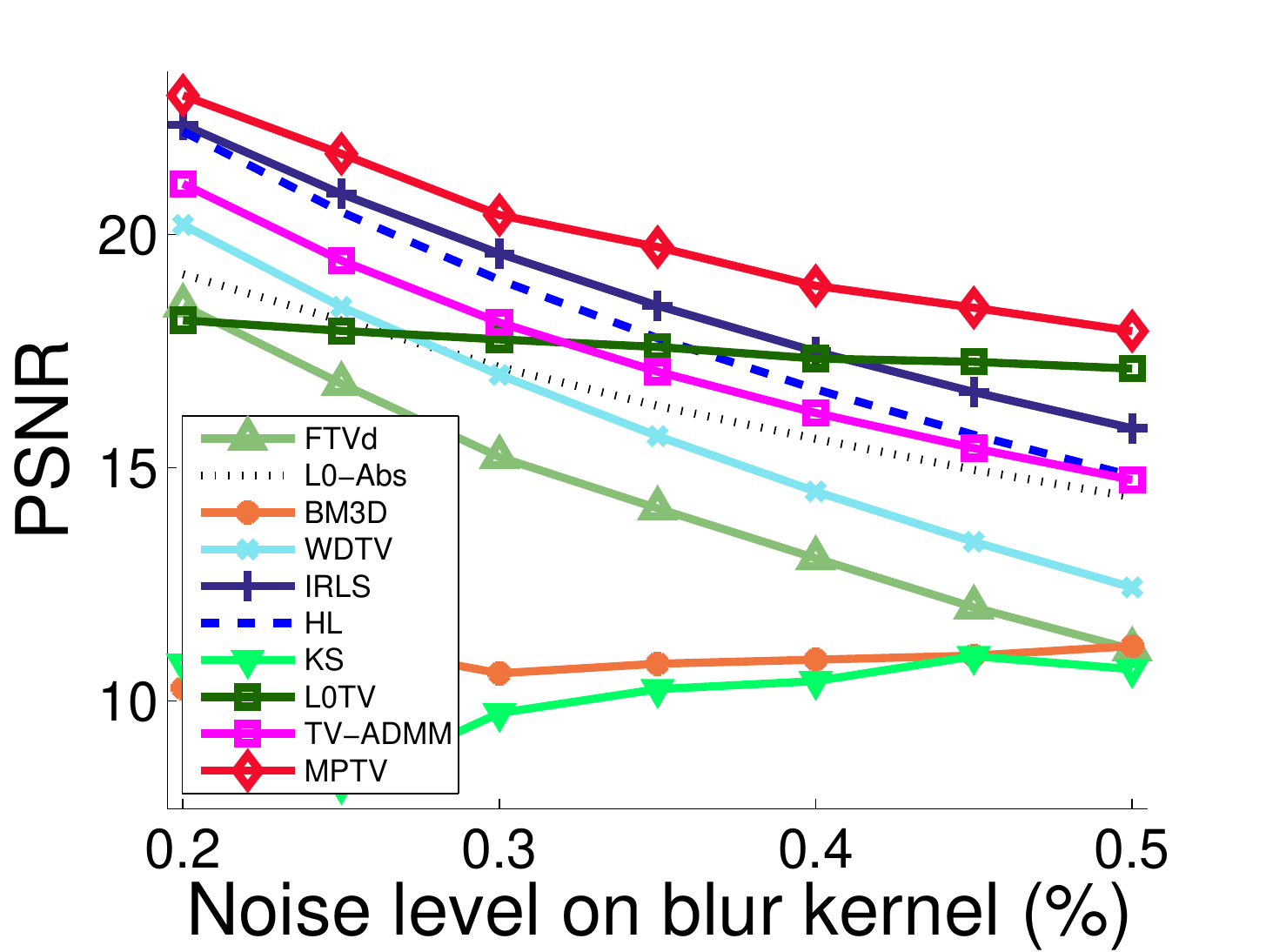}
}
\hfill
\subfigure[]{
\label{fig:blur_sens_d}
\centering
\includegraphics[trim =1.5mm 0mm 11mm 0mm, clip, width=0.23\linewidth]{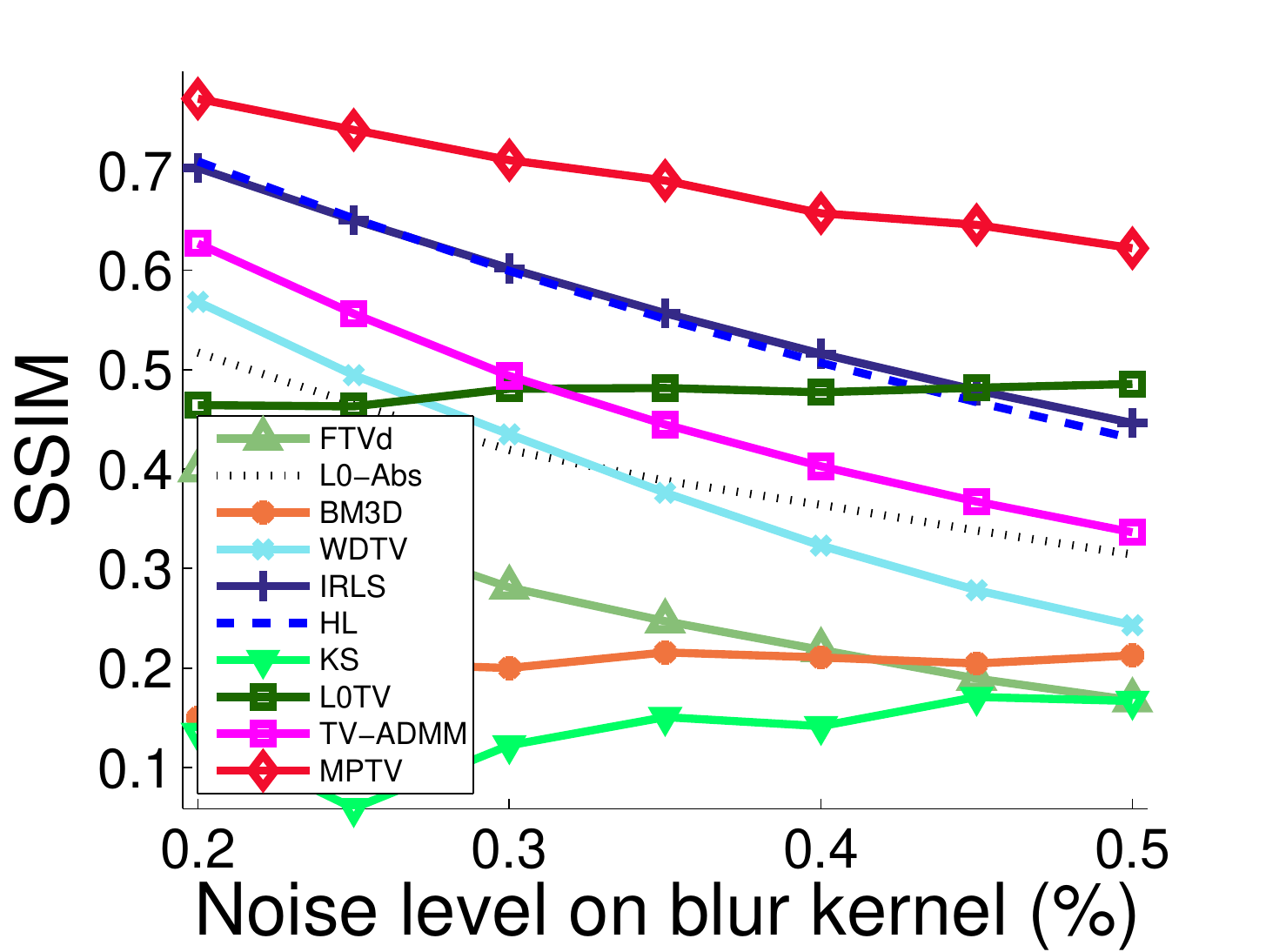}
}
\vspace{-0.2cm}
\caption{Study of  sensitivity to the error in the blur kernel estimate. Blur kernels with increasing noise level are given for deconvolution. PSNR and SSIM of the deconvolution results are evaluated for 7 blurry images contaminated by Gaussian blur or motion blur. (a) and (b) Results on the images blurred by Gaussian blur kernel. (c) and (d) Results on the images blurred by the motion blur kernel.}
\label{fig:syn_kerror_sens}
\end{figure*}

\subsection{Experiments on Real-world Images}
We \kui{evaluate} the performance of MPTV on two kinds of real-world blurred images, i.e. blurred text images and natural images. Here, the blur kernels are unknown. \dgrr{For real-world image deconvolution, the kernel estimation error is one main reason causing ringing artifacts~\cite{Shan2008High}.} For fair comparison, we use blur kernels estimated by the method in \cite{pan2014text} for all compared deconvolution methods.

\par
\kuire{We first report the results on a blurred text image in Fig. \ref{fig:real_text_hd}. From the figure, the deblurred results of most compared methods contain significant ringing artifacts and/or over-smooth strokes. In fact, since these methods usually use a universal criterion to suppress undesired components, such as thresholding on  gradients, they may cause over-smoothness (see Fig. \ref{fig:real_text_hd} (h)), or fail to suppress artifacts (see Fig. \ref{fig:real_text_hd} (b)). On the contrary, the proposed MPTV algorithm can recover sharper and clearer results by gradually activating significant strokes to fit the blurred observation (Fig. \ref{fig:real_text_hd} (k) and (l)).}

\par
In Fig. \ref{fig:real_harubang}, we show the deblurring results on a natural image containing both subtle textures and flat background. Although the ringing artifacts caused by the kernel estimation error are unavoidable,
MPTV performs favorably against the state-of-the-art competitors by alleviating ringing artifacts and preserving shape details simultaneously.
Fig.~\ref{fig:real_summerhouse} shows that, on an image with complex contents, the proposed MPTV can achieve favorable results than other methods. With a smaller regularization weight ($\lambda=0.00005$, Fig. \ref{fig:real_summerhouse} (j)), MPTV recovers much sharper and clearer details but still suppresses the ringing artifacts to quite a low level.

\begin{figure*}[htp]
\centering
\subfigure[Input $\by$ and $\bk$]{
\centering
\begin{minipage}[b]{.19\textwidth}
\centerline{
\begin{overpic}[width=1\textwidth]
{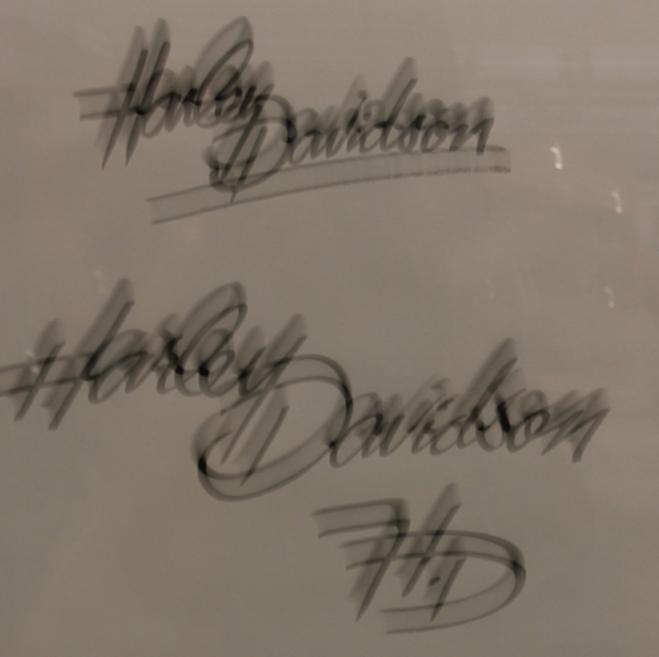}
\put(68,67){\includegraphics[width=0.32\linewidth]{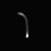}}
\put(69,69){\color{white}{\bf $\bk$}}
\end{overpic}}
\end{minipage}
}
\hspace{-0.4cm}
\subfigure[FTVd \cite{wang2008new}]{
\centering
\includegraphics[width=0.19\linewidth]{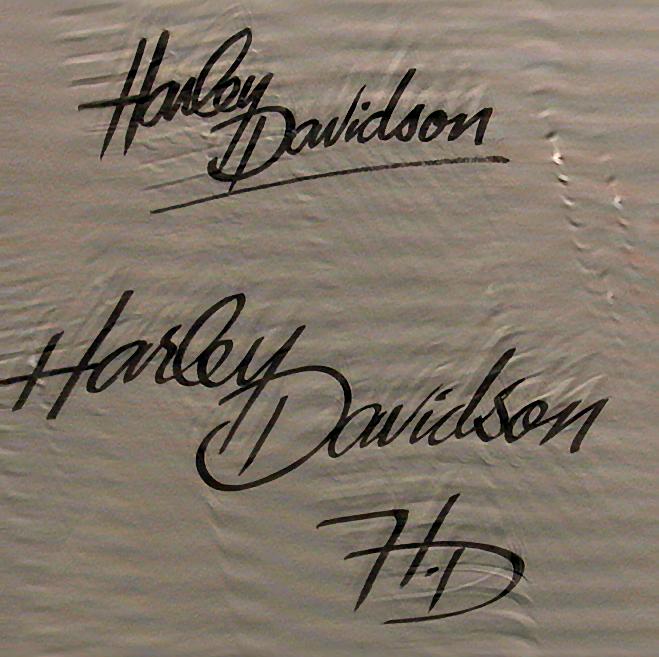}
}
\hspace{-0.4cm}
\subfigure[BM3D \cite{dabov2008bm3ddeb}]{
\centering
\includegraphics[width=0.19\linewidth]{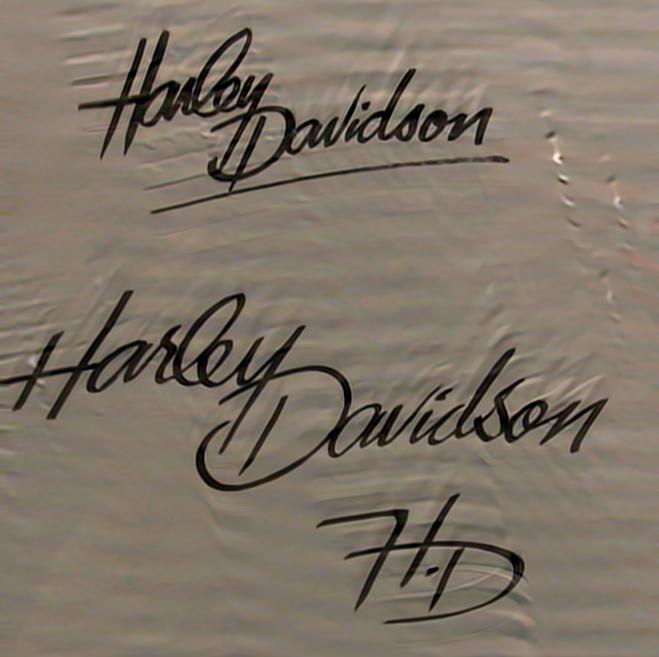}
}
\hspace{-0.4cm}
\subfigure[WDTV \cite{lou2015weighted}]{
\centering
\includegraphics[width=0.19\linewidth]{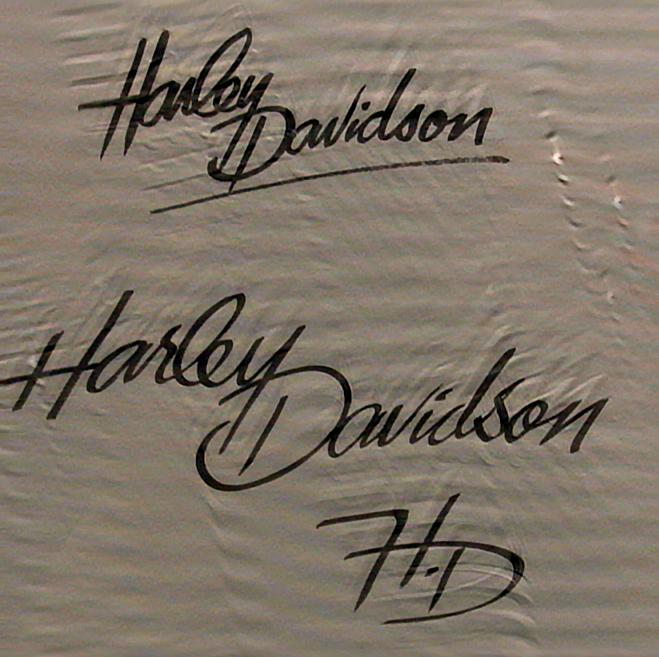}
}
\hspace{-0.4cm}
\subfigure[IRLS \cite{levin2007coded_irls}]{
\centering
\includegraphics[width=0.19\linewidth]{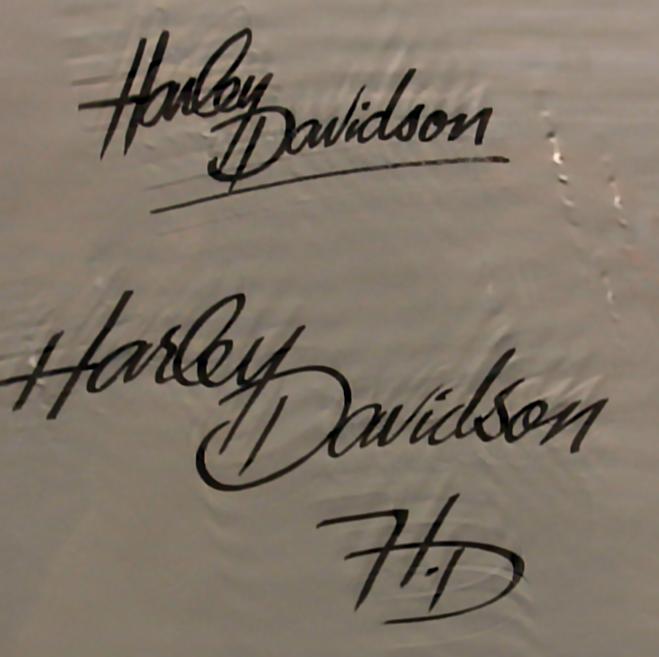}
}
\hspace{-0.4cm}
\subfigure[HL \cite{krishnan2009fast}]{
\centering
\includegraphics[width=0.19\linewidth]{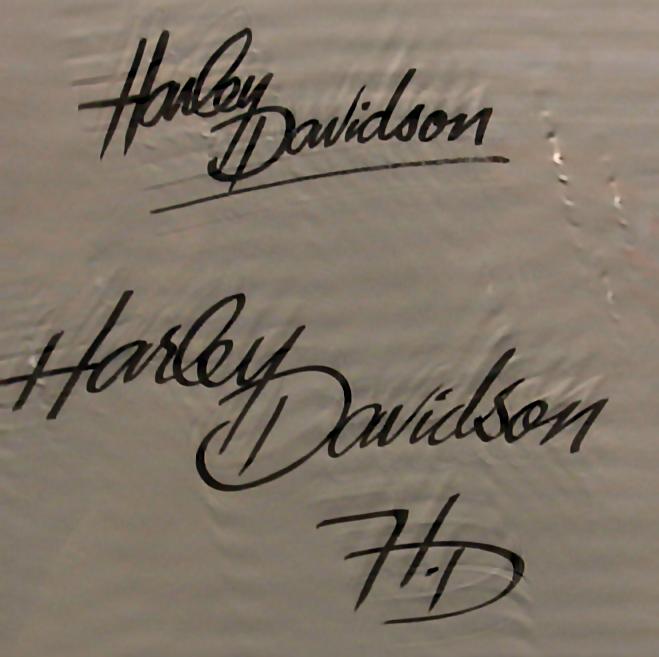}
}
\hspace{-0.4cm}
\subfigure[KS \cite{kheradmand2014Kernel}]{
\centering
\includegraphics[width=0.19\linewidth]{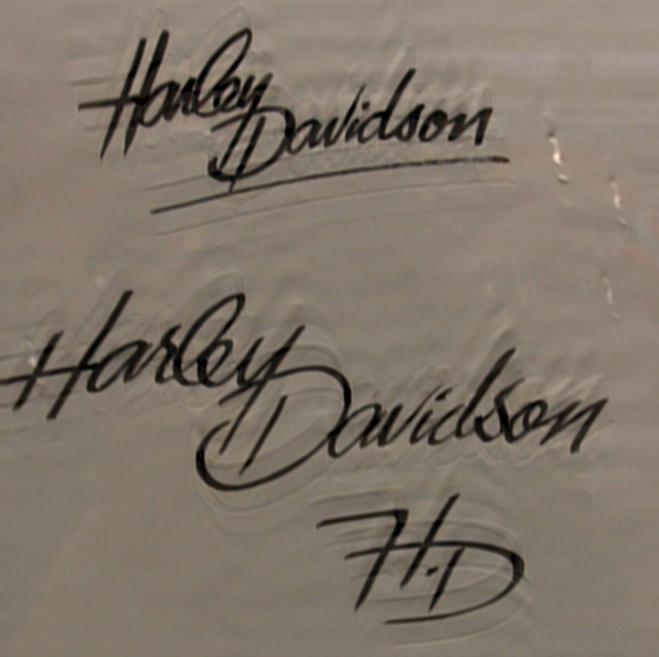}
}
\hspace{-0.4cm}
\subfigure[$\text{TV}_2$ \cite{pan2014text}]{
\centering
\includegraphics[width=0.19\linewidth]{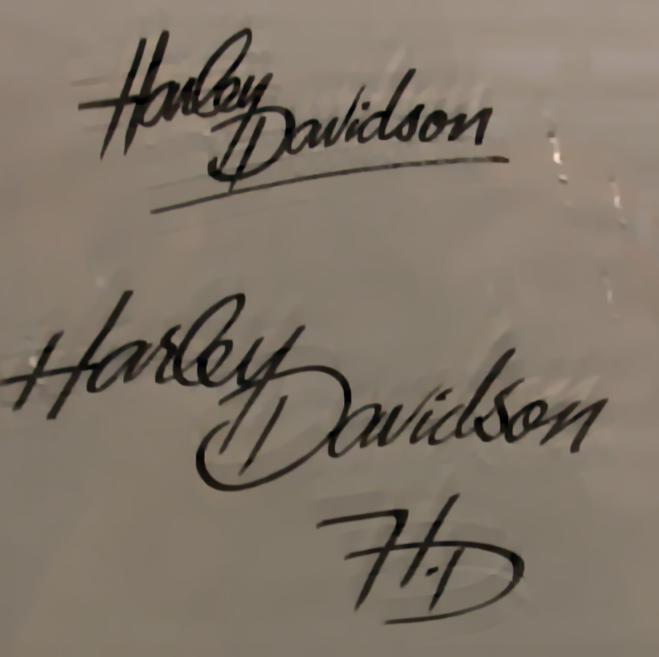}
}
\hspace{-0.4cm}
\subfigure[TV-ADMM]{
\centering
\includegraphics[width=0.19\linewidth]{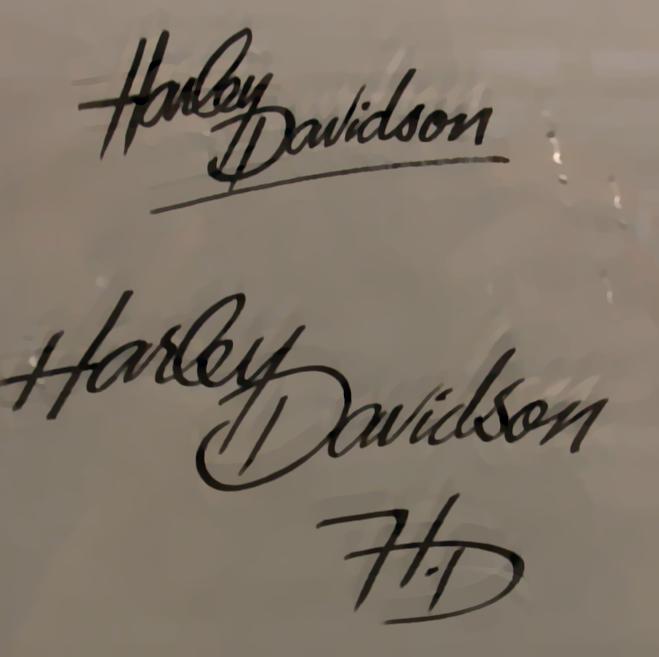}
}
\hspace{-0.4cm}
\subfigure[MPTV]{
\centering
\includegraphics[width=0.19\linewidth]{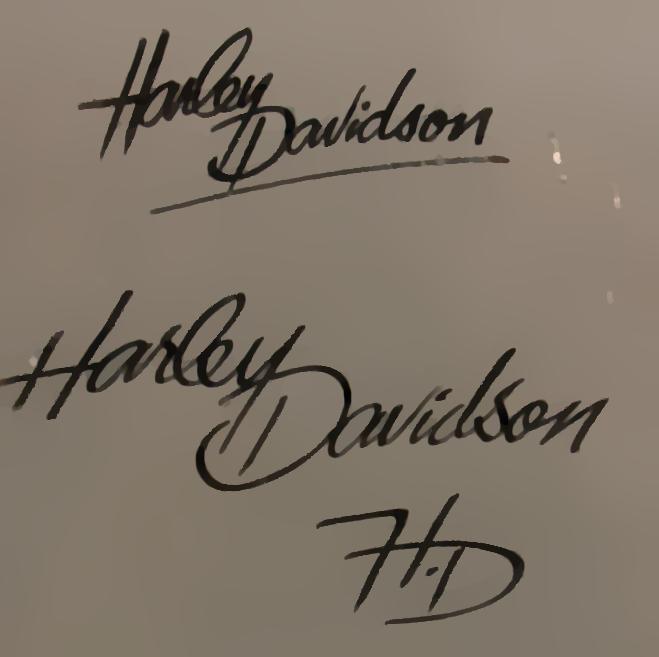}
}
\hspace{-0.4cm}
\subfigure[MPTV $\bg^t$]{
\begin{tabular}[]{c}
\begin{minipage}[b]{.157\textwidth}
\centerline{
\begin{overpic}[width=1\textwidth]
{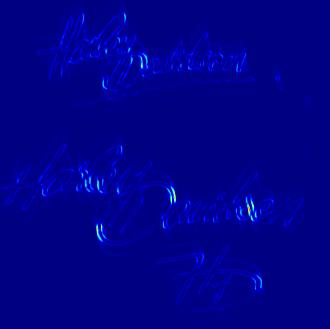}
\put(5,5){\color{white}{\bf $\bg^1$}}
\end{overpic}}
\end{minipage}
\hspace{-0.2cm}
\begin{minipage}[b]{.157\textwidth}
\centerline{
\begin{overpic}[width=1\textwidth]
{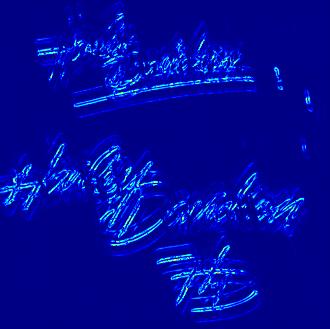}
\put(5,5){\color{white}{\bf $\bg^2$}}
\end{overpic}}
\end{minipage}
\hspace{-0.2cm}
\begin{minipage}[b]{.157\textwidth}
\centerline{
\begin{overpic}[width=1\textwidth]
{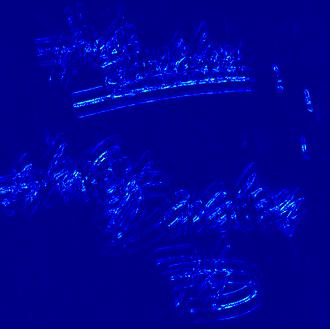}
\put(5,5){\color{white}{\bf $\bg^4$}}
\end{overpic}}
\end{minipage}
\end{tabular}
}
\hspace{-0.8cm}
\subfigure[MPTV $\bx^t$]{
\begin{tabular}[]{c}
\begin{minipage}[b]{.157\textwidth}
\centerline{
\begin{overpic}[width=1\textwidth]
{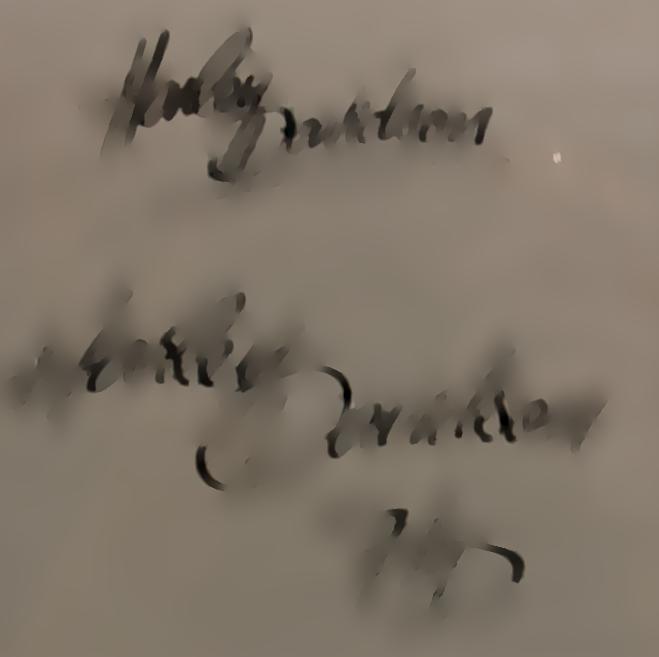}
\put(5,5){\color{white}{\bf $\bx^1$}}
\end{overpic}}
\end{minipage}
\hspace{-0.2cm}
\begin{minipage}[b]{.157\textwidth}
\centerline{
\begin{overpic}[width=1\textwidth]
{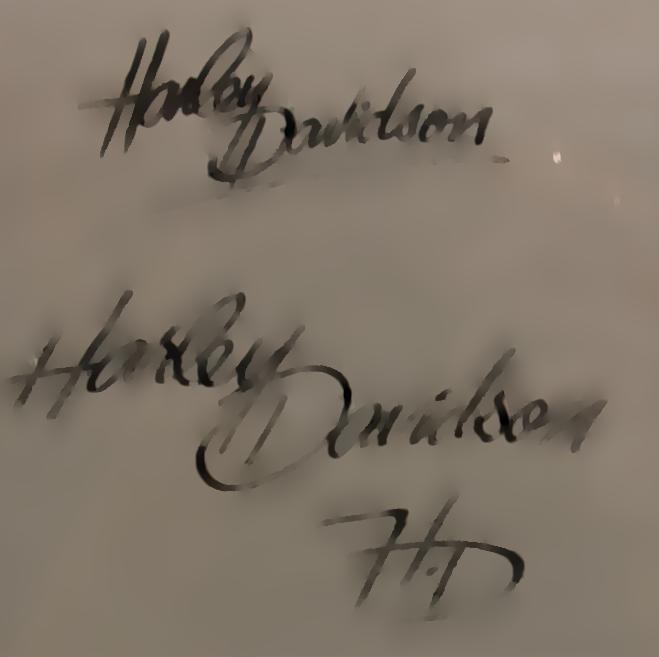}
\put(5,5){\color{white}{\bf $\bx^2$}}
\end{overpic}}
\end{minipage}
\hspace{-0.2cm}
\begin{minipage}[b]{.157\textwidth}
\centerline{
\begin{overpic}[width=1\textwidth]
{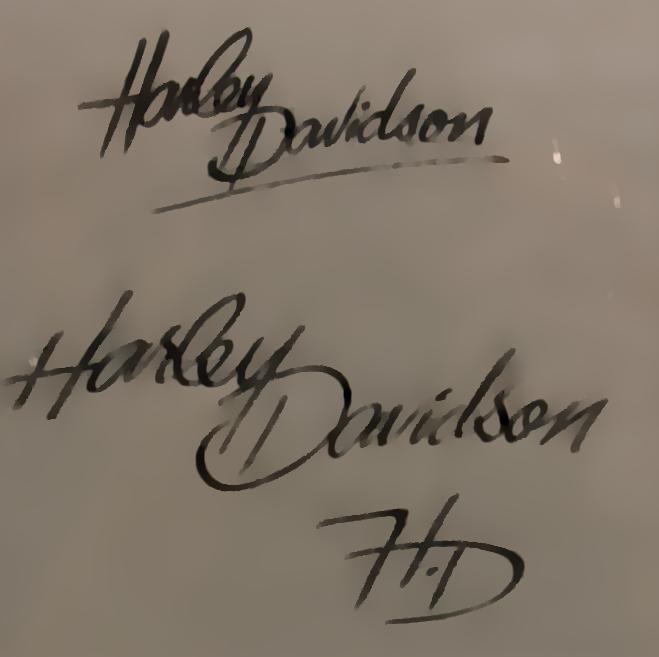}
\put(5,5){\color{white}{\bf $\bx^4$}}
\end{overpic}}
\end{minipage}
\end{tabular}
}
\caption{Experimental results on real text image. (a) Input blurred image $\by$ and blur kernel $\bk$ estimated using the method in \cite{pan2014text}. (b)-(j) Deconvolution results of different methods. The results using $\text{TV}_2$ in (h) are estimated using the method in \cite{pan2014text}. We show the intermediate results during the iterations of the proposed MPTV in (k) and (l). (k) $\bg^t$ for activating gradients. (l) Intermediate $\bx^t$. In this example, MPTV terminates after 6 iterations.}
\label{fig:real_text_hd}
\end{figure*}

\begin{figure*}[htp]
\centering
\subfigure[Input $\by$ and $\bk$]{
\centering
\begin{minipage}[b]{.19\textwidth}
\centerline{
\begin{overpic}[width=1\textwidth]
{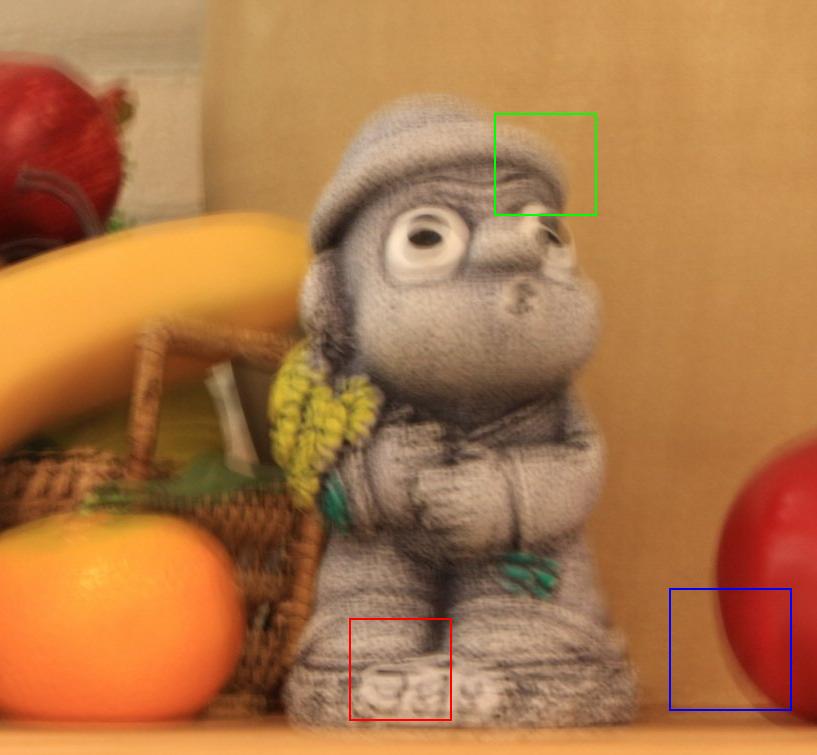}
\put(0,68){\includegraphics[width=0.24\linewidth]{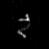}}
\put(76, 68.5){\includegraphics[width=0.235\linewidth]{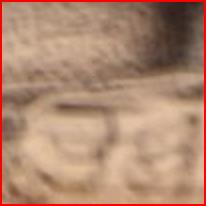}}
\put(76, 44.5){\includegraphics[width=0.235\linewidth]{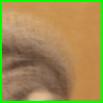}}
\put(2,69){\color{white}{\bf $\bk$}}
\end{overpic}}
\end{minipage}
}
\hspace{-0.4cm}
\subfigure[Xu and Jia \cite{xu2010two}]{
\centering
\begin{minipage}[b]{.19\textwidth}
\centerline{
\begin{overpic}[width=1\textwidth]
{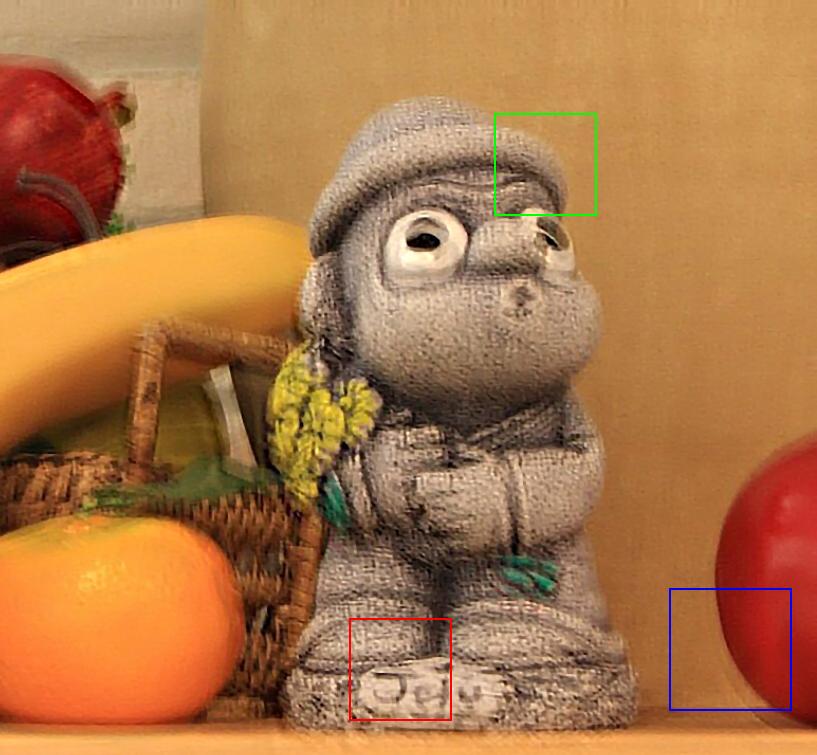}
\put(76, 68.5){\includegraphics[width=0.235\linewidth]{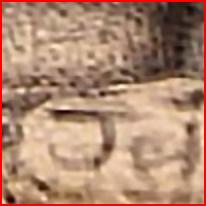}}
\put(76, 44.5){\includegraphics[width=0.235\linewidth]{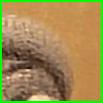}}
\end{overpic}}
\end{minipage}
}
\hspace{-0.4cm}
\subfigure[FTVd \cite{wang2008new}]{
\centering
\begin{minipage}[b]{.19\textwidth}
\centerline{
\begin{overpic}[width=1\textwidth]
{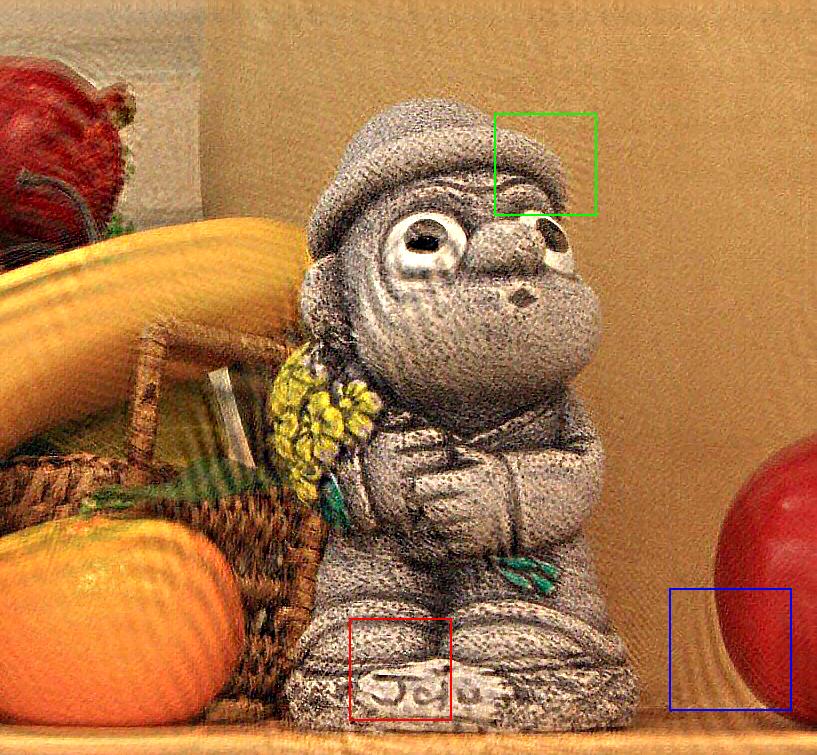}
\put(76, 68.5){\includegraphics[width=0.235\linewidth]{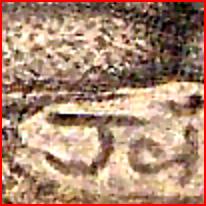}}
\put(76, 44.5){\includegraphics[width=0.235\linewidth]{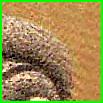}}
\end{overpic}}
\end{minipage}
}
\hspace{-0.4cm}
\subfigure[BM3D \cite{dabov2008bm3ddeb}]{
\centering
\begin{minipage}[b]{.19\textwidth}
\centerline{
\begin{overpic}[width=1\textwidth]
{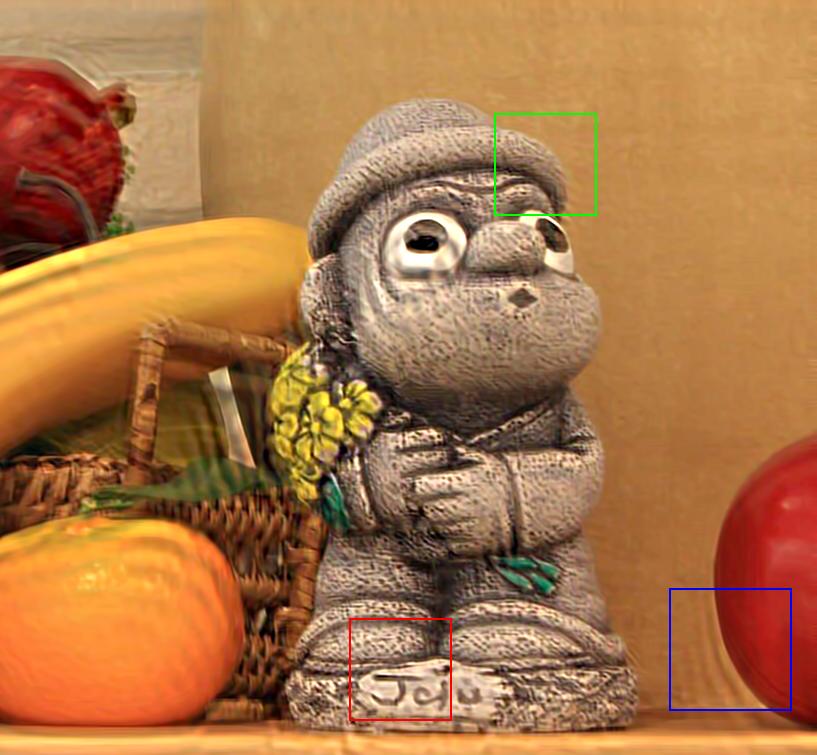}
\put(76, 68.5){\includegraphics[width=0.235\linewidth]{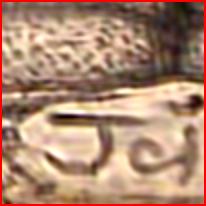}}
\put(76, 44.5){\includegraphics[width=0.235\linewidth]{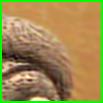}}
\end{overpic}}
\end{minipage}
}
\hspace{-0.4cm}
\subfigure[WDTV \cite{lou2015weighted}]{
\centering
\begin{minipage}[b]{.19\textwidth}
\centerline{
\begin{overpic}[width=1\textwidth]
{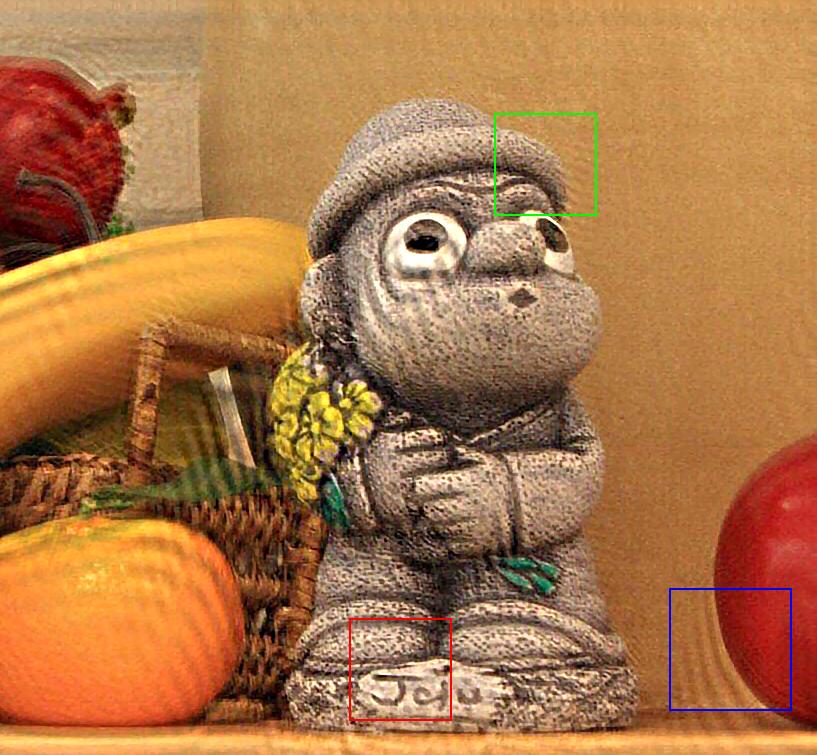}
\put(76, 68.5){\includegraphics[width=0.235\linewidth]{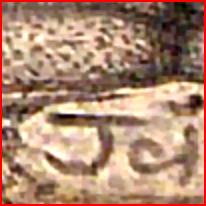}}
\put(76, 44.5){\includegraphics[width=0.235\linewidth]{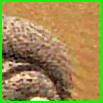}}
\end{overpic}}
\end{minipage}
}
\hspace{-0.4cm}
\subfigure[IRLS \cite{levin2007coded_irls}]{
\centering
\begin{minipage}[b]{.19\textwidth}
\centerline{
\begin{overpic}[width=1\textwidth]
{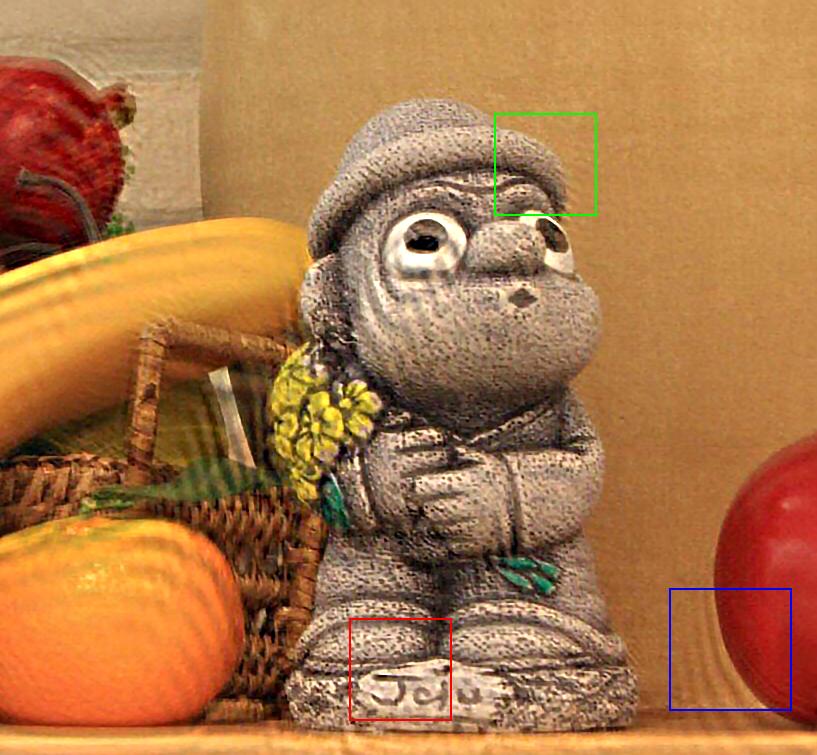}
\put(76, 68.5){\includegraphics[width=0.235\linewidth]{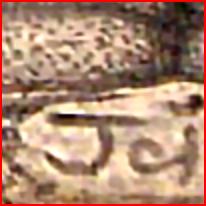}}
\put(76, 44.5){\includegraphics[width=0.235\linewidth]{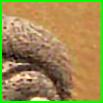}}
\end{overpic}}
\end{minipage}
}
\hspace{-0.4cm}
\subfigure[HL \cite{krishnan2009fast}]{
\centering
\begin{minipage}[b]{.19\textwidth}
\centerline{
\begin{overpic}[width=1\textwidth]
{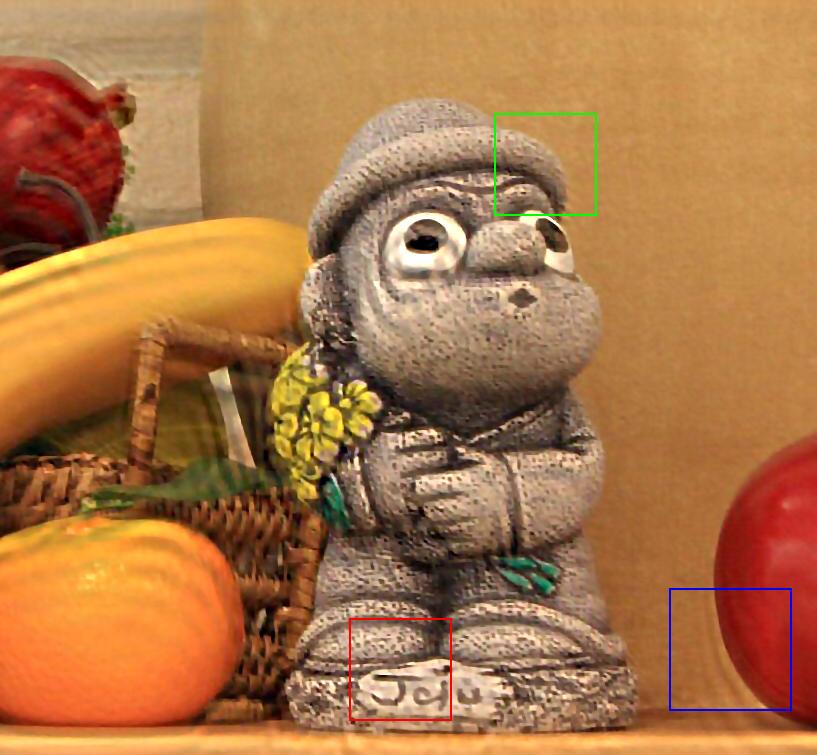}
\put(76, 68.5){\includegraphics[width=0.235\linewidth]{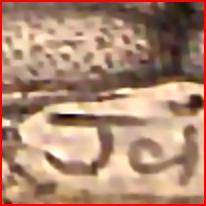}}
\put(76, 44.5){\includegraphics[width=0.235\linewidth]{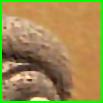}}
\end{overpic}}
\end{minipage}
}
\hspace{-0.4cm}
\subfigure[KS \cite{kheradmand2014Kernel}]{
\centering
\begin{minipage}[b]{.19\textwidth}
\centerline{
\begin{overpic}[width=1\textwidth]
{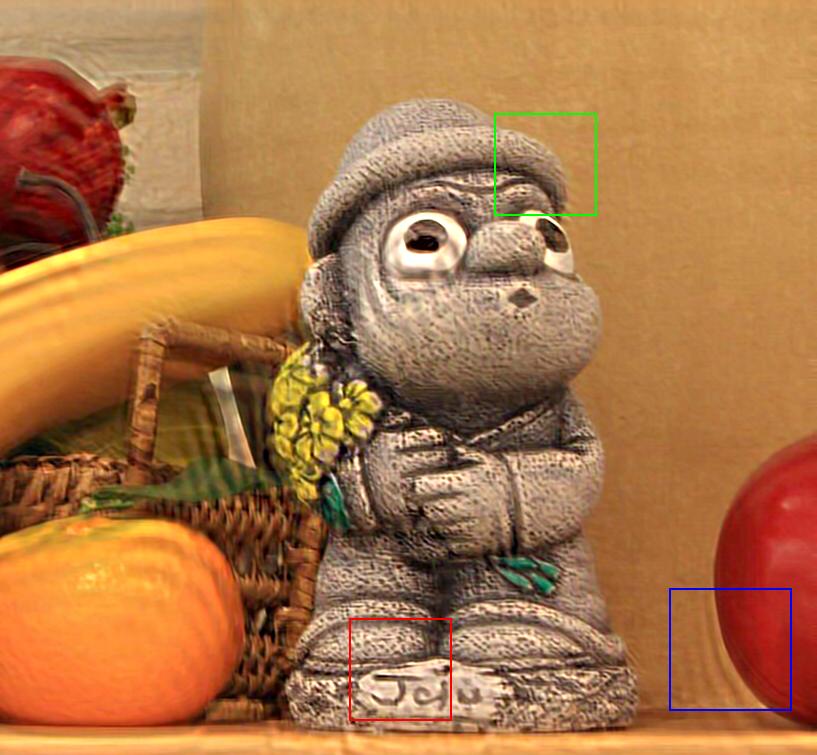}
\put(76, 68.5){\includegraphics[width=0.235\linewidth]{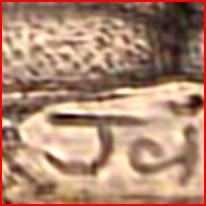}}
\put(76, 44.5){\includegraphics[width=0.235\linewidth]{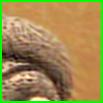}}
\end{overpic}}
\end{minipage}
}
\hspace{-0.4cm}
\subfigure[TVADMM]{
\centering
\begin{minipage}[b]{.19\textwidth}
\centerline{
\begin{overpic}[width=1\textwidth]
{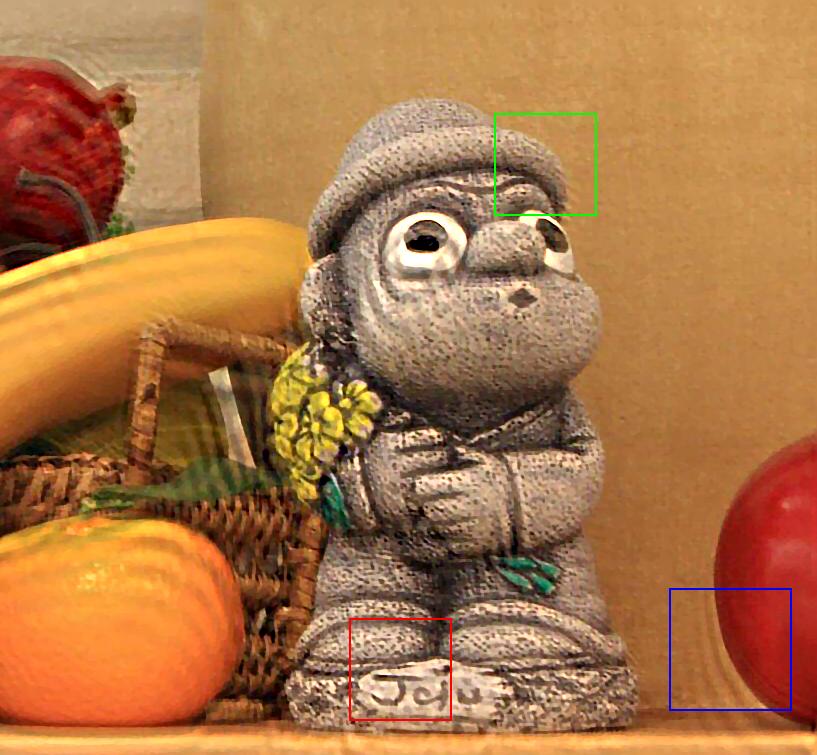}
\put(76, 68.5){\includegraphics[width=0.235\linewidth]{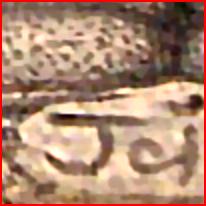}}
\put(76, 44.5){\includegraphics[width=0.235\linewidth]{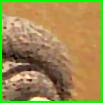}}
\end{overpic}}
\end{minipage}
}
\hspace{-0.4cm}
\subfigure[MPTV ($\lambda=0.00005$) ]{
\centering
\begin{minipage}[b]{.19\textwidth}
\centerline{
\begin{overpic}[width=1\textwidth]
{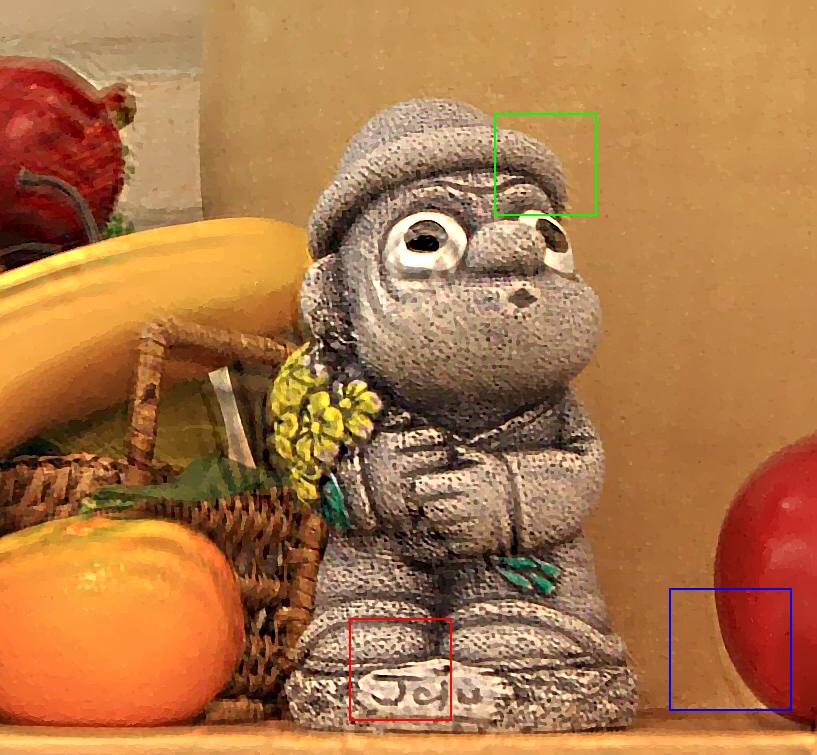}
\put(76, 68.5){\includegraphics[width=0.235\linewidth]{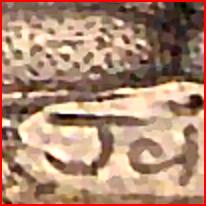}}
\put(76, 44.5){\includegraphics[width=0.235\linewidth]{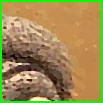}}
\end{overpic}}
\end{minipage}
}
\hspace{-0.4cm}
\caption{Experimental results on a real-world image. (a) Input blurred image $\by$ and blur kernel $\bk$. (b) Deblurring result of the software in \cite{xu2010two} is used as a baseline. (c)-(j) Deconvolution results of different methods. The proposed MPTV recovers more sharp details (see red box) and introduces fewer ringing artifacts (see blue box).}
\label{fig:real_harubang}
\end{figure*}

\begin{figure*}[htp]
\centering
\subfigure[Input $\by$ and $\bk$]{
\centering
\begin{minipage}[b]{.19\textwidth}
\centerline{
\begin{overpic}[width=1\textwidth]
{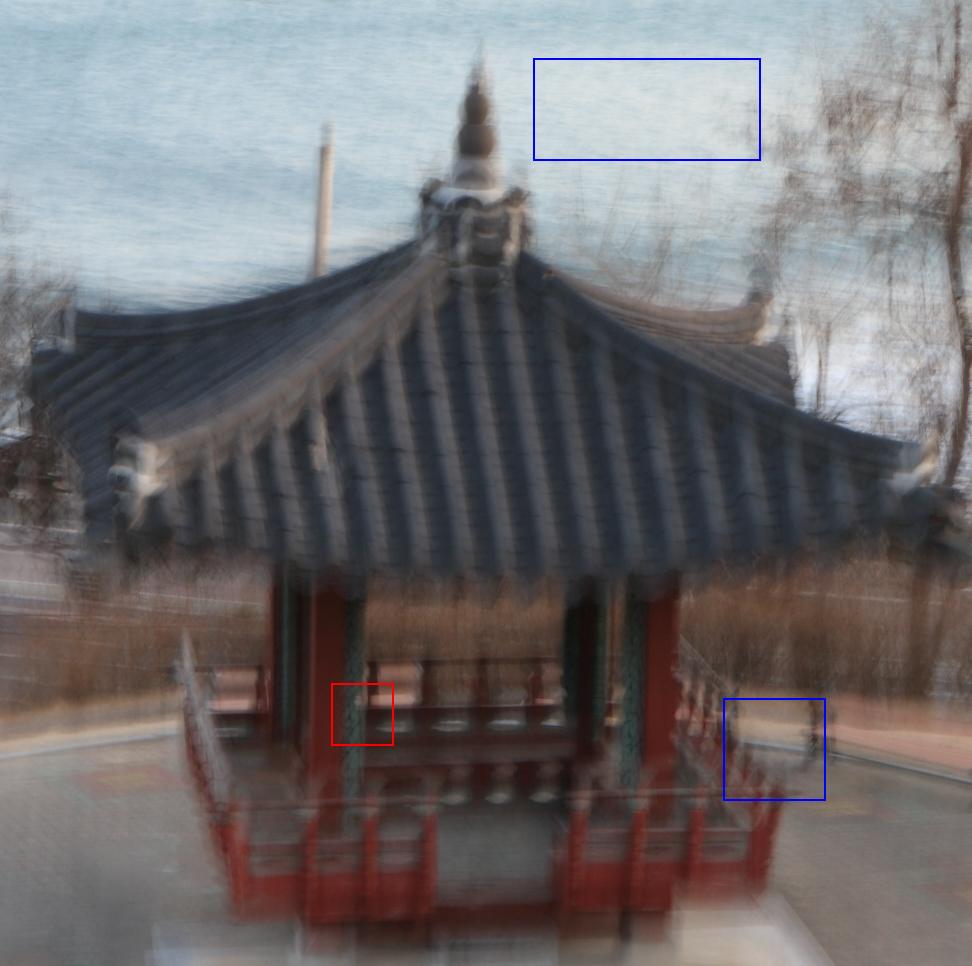}
\put(0,67){\includegraphics[width=0.32\linewidth]{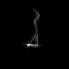}}
\put(1,1){\includegraphics[width=0.24\linewidth]{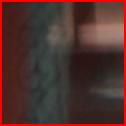}}
\put(2,69){\color{white}{\bf $\bk$}}
\end{overpic}}
\end{minipage}
}
\hspace{-0.4cm}
\subfigure[Xu and Jia \cite{xu2010two}]{
\centering
\begin{minipage}[b]{.19\textwidth}
\centerline{
\begin{overpic}[width=1\textwidth]
{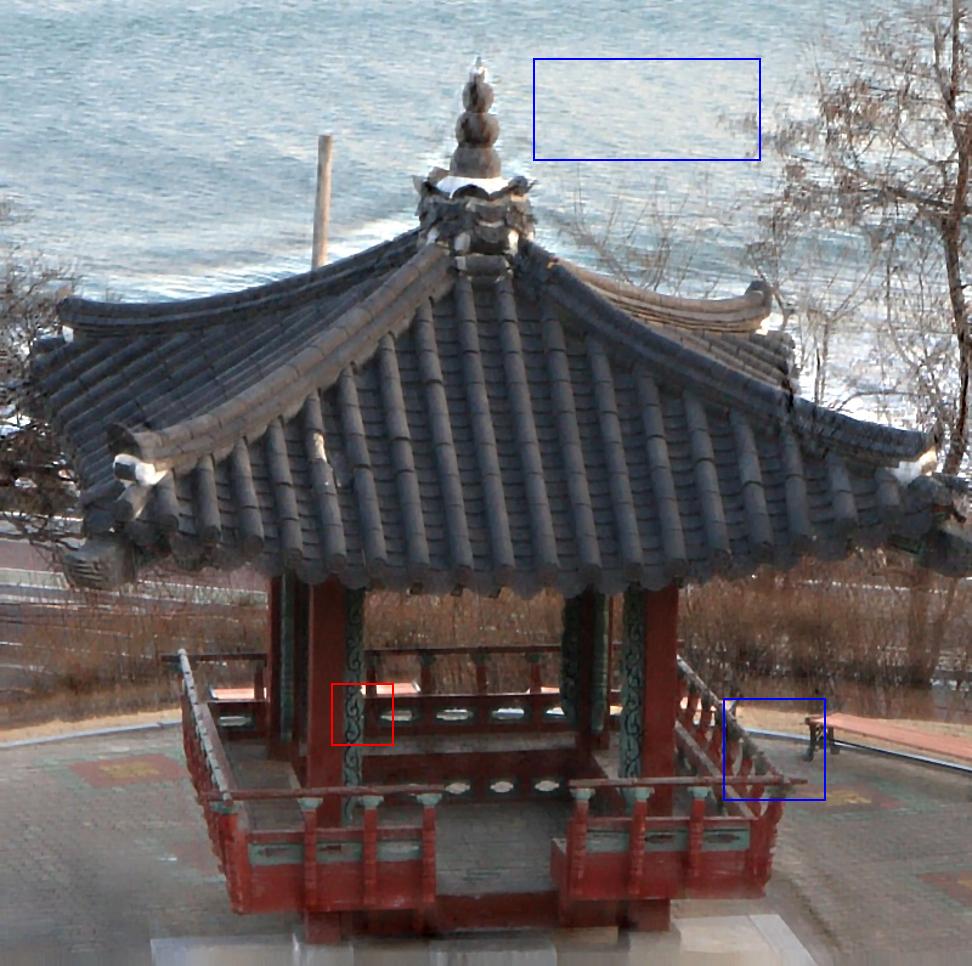}
\put(1,1){\includegraphics[width=0.24\linewidth]{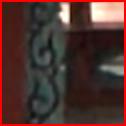}}
\end{overpic}}
\end{minipage}
}
\hspace{-0.4cm}
\subfigure[FTVd \cite{wang2008new}]{
\centering
\begin{minipage}[b]{.19\textwidth}
\centerline{
\begin{overpic}[width=1\textwidth]
{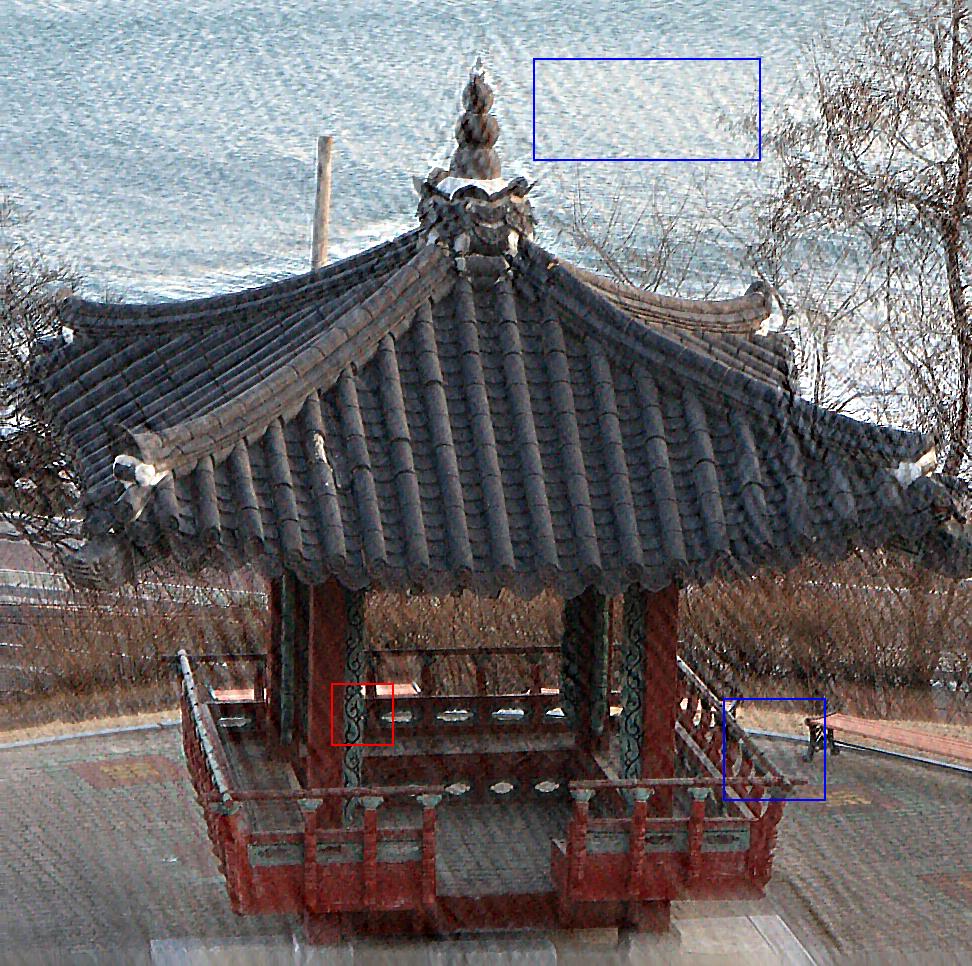}
\put(1,1){\includegraphics[width=0.24\linewidth]{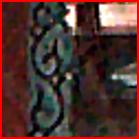}}
\end{overpic}}
\end{minipage}
}
\hspace{-0.4cm}
\subfigure[BM3D \cite{dabov2008bm3ddeb}]{
\centering
\begin{minipage}[b]{.19\textwidth}
\centerline{
\begin{overpic}[width=1\textwidth]
{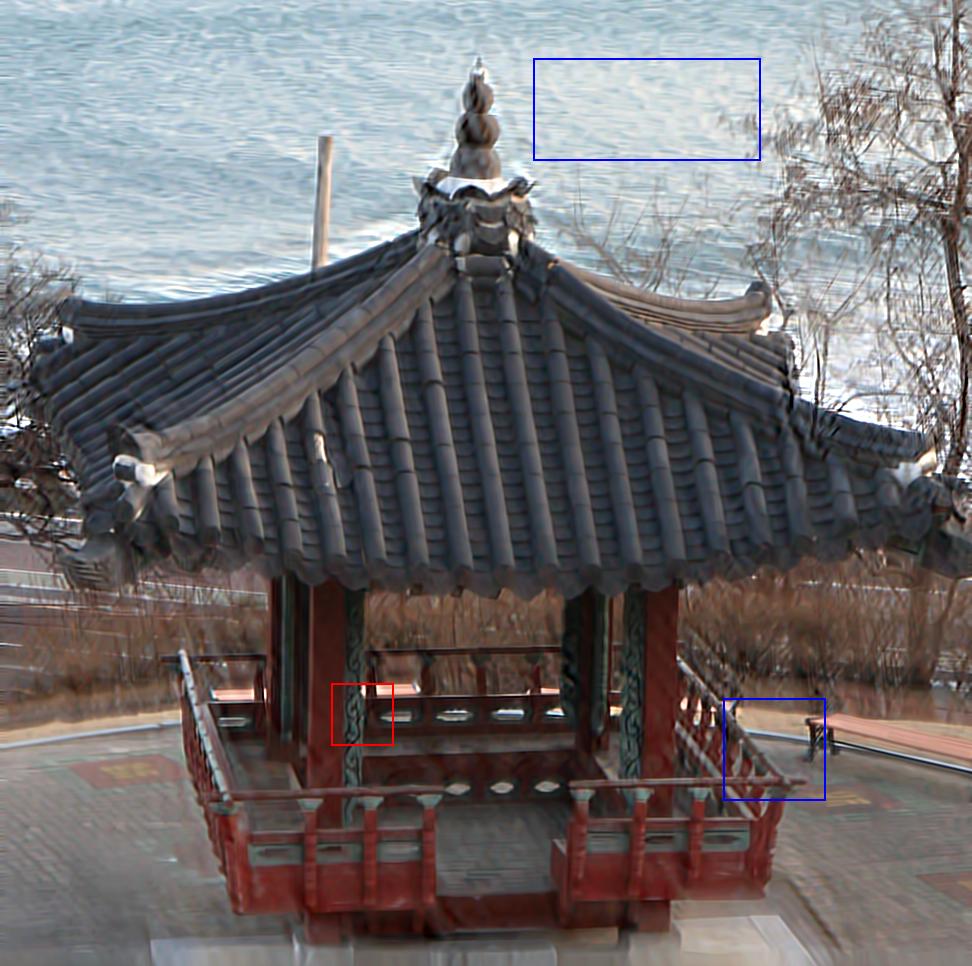}
\put(1,1){\includegraphics[width=0.24\linewidth]{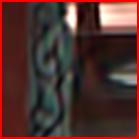}}
\end{overpic}}
\end{minipage}
}
\hspace{-0.4cm}
\subfigure[WDTV \cite{lou2015weighted}]{
\centering
\begin{minipage}[b]{.19\textwidth}
\centerline{
\begin{overpic}[width=1\textwidth]
{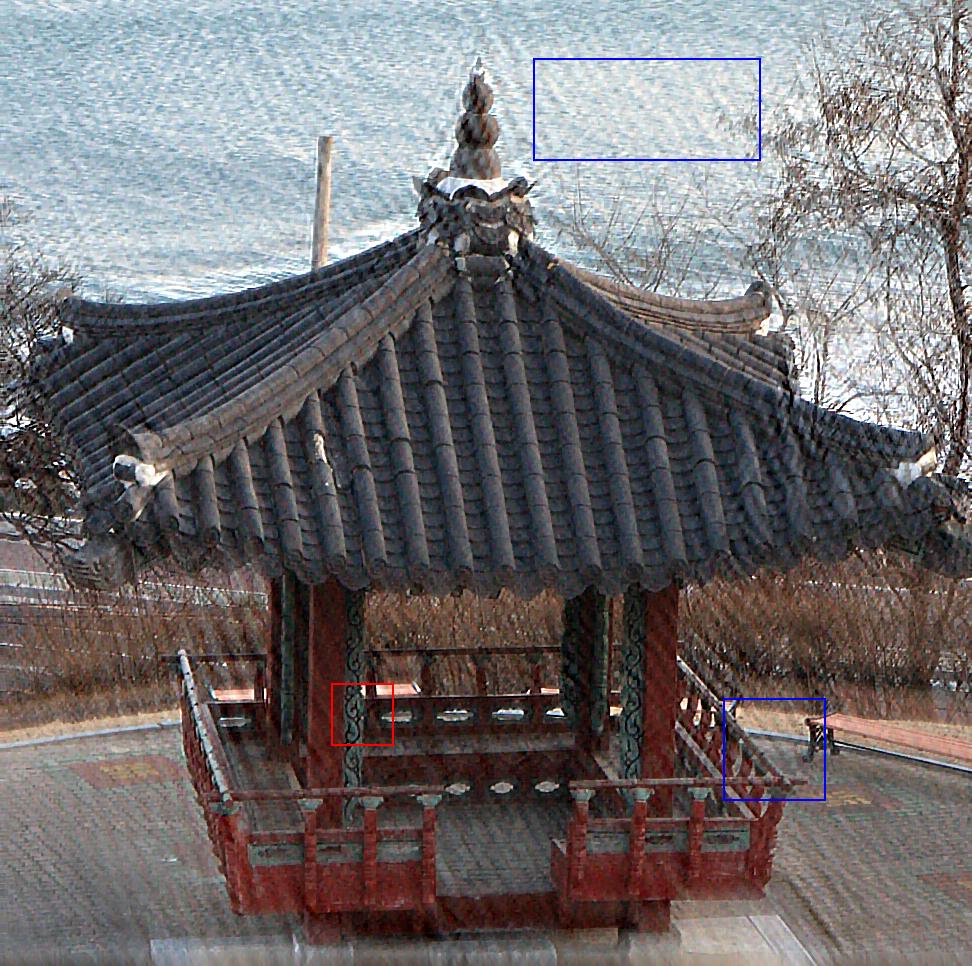}
\put(1,1){\includegraphics[width=0.24\linewidth]{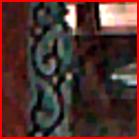}}
\end{overpic}}
\end{minipage}
}
\hspace{-0.4cm}
\subfigure[IRLS \cite{levin2007coded_irls}]{
\centering
\begin{minipage}[b]{.19\textwidth}
\centerline{
\begin{overpic}[width=1\textwidth]
{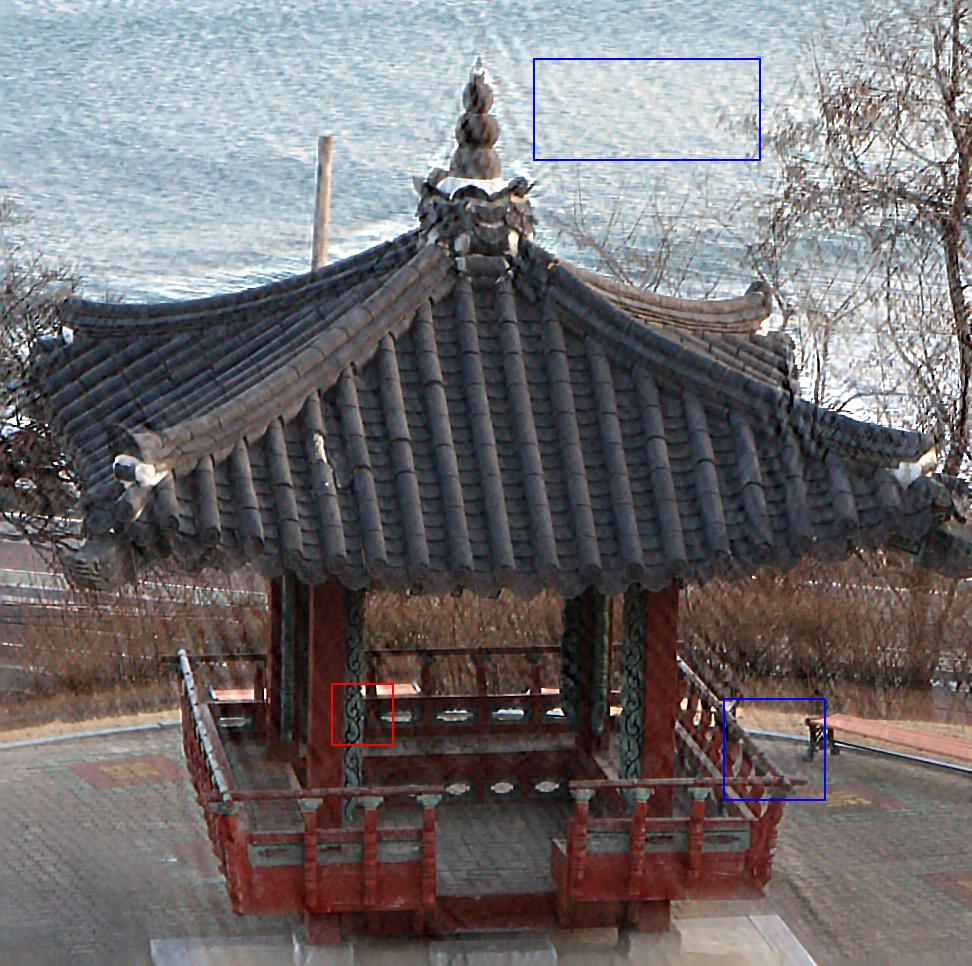}
\put(1,1){\includegraphics[width=0.24\linewidth]{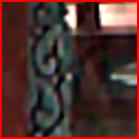}}
\end{overpic}}
\end{minipage}
}
\hspace{-0.4cm}
\subfigure[HL \cite{krishnan2009fast}]{
\centering
\begin{minipage}[b]{.19\textwidth}
\centerline{
\begin{overpic}[width=1\textwidth]
{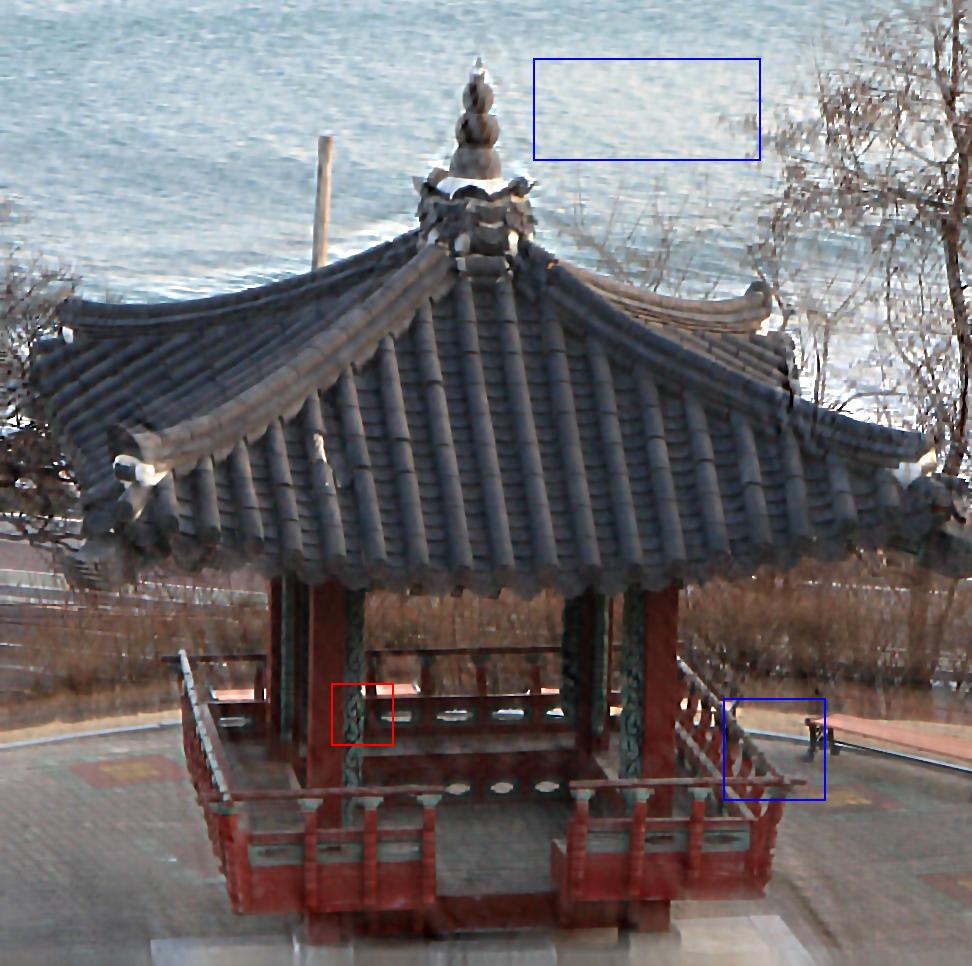}
\put(1,1){\includegraphics[width=0.24\linewidth]{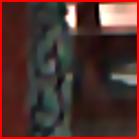}}
\end{overpic}}
\end{minipage}
}
\hspace{-0.4cm}
\subfigure[KS \cite{kheradmand2014Kernel}]{
\centering
\begin{minipage}[b]{.19\textwidth}
\centerline{
\begin{overpic}[width=1\textwidth]
{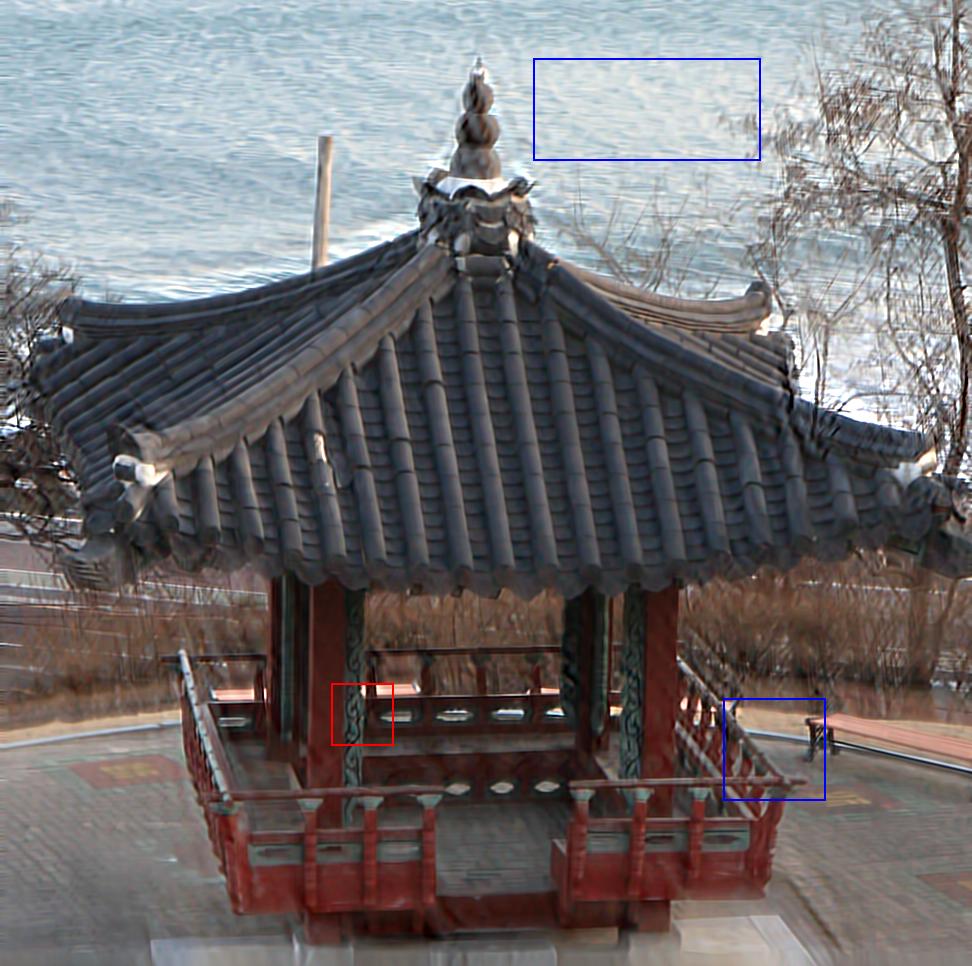}
\put(1,1){\includegraphics[width=0.24\linewidth]{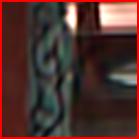}}
\end{overpic}}
\end{minipage}
}
\hspace{-0.4cm}
\subfigure[MPTV  ($\lambda=0.0005$)]{
\centering
\begin{minipage}[b]{.19\textwidth}
\centerline{
\begin{overpic}[width=1\textwidth]
{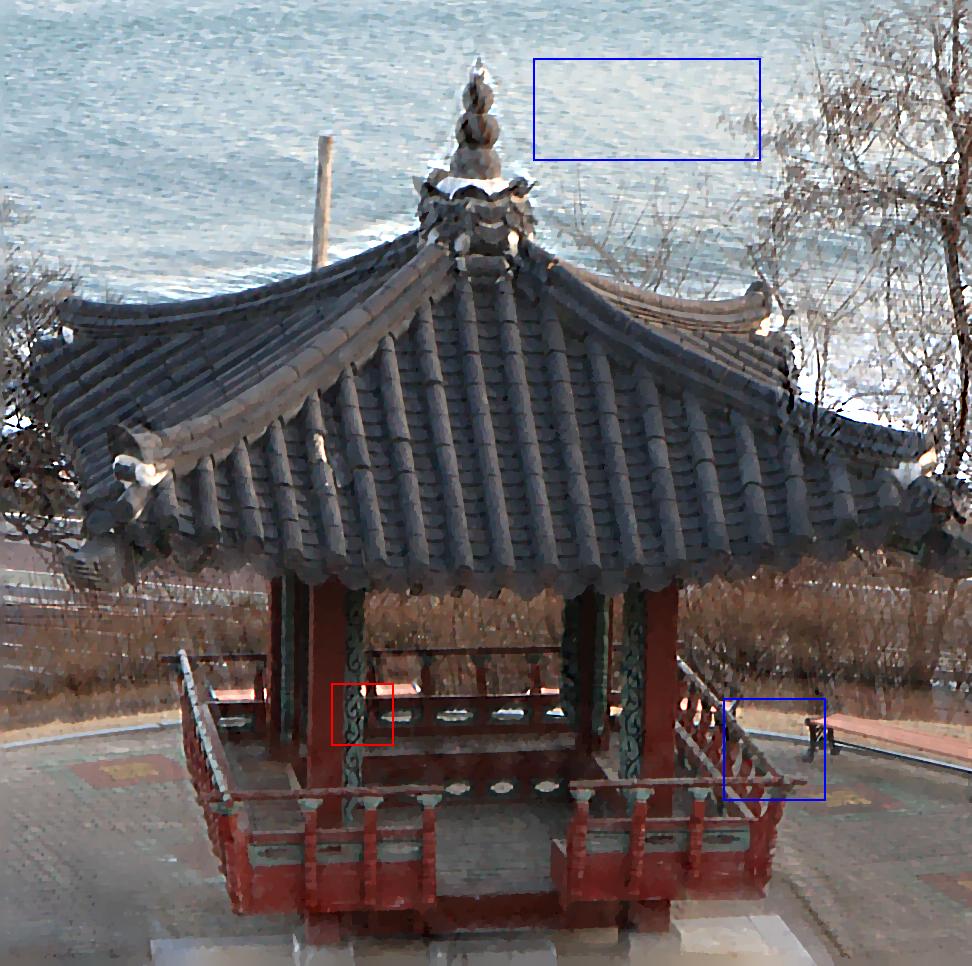}
\put(1,1){\includegraphics[width=0.24\linewidth]{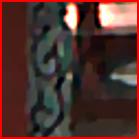}}
\end{overpic}}
\end{minipage}
}
\hspace{-0.4cm}
\subfigure[MPTV ($\lambda=0.00005$) ]{
\centering
\begin{minipage}[b]{.19\textwidth}
\centerline{
\begin{overpic}[width=1\textwidth]
{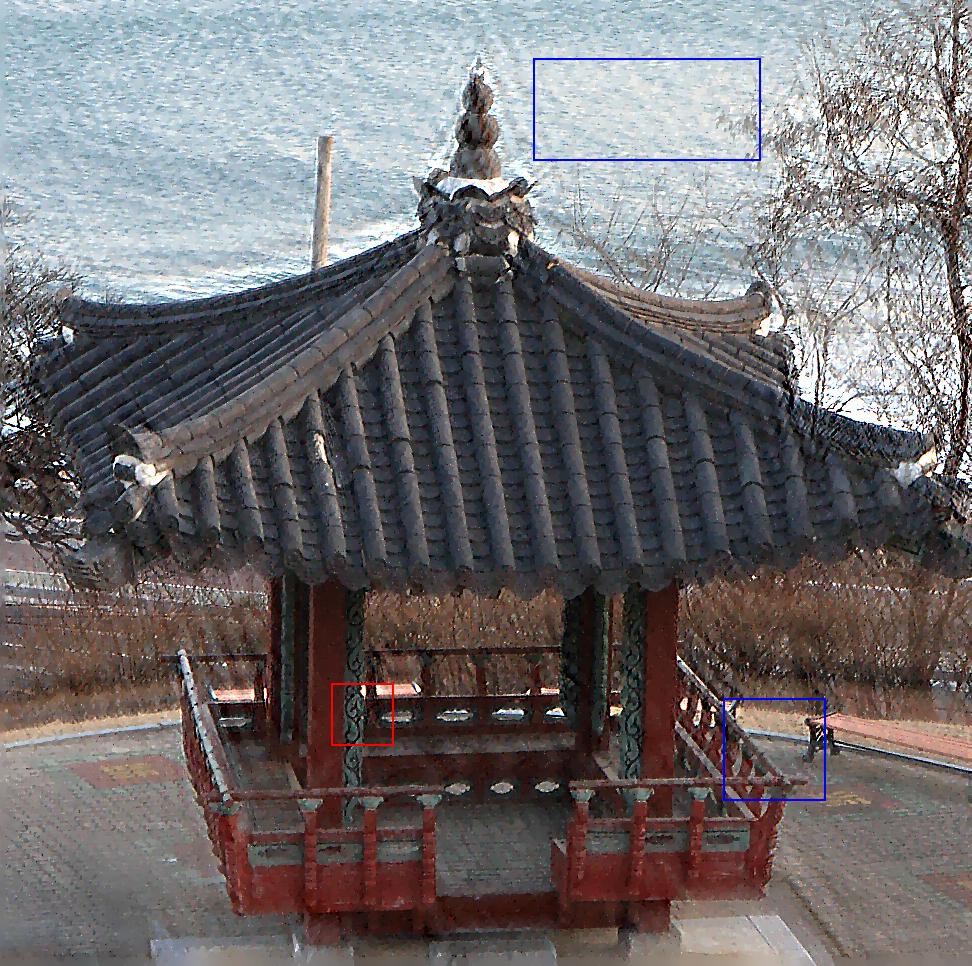}
\put(1,1){\includegraphics[width=0.24\linewidth]{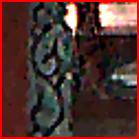}}
\end{overpic}}
\end{minipage}
}
\hspace{-0.4cm}
\caption{Experimental results on a real-world image. (a) Input blurred image $\by$ and blur kernel $\bk$. (b) Deblurring result of the software in \cite{xu2010two} is used as a baseline. (c)-(j) Deconvolution results of different methods. The results of MPTV with different values for parameter $\lambda$ are shown in (i) and (j). The results of MPTV suffer from fewer ringing artifacts (see blue boxes).}
\label{fig:real_summerhouse}
\end{figure*}

\section{Conclusion}
\label{sec:con}
We have proposed a matching pursuit based total variation minimization method for image deconvolution. By introducing a binary vector to indicate the significant components in the image gradients, we formulate the TV minimization problem as a QCLP problem and solve it via a cutting-plane method. The proposed algorithm with early stopping helps to reduce the solution bias, to alleviate over-smoothness, and suppress the ringing artifacts in the deconvolution results. Comprehensive empirical studies on many different datasets show the superior performance of the proposed MPTV over the state-of-the-art methods.

In the future, we plan to apply the proposed method to a variety of related image processing tasks, such image super-resolution. We \kui{are} also interested to extend the matching pursuit framework to the wavelet frame based on the models introduced in \cite{he2014waveletinpainting}.

\appendices

\section{Proof of Proposition \ref{prop:mptv_dual}} \label{sec:app_dual_derv}

\begin{proof}
For any fixed $\btau$, we have the inner problem as
\begin{equation}\label{eq:prob_mptv_inner}
\begin{split}
  \min_{\bx,\bxi,\bz}  &~~\frac{1}{2}\|\bxi\|_2^2 + \lambda\sum\nolimits_{i=1}^n \|\bd_i\|_2\\
  \st &~~\bxi=\by-\bA\bx,~\bD\bx=(\bz\odot\tbtau), \\
  &~~\bd_i=\bC_i\bz, \forall i\in [n].\\
\end{split}
\end{equation}
By introducing Lagrangian dual variables $\balpha$, $\bbeta$ and $\bgamma$ with $\bgamma_i=\bC_i\bgamma$ to these equality constraints, the Lagrangian function of the inner problem \eqref{eq:prob_mptv_inner} is:
\begin{equation}\label{eq:mptv_lag}
  \begin{split}
  \cL&(\bx,\bz,\bxi, \balpha, \bbeta, \bgamma)=\frac{1}{2} \|\bxi\|_2^2 + \lambda\sum\nolimits_{i=1}^n \|\bd_i\|_2 \\
  & + \balpha^\T(\by-\bA\bx-\bxi) +\bbeta^\T(\bD\bx-\diag(\widetilde{\btau})\bz)\\
  & +\sum\nolimits_{i=1}^n \bgamma_i^\T(\bC_i\bz-\bd_i).
\end{split}
\nonumber
\end{equation}
We then minimize $\cL(\bx,\bz,\bxi, \balpha, \bbeta, \bgamma)$ w.r.t. $\bxi$, $\bx$ and $\bz$, respectively:
\begin{equation}
  \min_{\bxi}~ \frac{1}{2}\|\bxi\|_2^2 - \balpha^\T\bxi + \balpha^\T\by = -\frac{1}{2}\|\balpha\|_2^2 + \balpha^\T\by,
  \nonumber
\end{equation}
\begin{equation}
  \min_{\bx}~ -\balpha^\T\bA\bx + \bbeta^\T\bD\bx = \left\{
      \begin{array}{l}
        ~0,~~~~ \text{if}~ \bA^\T\balpha =\bD^\T\bbeta,\\
        -\infty, ~\text{otherwise},\\
      \end{array}
      \right.
      \nonumber
\end{equation}
\begin{equation}
  \min_{\bz}~ -\bbeta^\T\diag(\tbtau)\bz + \bgamma^\T\bz= \left\{
      \begin{array}{l}
        ~0,~~~~ \text{if}~ \bgamma=\diag(\tbtau)\bbeta,\\
        -\infty, ~\text{otherwise},\\
      \end{array}
      \right.
      \nonumber
\end{equation}
\begin{equation}
\begin{split}
  & \min_{\bd}~ \sum\nolimits_{i=1}^n \left(\lambda\|\bd_i\|_2 - \bgamma_i^\T\bd_i \right)\\
  & = \min_{\bd} \max_{\|\bomega_i\|_2\leq 1, \forall i} \sum\nolimits_{i=1}^n \lambda \bomega_i^\T \bd_i - \bgamma_i^\T\bd_i \\
  & \geq \max_{\|\bomega_i\|_2\leq 1, \forall i} \min_{\bd} \sum\nolimits_{i=1}^n \lambda \bomega_i^\T \bd_i - \bgamma_i^\T\bd_i  \\
  & = \left\{
      \begin{array}{l}
        0,~~~ ~if~ \bgamma_i=\lambda \bomega_i, ~\|\bomega_i\|_2\leq 1, \forall i,\\
        -\infty, otherwise,\\
      \end{array}
      \right.
\end{split}
\nonumber
\end{equation}
where $\bomega_i$'s are the dual variables for the $\ell_2$-norm regularizer. We obtain $\balpha=\bxi$ at the optimality when optimizing $\cL$ w.r.t. $\bxi$.
By substituting above into $\cL(\bx,\bz,\bxi, \balpha, \bbeta, \bgamma)$, we obtain the dual of the inner problem
\begin{equation}\label{eq:mptv_dual_tmp}
\begin{split}
  \max_{\balpha}& ~-\frac{1}{2}\||\balpha\|_2^2 + \balpha^\T\by\\
  \st &~ \bA^\T\balpha =\bD^\T\bbeta, \\
  & \tau_i \|\bbeta_i\|_2 \leq \lambda, ~\bbeta_i=\bC_i\bbeta, \forall i\in [n].\\
\end{split}
\end{equation}
This completes the proof.

\end{proof}

\section{Solve Problem \eqref{eq:ls_worst_case_ana} for Finding the Most-violated Constraint} \label{sec:app_solve_ls_worst}

The problem in \eqref{eq:ls_worst_case_ana} has a closed-solution:
\begin{equation}
  \bbeta = \left( \bD\bD^\T+\rho\bI \right)^{-1} \bD\bA^\T\balpha.
  \label{eq:ls_wca_clof_sol}
\end{equation}
Since $\bD=[\bD_v^\T, \bD_h^\T]^\T$ is a concatenation of $\bD_v$ and $\bD_h$, directly calculating \eqref{eq:ls_wca_clof_sol} via FFTs is non-trivial. 

\par
To calculate \eqref{eq:ls_wca_clof_sol} using FFTs, we firstly rewrite \eqref{eq:ls_wca_clof_sol} as 
\begin{equation}
  \bbeta = \left[
\begin{array}{c|c}
\bD_v\bD_v^\T + \rho\bI & \bD_v\bD_h^\T\\
\hline
\bD_h\bD_v^\T & \bD_h\bD_h^\T + \rho\bI
\end{array}
\right]^{-1} \left[
\begin{array}{c}
\bD_v\\
\hline
\bD_h
\end{array}
\right]\bA^\T\balpha
\end{equation}
Through a series of algebra calculation, we can transform the block-wise matrix inversion \cite{bernstein2009matrix} as:
\begin{equation}
  \left[
\begin{array}{c|c}
\bD_v\bD_v^\T + \rho\bI & \bD_v\bD_h^\T\\
\hline
\bD_h\bD_v^\T & \bD_h\bD_h^\T + \rho\bI
\end{array}
\right]^{-1}=\left[
\begin{array}{c|c}
\bS_h^{-1} & \bT_v\\
\hline
\bT_h & \bS_v^{-1}
\end{array}
\right],
\end{equation}
where 
\begin{equation}
\begin{split}
\bS_v &= \bD_h\bD_h^\T + \rho\bI-\bD_h\bD_v^\T(\bD_v\bD_v^\T + \rho\bI)^{-1}\bD_v\bD_h^\T,\\
\bS_h &= \bD_v\bD_v^\T + \rho\bI-\bD_v\bD_h^\T (\bD_h\bD_h^\T + \rho\bI)^{-1}\bD_h\bD_v^\T,\\
\bT_v &= -(\bD_v\bD_v^\T+\rho\bI)^{-1}\bD_v\bD_h^\T \bS_v^{-1},\\
\bT_h &=-(\bD_h\bD_h^\T+\rho\bI)^{-1}\bD_h\bD_v^\T \bS_h^{-1}.
\end{split}
\end{equation}
Therefor, we can achieve $\bbeta$ by rewriting \eqref{eq:ls_wca_clof_sol} as
\begin{equation}
\begin{split}
  & \bbeta_v = \bS_h^{-1} \bD_v \bA^\T\balpha + \bT_v\bD_h \bA^\T\balpha,\\
  & \bbeta_h = \bS_v^{-1} \bD_v \bA^\T\balpha + \bT_h\bD_h \bA^\T\balpha,
\end{split}
\end{equation}
which can be easily and efficiently calculated using FFTs.

\section{Proof of Proposition \ref{prop:subproblem}} \label{sec:app_subproblem}

For simplifying the proof, we first reformulate problem \eqref{eq:subproblem_primal} by introducing $\bxi=\by-\bA\bx$, and one constraint $(\bD\bx)_{\cS^c}=\0$. Problem \eqref{eq:subproblem_primal} then can be equivalently reformulated as
\begin{equation}
\begin{split}
\min_{\bx,\bxi,\bz_{\tcS}}  &~\frac{1}{2}\|\bxi\|_2^2 + \lambda\sum\nolimits_{i\in \cS} \|\bC_{i\tcS}\bz_{\tcS}\|_2\\
    \st &~\bxi=\by-\bA\bx, (\bD\bx)_{\tcS}=\bz_{\tcS},(\bD\bx)_{\tcS^c}=\0.\\
\end{split}
\label{eq:subproblem_primal_2}
\end{equation}
Following this, before proving Proposition \ref{prop:subproblem}, we provide the following lemma.
\begin{lemma}
Let $\bmu\in \Pi=\{\bmu~|~\sum_{i=1}^t\mu_i=1, \mu_i\geq0, \forall i, \bmu\in \mbR^t\}$, problem \eqref{eq:ccp_master_prob} can be solved by the following minimax problem:
\begin{equation}\label{eq:tmp_dual_lemma}
 \dgr{ \max_{\balpha\in \cA} \min_{\bmu\in \Pi} ~ -\!\!\!\sum_{\btau_i\in \Lambda_t} \mu_i \phi(\balpha, \btau_i)}.
\end{equation}
\end{lemma}
\begin{proof}
By introducing the dual variable $\mu_i, \forall i\in [t]$, the Lagrangian function of \eqref{eq:ccp_master_prob} is \dgr{$\cL(\balpha, \theta, \mu) = \theta+\sum_{i=1}^t \mu_i(\phi(\balpha, \btau_i)-\theta)$}. Let the derivatives w.r.t. $\theta$ be zero, we then can obtain $\sum_{i=1}^t \mu_i=1$ at the optimum. 
Furthermore, we can exchange the order of the $\max$ and $\min$ operators based on the minimax theorem \cite{sion1958general}. This completes the proof.
\end{proof}

\par
\dgr{We can directly rewrite problem \eqref{eq:tmp_dual_lemma} as
\begin{equation}
  \begin{split}
  \max_{\balpha\in\cA}\min_{\bmu\in\Pi}& ~\sum_{\btau_i\in\Lambda_t}\mu_i \left(-\frac{1}{2}\|\balpha\|_2^2 + \balpha^\T\by \right)\\
  \st &~~ \bA^\T\balpha =\bD^\T\bbeta, \\
  &~~ \tau_i \|\bbeta_i\|_2 \leq \lambda, ~\bbeta_i=\bC_i\bbeta, \forall i\in [n], \forall \btau\in \Lambda_t.
  \nonumber
\end{split}
\end{equation}}
\dgr{Because of the assumption $\sum_{i=1}^t\mu_i=1$, problem \eqref{eq:tmp_dual_lemma} can be reformulated as:}
\begin{equation} \label{eq:mptv_dual_proof}
\begin{split}
  \max_{\balpha\in\cA}& ~-\frac{1}{2}\|\balpha\|_2^2 + \balpha^\T\by\\
  \st &~~ \bA^\T\balpha =\bD^\T\bbeta, \\
  &~~ \tau_i \|\bbeta_i\|_2 \leq \lambda, ~\bbeta_i=\bC_i\bbeta, \forall i\in [n], \forall \btau\in \Lambda_t.\\
\end{split}
\end{equation}

Based on the assumption that there are no overlapping elements among $\cC_i$'s, we have $\sum\nolimits_{i\in \cS} \|\bC_{i\tcS}\bz_{\tcS}\|_2 = \sum_{i=1}^t \sum\nolimits_{j\in \cC_i} \|\bC_{j\tcS}\bz_{\tcS}\|_2$. To furthermore simplify the representation, we rewrite the constraints in problem \eqref{eq:subproblem_primal_2} with the group configuration matrices as $\bC_j\bD\bx = \bC_{j\tcS}\bz_{\tcS},\forall j\in \cC_i, \forall i\in[t]$ and $\bC_j\bD\bx=\0,\forall j\in\cS^c$, and introduce new variables $\bd_j$'s with $\bd_j=\bC_{j\tcS}\bz_{\tcS}$.
Thus the primal problem \eqref{eq:subproblem_primal_2} can be rewritten as
\begin{equation}
\begin{split}
&\min_{\bx,\bxi,\bz,\bd}  ~\frac{1}{2}\|\bxi\|_2^2 + \lambda\sum_{i=1}^t \sum\nolimits_{j\in \cC_i} \|\bd_j\|_2\\
&~~\st ~\bxi=\by-\bA\bx, \bC_j\bD\bx=\bC_{j\tcS}\bz_{\tcS}, \forall j\in \cC_i, \forall i\in[t],\\
&~~~\bC_j\bD\bx=\0, \forall j\in\cS^c, \bd_j = \bC_{j\tcS}\bz_{\tcS}, \forall j\in\cC_i, \forall i\in[t]. \\
\end{split}
\label{eq:subproblem_primal_3}
\end{equation}

\par
By introducing dual variables $\balpha$, $\bbeta$ and $\bgamma$, and letting $\bbeta_i=\bC_i\bbeta$, the Lagrangian function of problem is
\begin{equation}
\begin{split}
  \cL&(\bx, \bz_{\tcS}, \bxi, \bd, \balpha, \bbeta,\bgamma) = \frac{1}{2}\|\bxi\|_2^2 + \lambda\sum_{i=1}^t \sum_{j\in \cC_i} \|\bd_j\|_2 \\
  &+\balpha^\T(\by -\bA\bx-\bxi) + \sum_{i=1}^t \sum_{j\in\cC_i} \bbeta_j^\T(\bC_j\bD\bx-\bC_{j\tcS}\bz_{\tcS}) \\
  &+\sum_{j\in\cS^c} (\bC_j\bD\bx) + \sum_{i=1}^t \sum_{j\in\cC_i} \bgamma_j^\T(\bC_{j\tcS} \bz_{\tcS}-\bd_j).
\end{split}
\nonumber
\end{equation}
To derive the dual form, we minimize $\cL(\bx, \bz_{\tcS}, \bxi, \bd,\balpha, \bbeta, \bgamma)$ w.r.t. $\bx$, $\bz_{\tcS}$, $\bxi$ and $\bd$:
\begin{equation}
\begin{split}
  & \min_{\bxi}~ \frac{1}{2}\|\bxi\|_2^2 - \balpha^\T\bxi + \balpha^\T\by = -\frac{1}{2}\|\balpha\|_2^2 + \balpha^\T\by,\\
      \end{split}
  \nonumber
\end{equation}
\begin{equation}
\begin{split}
  & \min_{\bx}~ -\balpha^\T\bA\bx + \sum_{i=1}^t \sum_{j\in \cC_i} \bbeta_j^\T\bC_j\bD\bx + \sum_{j\in \cS^c} \bbeta_j^\T\bC_j\bD\bx \\
  &~~~~~~~~~~~~~~~~~~= \left\{
      \begin{array}{l}
        ~0,~~~~ \text{if}~ \bA^\T\balpha =\bD^\T\bbeta,\\
        -\infty, ~\text{otherwise},\\
      \end{array}
      \right.
      \end{split}
  \nonumber
\end{equation}
\begin{equation}
\begin{split}
  & \min_{\bz_{\tcS}}~ \sum_{i=1}^t \sum_{j\in \cC_i} -\bbeta_j^\T\bC_{j\tcS}\bz_\tcS + \sum_{i=1}^t \sum_{j\in \cC_i}\bgamma_j^\T\bC_{j\tcS}\bz_{\tcS}\\
  &~~~~~~~~~~~~~~~~~~ = \left\{
      \begin{array}{l}
        ~0,~~~~ \text{if}~ \bgamma_j=\bbeta_j, \forall j\in \cS,\\
        -\infty, ~\text{otherwise},\\
      \end{array}
      \right.
            \end{split}
  \nonumber
\end{equation}
\begin{equation}
\begin{split}
  & \min_{\bd}~ \sum_{i=1}^t \sum_{j\in \cC_i} \left(\lambda\|\bd_j\|_2 -  \bgamma_j^\T\bd_j\right)\\
  & = \min_{\bd} \max_{\|\bomega_j\|_2\leq 1, \forall j} \sum_{i=1}^t \sum_{j\in \cC_i} \left(\lambda \bomega_j^\T \bd_j - \bgamma_j^\T\bd_j \right)\\
  & \geq \max_{\|\bomega_j\|_2\leq 1, \forall j} \min_{\bd} \sum_{i=1}^t \sum_{j\in \cC_i} \lambda \bomega_j^\T \bd_j - \bgamma_j^\T\bd_j  \\
  & = \left\{
      \begin{array}{l}
        0,~~~ ~\text{if}~ \bgamma_j=\lambda \bomega_j, ~\|\bomega_j\|_2\leq 1, \forall j\in \cC_i, \forall i\in [t],\\
        -\infty, \text{otherwise},\\
      \end{array}
      \right.
\end{split}
\nonumber
\end{equation}
where $\bomega_j$'s are the dual variable for the $\ell_2$-norm regularizer. We then can obtain the dual problem of problem \eqref{eq:subproblem_primal_2}:
\begin{equation}\label{eq:subprob_dual_for_proof}
\begin{split}
  \max_{\balpha\in\cA}& ~-\frac{1}{2}\|\balpha\|_2^2 + \balpha^\T\by\\
  \st &~ \bA^\T\balpha =\bD^\T\bbeta, \\
  & \tau_j\|\bbeta_j\|_2 \leq \lambda, ~\bbeta_j=\bC_j\bbeta, \forall j\in \cC_i, \forall i\in [t].\\
\end{split}
\end{equation}
Since, for all $i\in [t]$, each $\cC_i$ corresponds to a $\btau_i$ in $\Lambda_t$, we thus can verify that \eqref{eq:mptv_dual_proof} is the dual form of problem \eqref{eq:subproblem_primal_2}. And there is $\balpha^*=\bxi^*$. Hence, the subproblem \eqref{eq:ccp_master_prob} can be addressed by solving \eqref{eq:subproblem_primal}, and there is $\balpha^*=\bxi^*$ and $\bxi^*=\by-\bA\bx^*$ at the optimum.
This completes the proof.


\section{Proof of the Convergence Analysis}\label{sec:proof_conv}
Before proving the proposition for convergence analysis in Proposition \ref{thm:convergence} in the main paper, we first give the following lemma.
\begin{lemma}
Let $(\balpha^*, \theta^*)$ be the global optimal solution of \eqref{eq:mptv2_qclp}, and $\{\theta_t\}_{t=1}^T$ be a sequence of $\theta$ obtained in the iterations in Algorithm \ref{algo:ccp_dual}, where $T=|\Lambda|$ denotes the possibly maximum iteration number. As the iteration index $t$ increases, $\{\theta_t\}$ is monotonically increasing. And there is $\theta_t\leq \theta^*$.
 \label{lem:mono_inc}
\end{lemma}

\begin{proof}
According to definition of problem \eqref{eq:mptv2_qclp} and subproblem \eqref{eq:ccp_master_prob}, we have
\begin{equation}
\begin{split}
  &\theta^*=\min_{\balpha\in \cA}\max_{\btau\in \Lambda}\phi(\balpha, \btau) \\
  \text{and}~~ &\theta_t=\min_{\balpha\in \cA}\max_{\btau\in \Lambda_t}\phi(\balpha, \btau).
\end{split}
\end{equation}
For any fixed $\balpha$, we have $\max_{\btau\in \Lambda_t}\phi(\balpha, \btau)\leq \max_{\btau\in \Lambda}\phi(\balpha, \btau)$, then
\begin{equation}
  \min_{\balpha\in \cA}\max_{\btau\in \Lambda_t}\phi(\balpha, \btau) \leq \min_{\balpha\in \cA}\max_{\btau\in \Lambda}\phi(\balpha, \btau),
\end{equation}
which means $\theta_t\leq \theta^*$. Furthermore, as $t$ increases, the size of the subset $\Lambda_t$ is increasing monotonically, so $\{\theta_t\}$ is monotonically increasing. The proof is completed.
\end{proof}

\subsection{Proof of Proposition \ref{thm:convergence}}
For convenience, we first define
\begin{equation}
\varphi_t=\min_{i\in[t]}(\max_{\btau\in \Lambda}\phi(\balpha_i, \btau) ).
\end{equation}
Furthermore, to complete the proof, we introduce the following lemma.
\begin{lemma}
Let $(\balpha^*, \theta^*)$ be the global optimal solution of \eqref{eq:mptv2_qclp} in the main paper, and $\{\varphi_t\}$ is corresponding to a sequence $\{\balpha_t, \theta_t\}$ generated by Algorithm \ref{algo:ccp_dual}. There is $\varphi_t\leq \theta^*$, and the sequence $\{\varphi_t\}$ is monotonically decreasing.
\label{lem:varphi}
\end{lemma}
\begin{proof}
For $\forall i$, $(\balpha_i, \max_{\btau\in \Lambda}\phi(\balpha_i, \btau))$ is the feasible solution of \eqref{eq:mptv2_qclp}. Then $\theta^*\leq \max_{\btau\in\Lambda}\phi(\balpha_i, \btau)$ for $\forall i\in[t]$. Thus we have
\begin{equation}
  \theta^*\leq \varphi_t =\min_{i\in[t]}(\max_{\btau\in \Lambda}\phi(\balpha_i, \btau) ),
\end{equation}
which shows that, with the increasing iteration $t$, $\{\varphi_t\}$ is monotonically decreasing.
\end{proof}
Then we conduct the proof for \dgrrr{Proposition} \ref{thm:convergence} in the main paper.
\begin{proof}
  We measure the convergence of the \dgr{MPTV} via the gap between the sequence $\{\theta_t\}$ and $\{\varphi_t\}$. 
We prove the convergence of the sequence $\{(\balpha_t, \theta_t)\}$ generated by the updating procedure defined in Algorithm \ref{algo:ccp_dual}. Specifically, we adapt the convergence analysis from \cite{Tan2014Towards,kortanek1993ccp}.
\dgr{We assume a limit point $(\bar{\balpha}, \bar{\theta})$ exists for $(\balpha_t, \theta_t)$ and $\bar{\theta}\leq \theta^*$, where $\theta^*$ denotes the optimal solution. To show $(\bar{\balpha}, \bar{\theta})$ is global optimal of problem \eqref{eq:mptv2_qclp}, we need to show $(\bar{\balpha}, \bar{\theta})$ is a feasible point of problem \eqref{eq:mptv2_qclp}, \ie $\bar{\theta}\geq \phi(\bar{\balpha}, \theta)$ for all $\btau\in \Lambda$, so $\bar{\theta}\geq \theta^*$ and we must have $\bar{\theta}=\theta^*$. Let $f(\balpha, \theta)=\min_{\btau\in \Lambda}(\theta -\phi(\balpha, \btau))=\theta-\max_{\btau\in \Lambda} \phi(\balpha, \btau)$. Then $f(\balpha, \theta)$ is continuous w.r.t. $(\balpha, \theta)$. Relying on the continuity, we then have 
\begin{equation}
\begin{split}
  f(\bar{\balpha}, \bar{\theta})& = f(\balpha_t, \theta_t) + (f(\bar{\balpha}, \bar{\theta}) -f(\balpha_t, \theta_t)) \\
  & = (\theta_t-\phi(\balpha_t, \btau_{t+1})) + (f(\bar{\balpha}, \bar{\theta}) - f(\balpha_t, \theta_t))\\
  & \geq (\theta_t-\phi(\balpha_t, \btau_{t+1})) - (\bar{\theta}-\phi(\bar{\balpha}, \btau_t)) \\
  &~~~~~~~~+ (f(\bar{\balpha}, \bar{\theta})-f(\balpha_t, \theta_t)) \rightarrow 0 ~~(\text{when} t\rightarrow 0).
\end{split}
\end{equation}
The proof is completed.
}
\end{proof}

\ifCLASSOPTIONcaptionsoff
  \newpage
\fi

\section*{Acknowledgment}
This work was partially supported by National Natural Science Foundation of China (61231016, 61602185), ARC Discovery Project Grant (DP160100703). Yanning Zhang is supported by Chang Jiang Scholars Program of China (100017GH030150, 15GH0301). Mingkui Tan is supported by Recruitment Program for Young Professionals.

\bibliographystyle{IEEEtranbib}
\bibliography{TVbib}


\end{document}